\documentclass[twoside]{article}

\usepackage[preprint]{aistats2026}

\usepackage[utf8]{inputenc} 
\usepackage[T1]{fontenc}    
\usepackage{hyperref}       
\usepackage{url}            
\usepackage{booktabs}       
\usepackage{amsmath, amsfonts, amsthm, amssymb}       
\usepackage{nicefrac}       
\usepackage{microtype}      
\usepackage{xcolor}         
\usepackage{tikz}
\usepackage{float}
\usepackage{dsfont}
\usepackage{enumitem}
\usepackage[linesnumbered,ruled,vlined,boxed]{algorithm2e}
\SetKwInput{KwInit}{Init}
\newcommand{\algoname}[1]{\texttt{#1}}
\newcommand{\ouralgo}{\algoname{MaxNorm} }
\newcommand{\oursecondalgo}{\algoname{Newton} }
\newcommand{\figuresize}{0.15\textwidth}
\newtheorem{definition}{Definition}
\newtheorem{proposition}{Proposition}
\newtheorem{lemma}{Lemma}
\newtheorem{theorem}{Theorem}

\newtheorem{remark}{Remark}
\newtheorem{corollary}{Corollary}

\newcommand{\PP}{\mathbb{P}}
\newcommand{\RR}{\mathbb{R}}
\newcommand{\R}{\mathbb{R}}

\renewcommand{\leq}{\leqslant} 
\renewcommand{\le}{\leqslant} 
\renewcommand{\geq}{\geqslant} 
\renewcommand{\ge}{\geqslant}

\newcommand{\cX}{\mathcal X}

\newcommand{\eps}{\varepsilon}
\renewcommand{\epsilon}{\varepsilon}

\newcommand{\problem}{\ref{prob:bilinear} }

\DeclareMathOperator{\diag}{diag}

\DeclareMathOperator{\sign}{sign}

\usepackage[style=numeric-comp, maxcitenames=50]{biblatex}
\usepackage{graphicx}
\usepackage{subcaption}

\usepackage{array}

\usepackage{multirow}
\usepackage{diagbox}

\usepackage[toc, acronym, 
    section=subsection, sort=def,style=long3colheader, 
    ]{glossaries}
\newglossary{linbandit}{linb}{lb}{Notation related to the Linear Bandit setting}

\newglossary{Opbdef}{pbdef}{pbd}{Notation related to optimization problem and its definitions}

\newglossary{Oalgorithm}{algo}{agl}{Notation related to optimization algorithm and its analysis}

\newglossary{abbreviation}{}{abl}{Abreviation used in the paper}

\makenoidxglossaries

\newglossaryentry{t}
{
    type=linbandit,
    name={\ensuremath{t}},
    description = {Current time step. Not to be confused with the barrier parameter of Newton's method},
}

\newglossaryentry{dim}
{
    type=linbandit,
    name={\ensuremath{d}},
    description = {Dimension of the problem considered},
}

\newglossaryentry{rewardt}
{
    type=linbandit,
    name={\ensuremath{y_{t}}},
    description = {Reward received at time $t$},
}

\newglossaryentry{linkf}
{
    type=linbandit,
    name={\ensuremath{g}},
    description = {Link function, usually monotonic in this paper we assume non-decreasing},
}

\newglossaryentry{truezeta}
{
    type=linbandit,
    name={\ensuremath{\zeta}},
    description = {Unkown true parameter of the linear model $\zeta \in \mathbb{R}^d$},
}

\newglossaryentry{center}
{
    type=linbandit,
    name={\ensuremath{c}},
    description = {Center of the ellipsoid that define the action set},
}

\newglossaryentry{actiont}
{
    type=linbandit,
    name={\ensuremath{x_t}},
    description = {Action taken at time $t$},
}
\newglossaryentry{actionsett}
{
    type=linbandit,
    name={\ensuremath{\mathcal{X}_t}},
    description = {Set of vailable actions at time $t$},
}

\newglossaryentry{noiset}
{
    type=linbandit,
    name={\ensuremath{z_{t}}},
    description = {Noise at time $t$},
}

\newglossaryentry{noiselvl}
{
    type=linbandit,
    name={\ensuremath{\sigma}},
    description = {Subgaussian parameter of the noise},
}

\newglossaryentry{xstart}
{
    type=linbandit,
    name={\ensuremath{x^{\star}_t}},
    description = {Optimal action as a function as a function of the time $t$ and $\zeta$},
}

\newglossaryentry{avregret}
{
    type=linbandit,
    name={\ensuremath{R_{T}(\zeta)}},
    description = {Average regret at time $T$ on the environment parametrized by $\zeta$},
}

\newglossaryentry{confct}
{
    type=linbandit,
    name={\ensuremath{\mathcal{C}_{t}}},
    description = {Confidence ellipsoid at time $t$ for optimistic algorithms},
}

\newglossaryentry{hatzetat}
{
    type=linbandit,
    name={\ensuremath{\hat{\zeta}_{t-1}}},
    description = {Estimate of the parameter $\zeta$ at time $t$},
}

\newglossaryentry{normdesignt}
{
    type=linbandit,
    name={\ensuremath{W_{t-1}}},
    description = {Symmetric definite positive matrix defining the norm and the confidence ellipsoid at time $t$ for optimistic algorithms},
}

\newglossaryentry{normM}
{
    type=linbandit,
    name={\ensuremath{\|z\|_{M}}},
    description = {Norm of $z$ defined by the symmetric positive definite matrix $M$, $\|z\|_{M}^2 = z^\top M z$},
}

\newglossaryentry{Pb}
{
    type=Opbdef,
    name={\ensuremath{P_B}},
    description = {Bilinear optimization problem, $\text{maximize } x^\top \theta \text{ subject to } (x,\theta) \in \mathcal{X} \times \Theta$},
}

\newglossaryentry{xset}
{
    type=Opbdef,
    name={\ensuremath{\mathcal{X}}},
    description = {Set of actions, a convex subset of $\mathbb{R}^d$ },
}

\newglossaryentry{A}
{
    type=Opbdef,
    name={\ensuremath{A}},
    description = {Symmetric positive definite  matrix defining a norm and an ellipsoid action set},
}

\newglossaryentry{ellipsoid}
{
    type=Opbdef,
    name={\ensuremath{\Theta}},
    description = {Ellipsoid set of parameters, $\Theta = \{\theta \in \mathbb{R}^d:  \|\theta-c\|_{W} \le 1 \}$},
}

\newglossaryentry{ellcenter}
{
    type=Opbdef,
    name={\ensuremath{c}},
    description = {Center of the ellipsoid in canonical base defining the ellipsoid set $\Theta$},
}

\newglossaryentry{Wmatrix}
{
    type=Opbdef,
    name={\ensuremath{W}},
    description = {Symmetric positive definite matrix in the canonical base defining the ellipsoid $\Theta$},
}

\newglossaryentry{p}
{
    type=Opbdef,
    name={\ensuremath{p}},
    description = {Positive number defining the $p$ norm, $\|z\|_p^p = \sum_{i=1}^d |z_i|^p$},
}

\newglossaryentry{pnorm}
{
    type=Opbdef,
    name={\ensuremath{\|z\|_p}},
    description = {Norm of $z$ defined by the positive number $p$, $\|z\|_p^p = \sum_{i=1}^d |z_i|^p$},
}

\newglossaryentry{Odot}
{
    type=Opbdef,
    name={\ensuremath{\odot}},
    description = {Hadamard product (entries-by-entries) between two matrices},
}

\newglossaryentry{diag}
{
    type=Opbdef,
    name={\ensuremath{\diag}},
    description = {For $z \in R^d$, $\diag(z)$ denote the diagonal matrix whose diagonal entries are the entries of $z$},
}

\newglossaryentry{ballp}
{
    type=Opbdef,
    name={\ensuremath{B_{p,d}}},
    description = {Unit ball of the $p$ norm, $\smash{B_{p,d} = \{x \in \mathbb{R}^d: \|x\|_{p} \le 1\}}$},
}

\newglossaryentry{Px}
{
    type=Opbdef,
    name={\ensuremath{P_X}},
    description = {Optimization problem when maximizing $P_B$ in $\theta$ first, $\text{maximize } x^\top c + \|x\|_{W^{-1}} \text{ subject to } x \in \mathcal{X}$},
}

\newglossaryentry{xstar}
{
    type=Opbdef,
    name={\ensuremath{x^\star}},
    description = {Optimal solution of the problem $P_B$ and $P_X$},
}

\newglossaryentry{thetastar}
{
    type=Opbdef,
    name={\ensuremath{\theta^\star}},
    description = {Optimal solution of the problem $P_B$},
}

\newglossaryentry{U}
{
    type=Oalgorithm,
    name={\ensuremath{U}},
    description = {Change of basis matrix, so that $U^\top U = I$ and $A ^{1/2} W A^{1/2}  = U^\top\Lambda U$},
}

\newglossaryentry{lambdamat}
{
    type=Oalgorithm,
    name={\ensuremath{\Lambda}},
    description = {Diagonal matrix of the decreasing ordered eigenvalues of the matrix $A^{1/2} W A^{1/2}$},
}

\newglossaryentry{lambda}
{
    type=Oalgorithm,
    name={\ensuremath{\lambda}},
    description = {Vector of the decreasing ordered eigenvalues of the matrix $A^{1/2} W A^{1/2}$},
}

\newglossaryentry{kappa}
{
    type=Oalgorithm,
    name={\ensuremath{\kappa}},
    description = {Vector of the decreasing ordered eigenvalues of the matrix $A^{1/2} W A^{1/2}$},
}

\newglossaryentry{centerb}
{
    type=Oalgorithm,
    name={\ensuremath{b}},
    description = {Change of basis (variable) of $c$, the center of the ellipsoid $\Theta$, $b := U 
    A^{-1/2} c $},
}

\newglossaryentry{bset}
{
    type=Oalgorithm,
    name={\ensuremath{\mathcal{I}}},
    description = {Set where the $b_i$ are non-zero, $\mathcal{I} = \{i \in [d]: b_i \neq 0\}$},
}

\newglossaryentry{clipcenterb}
{
    type=Oalgorithm,
    name={\ensuremath{b_{+}}},
    description = {Clipped version of $b$, $b_{+} = ({\bf sign}(b_i) \max(
  |b_i|,\epsilon/(2 \sqrt{d})))_{i \in [d]}$},
}

\newglossaryentry{mu}
{
    type=Oalgorithm,
    name={\ensuremath{\mu}},
    description = {Lagrange multiplier associated to the constraint $\|\phi-b_+\|_{\Lambda} \le 1$},
}

\newglossaryentry{mustar}
{
    type=Oalgorithm,
    name={\ensuremath{\mu^\star}},
    description = {Optimal Lagrange multiplier associated to the constraint $\|\phi-b_+\|_{\Lambda} \le 1$},
}

\newglossaryentry{muhat}
{
    type=Oalgorithm,
    name={\ensuremath{\hat{\mu}}},
    description = {Approximation of the Lagrange multiplier $\mu^\star$},
}

\newglossaryentry{epsilon}
{
    type=Oalgorithm,
    name={\ensuremath{\epsilon}},
    description = {Target required accuracy of the algorithm, $\epsilon > 0$},
}

\newglossaryentry{u}
{
    type=Oalgorithm,
    name={\ensuremath{u}},
    description = {Change of basis (variable) of $x$, $u = U A^{1/2} x$},
}

\newglossaryentry{ustar}
{
    type=Oalgorithm,
    name={\ensuremath{u^\star}},
    description = {Optimal solution of the problem $P_B'$},
}

\newglossaryentry{phi}
{
    type=Oalgorithm,
    name={\ensuremath{\phi}},
    description = {Change of basis (variable) of $\theta$, $\phi = U A^{-1/2} \theta$},
}

\newglossaryentry{phistar}
{
    type=Oalgorithm,
    name={\ensuremath{\phi^\star}},
    description = {Optimal solution of the problem $P_B''$},
}

\newglossaryentry{phistarplus}
{
    type=Oalgorithm,
    name={\ensuremath{\phi^\star_+}},
    description = {Optimal solution of the problem $P_B'''$},
}

\newglossaryentry{phiprime}
{
    type=Oalgorithm,
    name={\ensuremath{\phi'}},
    description = {Alternate variable of $\phi$, used to show the sub-optimality of some $\phi$},
}

\newglossaryentry{pbprime}
{
    type=Oalgorithm,
    name={\ensuremath{P_{B'}}},
    description = {Equivalent problem to $P_B$ in the new basis, $\text{maximize } u^\top \phi \text{ subject to } \|u\|_{2} \le 1 \text{ and } \|\phi-b\|_{\Lambda} \le 1$},
}

\newglossaryentry{dichonumber}
{
    type=Oalgorithm,
    name={\ensuremath{J}},
    description = {Number of steps of binary search},
}

\newglossaryentry{lowermu}
{
    type=Oalgorithm,
    name={\ensuremath{m}},
    description = {Lower bound on the optimal Lagrange multiplier $\mu^\star$, $m =\lambda_{d}^{-1}
    (\lambda_{d}^{1/2} |b_{+, d}| + 1)$},
}
\newglossaryentry{uppermu}
{
    type=Oalgorithm,
    name={\ensuremath{M}},
    description = {Upper bound on the optimal Lagrange multiplier $\mu^\star$, $M = \lambda_{d}^{-1} \sqrt{
        \sum_{i=1}^{d} \lambda_{i} b_{+, i}^2} + \lambda_{d}^{-1} $},
}

\newglossaryentry{Pconv}
{
    type=Oalgorithm,
    name={\ensuremath{P_C}},
    description = {Reduction of the problem $P_B$ to a convex minimization problem $P_C$},
}

\newglossaryentry{simplex}
{
    type=Oalgorithm,
    name={\ensuremath{\Delta_{d-1}}},
    description = {The simplex of dimension $d-1$, $\Delta_{d-1} = \{y \in \mathbb{R}^d: y_i \ge 0, i \in [d] \text{ and } \sum_{i=1}^{d} y_i = 1 \}$},
}

\newglossaryentry{y}
{
    type=Oalgorithm,
    name={\ensuremath{y}},
    description = {Change of variable such that $x_i = \sqrt{y_i}\sign{b_i}$ used to reduce the problem $P_B$ to a convex minimization problem $P_C$},
}

\newglossaryentry{ystar}
{
    type=Oalgorithm,
    name={\ensuremath{y^\star}},
    description = {Optimal solution of the problem $P_C$},
}

\newglossaryentry{F}
{
    type=Oalgorithm,
    name={\ensuremath{F}},
    description = {Objective function of the problem $P_C$, $F(y) :=
    -\sum_{i=1}^{d} \sqrt{y_i} \, |b_i|
    - \sqrt{\smash{\sum_{i=1}^{d} \lambda_i^{-1} y_i}
      \vphantom{\sum_{i=1}}
      \rule{0pt}{15.5pt}
      }$ },
}

\newglossaryentry{ntot}
{
    type=Oalgorithm,
    name={\ensuremath{N_{tot}}},
    description = {Total number of Newton steps required to solve the problem $P_C$},
}

\newglossaryentry{tzero}
{
    type=Oalgorithm,
    name={\ensuremath{t_0}},
    description = {Initial barrier parameter of the Newton method, that makes $F^{(t_0)}$ self concordant},
}

\newglossaryentry{Ft}
{
    type=Oalgorithm,
    name={\ensuremath{F^{(t)}}},
    description = {Penalized objective function of the problem $P_C$, $F^{(t)}(y) = t F(y) 
    - \sum_{i = 1}^{d} \log( y_i - B_i)
    -  \log\bigg( 1 - \sum_{i=1}^{d} y_i\bigg)$},
}

\newglossaryentry{ytstar}
{
    type=Oalgorithm,
    name={\ensuremath{y^{t,\star}}},
    description = {Optimal solution of the problem that minimize $F^{(t)}$},
}
\newglossaryentry{nutstar}
{
    type=Oalgorithm,
    name={\ensuremath{\nu^{t,\star}}},
    description = {Lagrange multiplier associated with $y^{t,\star}$ used to bound the optimality gap of $F(y^{t,\star})$},
}

\newglossaryentry{Bi}
{
    type=Oalgorithm,
    name={\ensuremath{B_i}},
    description = {Vector of the lower bounds of the variables $y_i$ in the problem $P_C$ that define a set where $F^{(t)}$ is self concordant, },
}

\newglossaryentry{B}
{
    type=Oalgorithm,
    name={\ensuremath{B}},
    description = {The vector of the $B_i$, $B = (b_i^2 ( \|b\|_{2} + \lambda_d^{-1/2} )^{-2})_{i \in [d]}  $},
}

\newglossaryentry{DF}
{
    type=Oalgorithm,
    name={\ensuremath{D_F}},
    description = {Subspace of where the function $F^{(t)}$ is self concordant, $DF = \{y \in \mathbb{R}^d: y_i \ge B_i, i \in [d] \text{ and } 1 - \sum_{i=1}^{d} y_i > 0\}$},
}

\newglossaryentry{Lagrangian}
{
    type=Oalgorithm,
    name={\ensuremath{\mathcal{L}}},
    description = {Lagrangian function of maximizing $F$ over $D_F$},
}

\newglossaryentry{mlagrangian}
{
    type=Oalgorithm,
    name={\ensuremath{\nu}},
    description = {Lagrange multiplier associated to $D_F$},
}

\newglossaryentry{mlagrangianstar}
{
    type=Oalgorithm,
    name={\ensuremath{\nu^\star}},
    description = {Optimal Lagrange multiplier of minimizing $F$ over $D_F$},
}

\newglossaryentry{y0}
{
    type=Oalgorithm,
    name={\ensuremath{y^0}},
    description = {Initial feasible point for newton's method.},
}

\newglossaryentry{dif}
{
    type=Oalgorithm,
    name={\ensuremath{\nabla F}},
    description = {First derivative, gradient of the function $F$},
}

\newglossaryentry{dif2}
{
    type=Oalgorithm,
    name={\ensuremath{\nabla^2 F}},
    description = {Second derivative, Hessian of the function $F$},
}

\newglossaryentry{dif3}
{
    type=Oalgorithm,
    name={\ensuremath{\nabla^3 F}},
    description = {Third derivative of the function $F$},
}

\newglossaryentry{localnorm}
{
    type=Oalgorithm,
    name={\ensuremath{\lambda_F}},
    description = {Local norm of the gradient of the function $F$, $\lambda_F(y) = [(\nabla F)(y)]^\top  [(\nabla^2 F)(y)]^{-1} [(\nabla F)(y)]$},
}

\newglossaryentry{yplus}
{
    type=Oalgorithm,
    name={\ensuremath{y^+}},
    description = {One damped newton set starting at $y$},
}

\newglossaryentry{omega}
{
    type=Oalgorithm,
    name={\ensuremath{\omega}},
    description = {Function used to lower bound the decrease of the objective value after one newton step $\omega(l) := l-\ln(1+l)$},
}

\newglossaryentry{omegastar}
{
    type=Oalgorithm,
    name={\ensuremath{\omega_\star}},
    description = {Function used to bound the optimality gap  $\omega_\star(a) := -a - \ln(1-a) < a^2 / (2(1-a))$},
}

\newglossaryentry{eta}
{
    type=Oalgorithm,
    name={\ensuremath{\eta}},
    description = {Parameter used to increase the barrier parameter exponentially},
}

\newglossaryentry{ouralgo}
{
    type=abbreviation,
    name={~\ouralgo~},
    description={Our proposed algorithm to solve $P_B$ when  $\mathcal{X} = \{x \in \RR^d, \|x\|_A \leq 1 \}$ },
}

\newglossaryentry{oursecondalgo}
{
    type=abbreviation,
    name={~\oursecondalgo~},
    description={Our second proposed algorithm to solve $P_B$ when  $\mathcal{X} = \{x \in \RR^d, \|x\|_A \leq 1 \}$  based on a convex reduction of $P_B$ to $P_C$},
}

\newglossaryentry{scdi}
{
    type=abbreviation,
    name={SCDI},
    description={Self concordant differential inequality, it relates the third derivative of a function to its second derivative},
}

\newglossaryentry{sdp}
{
    type=abbreviation,
    name={SDP},
    description={Semi definite programming, a class of convex optimization problems where the variable is a positive semi-definite matrix},
}

\begin{document}

%

%

\twocolumn[

\aistatstitle{Tractable Instances of Bilinear Maximization: Implementing LinUCB on Ellipsoids}

\aistatsauthor{ Raymond Zhang  \And Hédi Hadiji \And  Richard Combes }

\aistatsaddress{ Laboratoire des signaux et systèmes, Université Paris-Saclay, CNRS, CentraleSupélec, France} ]

\begin{abstract}
We consider the maximization of $x^\top \theta$ over $(x,\theta) \in \mathcal{X} \times \Theta$, with $\mathcal{X} \subset \mathbb{R}^d$ convex and $\Theta \subset \mathbb{R}^d$ an ellipsoid. This problem is fundamental in linear bandits, as the learner must solve it at every time step using optimistic algorithms.
We first show that for some sets $\mathcal{X}$ e.g. $\ell_p$ balls with $p>2$, no
efficient algorithms exist unless $\mathcal{P} = \mathcal{NP}$.
We then provide two novel algorithms solving this problem efficiently when
$\mathcal{X}$ is a centered ellipsoid.
Our findings provide the first known method to implement optimistic algorithms for linear bandits in high dimensions.
\end{abstract}

\section{Introduction}
\label{sec:Introduction}
\subsection{Optimistic Algorithms for Linear Bandits}
\label{subsec:Optimistic algorithms for linear bandits}
In the generalized linear bandit problem \cite{filippi2010glsucb} with subgaussian rewards, at each time $\gls{t} \in \mathbb{N}$, a learner selects an action $\gls{actiont}$ in $\gls{actionsett} \subset \mathbb{R}^{\gls{dim}}$ the set of actions available at time $t$ and
observes a reward
\begin{align*}
	\gls{rewardt} = g\big(x_t^\top \zeta\big) + z_t
	\,, 
\end{align*}
where $\gls{linkf}:\mathbb{R} \to \mathbb{R}$ is a non-decreasing link function, $\gls{truezeta} \in
\mathbb{R}^d$ is an unknown vector, and $(\gls{noiset})_{t \in \mathbb{N}}$ is a zero-mean noise sequence.
For instance, if $g(a) = a$, then the problem is the classical linear bandit problem
\cite{dani_stochastic_2008}
and if $\smash{g(a) = (1+e^{-a})^{-1}}$, it becomes a logistic bandit problem
\cite{faury2020improved}.
The noise $\smash{(z_t)}$ is a sequence of independent $\gls{noiselvl}^2$-subgaussian
variables, such that for any $t \in \mathbb{N}$ and $u\in \mathbb{R}$,
\begin{align*}
	\mathbb{E}(e^{u z_t}) \le e^{  u^2 \sigma^2 / 2} \,.
\end{align*}
The goal is to minimize the regret, which is the difference between the expected reward
obtained by the learner, and that of an oracle who knows $\zeta$ and selects the best
decision $\gls{xstart} \in \arg\max_{x \in \mathcal{X}_t} x^\top \zeta$
\begin{align*}
	\gls{avregret}
	= \sum_{t=1}^{T} \left[
		g\Big(\max_{x \in \mathcal{X}_t} x^\top \zeta\Big) 
		-  \mathbb{E}\big(g(x_t^\top \zeta)\big) 
		\right]  
	\,.
\end{align*}
For this problem, a class of algorithms called optimistic algorithms (listed in
Subsection~\ref{ssec:Our contribution}) select the action $x_t$ 
that solves the optimization problem
\begin{align*}
	\text{maximize } x^\top \theta 
	\text{ subject to } 
	(x,\theta) \in \mathcal{X}_t \times \mathcal{C}_t \,, 
\end{align*}
in which $\mathcal{C}_t$ is the so-called confidence ellipsoid that contains the unknown $\zeta$ with high probability
\begin{align*}	
	\gls{confct} = \{ \theta \in \mathbb{R}^d : \|\theta - \hat\zeta_{t-1} \|_{W_{t-1}} \le 1 \} 
	\,; 
\end{align*} 
the vector $\gls{hatzetat}$ is an estimate of $\zeta$ at time $t$, and $\gls{normdesignt}$ is
a well-chosen symmetric positive semi-definite matrix.
Given a vector $z\in \mathbb{R}^d$ and a symmetric positive (semi-)definite matrix $M \in
\mathbb{R}^{d \times d}$ we denote by $\gls{normM}^2 = z^\top M z$ the corresponding
(semi-)norm.
Of course $\hat\zeta_{t-1}$ and $W_{t-1}$ may depend on $t$, the chosen actions
$x_1,...,x_{t-1}$ and the observed rewards $y_1,...,y_{t-1}$, and are defined
differently for each algorithm.
However, the general idea remains the same: the learner $\mathcal{C}_t$ designs using concentration
inequalities under sub-gaussian noise so that $\zeta \in \mathcal{C}_t$ with high probability,
and choosing $x_t$ as above will both ensure that we explore enough, while also
exploiting directions that gave high rewards historically.
While optimistic algorithms enjoy strong performance guarantees, they are hard to implement because computing $x_t$ by solving the above non-convex optimization problem can be difficult, and implementing them in high dimensions might be challenging.
The present work aims to study this family of optimization problems in detail.
\subsection{Bilinear Maximization Problems}\label{ssec:Model}
We study instances of the following bilinear optimization problem, called $\gls{Pb}$
\begin{align}
	\label{prob:bilinear}
	\tag{$P_B$}
	\text{maximize } x^\top \theta \text{ subject to } (x,\theta) \in \mathcal{X} \times \Theta
	\,.
\end{align}
We focus on the special case in which $\gls{xset}$ is a convex subset of
$\mathbb{R}^d$ and $\Theta$ is an ellipsoid
\begin{align*}
	\gls{ellipsoid} = \{\theta \in \mathbb{R}^d:  \|\theta-c\|_{W} \le 1 \}
\end{align*}
where $\gls{ellcenter}$ is the center of $\Theta$ and $\gls{Wmatrix}$ is a known symmetric positive definite matrix.  
This problem is non-convex and non-trivial to solve, as we shall see.
Indeed, in order to implement any of the optimistic algorithms, one must
solve $P_B$ at each time $t \in \mathbb{N}$ with inputs $\smash{c = \hat{\zeta}_{t-1}}$
and $W = W_{t-1}$ and $\mathcal{X} = \mathcal{X}_t$.
We will study $P_B$ for various convex sets $\mathcal{X}$, to design efficient algorithms to solve $P_B$ when possible. 
We also exhibit families of pairs $\cX, \Theta$ for which no efficient algorithms exist; in this case, none of the optimistic algorithms are implementable in high dimensions.
\subsection{Our Contributions}\label{ssec:Our contribution}
We make three novel contributions:
\begin{enumerate}[leftmargin=*, topsep=0pt, itemsep=0pt]
\item  We show that in some cases, for instance when $\mathcal{X}$ is an $\ell_p$ ball for any $p>2$, $P_B$ cannot be solved exactly, nor can it be solved beyond a fixed approximation ratio, in time polynomial in the dimension $d$, unless $\mathcal{P}=\mathcal{NP}$. So, it is impossible to implement any of the known optimistic algorithms efficiently in high dimensions. We also show that implementing optimistic algorithms approximately by solving $P_B$ up to a fixed approximation ratio is useless as it always yields linear regret. 
\item We show that when $\mathcal{X}$ is a centered ellipsoid, $P_B$ can be efficiently solved and we propose two algorithms called \ouralgo and \oursecondalgo to do so. \ouralgo is an intricate, non-iterative method that computes an $\epsilon$-optimal solution to $P_B$ in time $O(d^3 + d\log\left(d \pi_1(b,1/\lambda_{\min},\lambda_{\max})/ \epsilon \right) )$ and memory $O(d^2)$ where $b,\lambda$ are two vectors which depend on $\mathcal{X},\Theta$ and $\pi_1$ is a polynomial. Our second algorithm \oursecondalgo is an iterative method based on Newton's method applied to a convex problem related to $P_B$ ($P_B$ is non-convex and one cannot solve it by either Newton's method, alternate minimization, or local search) and computes an $\epsilon$-optimal solution to $P_B$ in time $\tilde{O}(d^3 + d^{5/2}\log_2\log_2(1/\epsilon) + d^{5/2}\log{1/\epsilon})$ where $\tilde{O}$ hide polylog dependency in $b,\lambda$ and memory $O(d^2)$.
\item We implement both \ouralgo and \oursecondalgo and show through extensive numerical experiments that they can solve instances of $P_B$ in very high dimensions, say $d \ge 10^3$, even for ill-conditioned problem instances. The code for all experiments is available as supplementary material, making experiments fully reproducible, and will be made public after the review process.
\end{enumerate}
\subsection{Related Work}\label{ssec:Related Work}
\textbf{Algorithms for linear bandits} \quad 
Many authors have proposed algorithms for linear bandits: 
i) algorithms based around explore-then-commit such as PEGE~\cite{rusmi_linearly_2010} and E2TC~\cite{zhang_linear_2025} 
ii) sampling-based algorithms like TS~\cite{abeille_2017,abeille_when_2025} and 
iii) optimistic algorithms such as CB~\cite{dani_stochastic_2008}, OFUL~\cite{abbasi2011improved} and OSLOFUL~\cite{gales2022norm-agn} designed for linear bandits and GLS-UCB~\cite{filippi2010glsucb} designed for generalized linear bandits. 
The optimistic algorithms are conceptually simple, popular, enjoy good theoretical regret guarantees as well as good statistical performance in practice with regret upper bounded by $\tilde{O}(d\sqrt{T})$. Implementing any of those algorithms is equivalent to solving $P_B$ as explained above, so we believe that studying the hardness of $P_B$ and designing efficient algorithms to solve $P_B$ is an important open problem in linear bandits. In the case of ellipsoidal action sets, a reduction to a semi definite program (\gls{sdp}) is proposed in \cite{zhangb16onebit}, unfortunately without any proof on the complexity of solving the SDP. 

\textbf{Similarity with combinatorial semi-bandits}  \quad
A problem closely related to linear bandits is combinatorial semi-bandits, where $\mathcal{X}_t \subset \{0,1\}^d$, and several optimistic algorithms for this problem such as ESCB \cite{combes_2015escb} and OLS-UCB \cite{degenne_2016olsucb}, must solve a problem similar to $P_B$ and $P_X$ defined in Proposition~\ref{proposition:maximization_over_theta}. \cite{atamturk2017maximizing} further showed that for some combinatorial sets $\mathcal{X}$, the problem cannot be solved exactly unless $\mathcal{P} = \mathcal{NP}$, and derive approximation schemes
with a fixed approximation ratio, and \cite{cuvelier2021statistically} provides strategies to design approximate versions of ESCB and OLS-UCB.

The rest of the article is organised as follows. 
In Section~\ref{sec:Model}, we define the model and demonstrate that for some sets $\mathcal{X}$, the problem $P_B$ cannot be solved efficiently. In Section~\ref{sec:algs} we propose an efficient, non-iterative algorithm to solve $P_B$ when the set $\mathcal{X}$ is a centered ellipsoid. In Section~\ref{sec:convex_newton} we propose another efficient algorithm based on Newton's method, using a non-trivial reduction of $P_B$ to a convex program, once again when the set $\mathcal{X}$ is a centered ellipsoid. In Section~\ref{sec:Numerical experiments}, we assess the numerical performance of the proposed algorithms experimentally. Section~\ref{sec:Conclusion} concludes the paper. We present additional proofs and numerical experiments in the appendix. 

\textbf{Some notation}  \quad
Given $\gls{p} \geq 1$, a vector $z$ in $\mathbb{R}^d$ 
let $\smash{\gls{pnorm} = (|z_1|^p + \dots + |z_d|^p)^{1/p}}$ denote the $\ell_p$ norm.
We also use some compact notation for component-wise operations on vectors: 
$z \gls{Odot} z'$ denotes the Hadamard (component-by-component) product, and we write $z \geq z'$ if all entries of $z - z'$ are non-negative. 
$\gls{diag}(z)$ denotes the diagonal matrix whose diagonal entries are the entries of $z$.

\section{Model and Complexity} \label{sec:Model}In this section, we first show that solving $P_B$ only approximately usually results in linear regret in linear bandits, so the learner must solve $P_B$ to high accuracy.
We then discuss the complexity of solving $P_B$ as a function of $\mathcal{X}$.  
We consider the cases where $\mathcal{X}$ is a polytope, where $\mathcal{X}$ is a centered ellipsoid, and where $\mathcal{X}$ is the $\ell_p$ unit ball $\smash{\gls{ballp} = \{x \in \mathbb{R}^d: \|x\|_{p} \le 1\}}$ for $p \ge 1$.
Those cases cover a wide variety of convex sets, and while some are solvable efficiently, others are provably hard, even up to a fixed approximation ratio, as we show next. 
We present all proofs in Appendix~\ref{app:proofs_model}.
\subsection{Approximate $\epsilon$-LinUCB}
\label{ssec:Approximate LinUCB}
We now study the impact of solving $P_B$ approximately in the context of linear bandits. 
A couple $(\hat{x}, \hat{\theta}) \in \mathcal{X} \times \Theta$ is a solution of $P_B$ with approximation ratio $\epsilon$ if $\hat{x}^\top \hat{\theta} \geq (1-\epsilon) {x^\star}^\top \theta^\star$. 
An algorithm is an $\epsilon-$\algoname{LinUCB} algorithm if at each round $t$, it computes an approximate solution $(\hat{x}_t,\hat{\theta}_t)$ to $P_B$ with approximation ratio $\epsilon$ and selects action $\hat{x}_t$.

Proposition~\ref{proposition:LinUCB_approx} shows that for any fixed $\epsilon > 0$, there exists $\epsilon-$\algoname{LinUCB} algorithms with linear regret. 
The result holds for both discrete and continuous action sets. 
Furthermore, suppose $\epsilon$ is allowed to depend on the time horizon. 
In that case, it should be proportional to $1/\sqrt{T}$ for the minimax regret to scale at the optimal rate, i.e., proportionally to $\sqrt{T}$. 
Finally, if $P_B$ is computationally hard to approximate up to a fixed approximation ratio, then implementing the corresponding linear bandit algorithm is computationally hard, too. 
\begin{proposition}\label{proposition:LinUCB_approx}
	Consider $\epsilon \in [0,1)$, and either $\cX = \{1-\epsilon, 1\}$ or $\cX = [a, b]$ with $a<0<b$. Then there exists an approximate $\epsilon-$\algoname{LinUCB} algorithm and parameters $\zeta$ with such that $\lim\inf_{T \rightarrow +\infty} R_T(\zeta)/T \geq \epsilon$ 
\end{proposition}
\subsection{Equivalent Formulation}\label{ssec:Equivalent formulation}We now show that $P_B$ is equivalent to another optimization problem $\gls{Px}$ in Proposition~\ref{proposition:maximization_over_theta}, which will be helpful later. The proof follows from the definition of $P_B$ and maximization over $\theta$ for any fixed $x \in \mathcal{X}$. 
\begin{proposition}\label{proposition:maximization_over_theta}
Problem $P_B$ is equivalent to the optimization problem
	\begin{align}
	\tag{$P_X$}
	\text{maximize } \{ x^\top c + \|x\|_{W^{-1}} \} \text{ subject to } x \in \mathcal{X} 
\end{align}
in the sense that $x^\star$ solves $P_X$ iff $(\gls{xstar},\gls{thetastar})$ with $\theta^\star = c + W^{-1} x^\star/\|x^\star\|_{W^{-1}}$ solves~$P_B$.
\end{proposition}
\begin{proof}
Consider $x$ fixed so that we are maximizing a linear function $x^\top \theta$ over a
convex set $\Theta$. From the KKT conditions, the maximizer of $\theta \mapsto x^\top
\theta$ subject to $\|\theta-c\|_{W} \le 1$ is $\theta = c + W^{-1} x/\|x\|_{W^{-1}}$,
with value $x^\top \theta = x^\top c + \|x\|_{W^{-1}}$, and replacing in the definiton
of $P_B$ proves the result. 
\end{proof}
\subsection{Polytopes with Polynomially Many Vertices}
\label{ssec:Polytopes with polynomial number of vertices}
As a consequence, if $\mathcal{X}$ is a polytope with polynomially many vertices, for instance the $\ell_1$ unit ball which has $2d$ vertices, then $P_B$ is solvable in polynomial time, simply by solving $P_X$ over the vertices of $\mathcal{X}$ using exhaustive search, as stated in Proposition~\ref{proposition:polynomially_many_vertices}.
\begin{proposition}\label{proposition:polynomially_many_vertices}
Consider $\mathcal{X}$ a polytope with set of vertices $\mathcal{V}$, then $P_B$ can be solved in time $O(d^2|\mathcal{V}|)$ by computing $x^\star \in \arg\max_{x \in \mathcal{V}} \{ x^\top c + \|x\|_{W^{-1}} \}$ using exhaustive search over $\mathcal{V}$. Then $(x^\star,\theta^\star)$ with $\theta^\star = c + W^{-1} x^\star/\|x^\star\|_{W^{-1}}$  is an optimal solution to $P_B$.
\end{proposition}
\begin{proof}
The objective function $x \mapsto x^\top c + \|x\|_{W^{-1}}$ is convex by
convexity of norms, so its maximum over $x \in \mathcal{X}$ which is a polytope must be
attained at one of the vertices $\mathcal{V}$. The result then follows from applying
Proposition~\ref{proposition:maximization_over_theta}, so that solving $P_X$ is
equivalent to solving $P_B$. 
\end{proof}
\subsection{Centered Ellipsoids}\label{ssec:Centered ellipsoids}
Another special case of interest where $P_B$ can be solved efficiently is when both $\mathcal{X}$ and $\Theta$ are ellipsoids centered at $0$.  
Proposition~\ref{proposition:centered_ellipsoid} shows that one can compute the optimal solution efficiently by computing the dominant eigenvector of a well-chosen matrix, using either eigenvalue decomposition or power-iteration. 
When $\Theta$ is a non-centered ellipsoid, the problem is still efficiently solvable but much more intricate to solve, as shown by our results in the following sections.
\begin{proposition}\label{proposition:centered_ellipsoid}
Assume that $\mathcal{X} = \{x \in \mathbb{R}^d: \|x\|_{A} \le 1 \}$ and $\Theta = \{\theta \in \mathbb{R}^d: \|\theta\|_{W} \le 1\}$ (i.e. $c=0$) are two centered ellipsoids. 
Consider $\psi$ the normalized eigenvector of $(WA)^{-1/2}$ associated to the largest eigenvalue. 
Then $(x^\star,\theta^\star)$ with $x^\star = A^{-1/2} \psi$ and $\theta^\star = W^{-1/2} \psi$ is an optimal solution to $P_B$ which can be computed in time $O(d^3)$ using eigenvalue decomposition.
\end{proposition}
\subsection{$\ell_p$ balls}\label{ssec:Lp balls}
Finally, we show that, if $\mathcal{X}$ is an $\ell_p$ ball with $p > 2$, then $P_B$ is $\mathcal{NP}$-hard to approximate beyond a fixed approximation ratio. 
For instance, for polytopes like the $\ell_{\infty}$ ball, which has $\smash{2^d}$ vertices, $P_B$ is provably hard.
Therefore, in this case, one cannot implement any optimistic algorithm (or any approximate version of those) in high-dimensional linear bandits.
\begin{proposition}\label{proposition:hardness}
	Assume that $\mathcal{X} = B_{p,d}$ with $p > 2$. There exists $\epsilon_p > 0$ such that, solving $P_B$ with approximation ratio $\epsilon_p$ is $\mathcal{NP}$-Hard. The result remains true if $\Theta$ is assumed centered (i.e. $c=0$).
\end{proposition}

\begin{proof}
 If $\cX = B_{p,d}$ and $\Theta$ is centered at $0$, consider the change of variables $v = W^{1/2} \theta$, then $P_B$ is 
\begin{align*}
  \max_{\| x \|_p \leq 1} \max_{\theta^\top W \theta \leq 1} x^\top \theta 
  &=
  \max_{\| x \|_p \leq 1} \max_{ \|v\|_{2} \leq 1} x^\top W^{-1/2} v \\
  &= 
  \max_{\| x \|_p \leq 1} \sqrt{x^\top W^{-1} x}
\end{align*}
which is equivalent to computing the operator norm $\| W^{-1/2} \|_{p \to 2}$ of the matrix
$W^{-1/2}$.
From Theorem 1.3 of~\cite{bhattiprolu2023inapprox} there exists $\epsilon_p$ such that
computing the $p \to 2$ operator norm of a symmetric matrix with approximation ratio
$\epsilon_p$ is $\mathcal{NP}$-hard, for any $p > 2$.
Note that when $p = +\infty$, the maximum cut problem, which is $\mathcal{NP}$-Hard and
conjectured to be not approximable efficiently (see~\cite{khot2007optimal}) reduces to
$P_B$. 
\end{proof}

\section{The \ouralgo Algorithm for Ellipsoids}\label{sec:algs}
\subsection{The~\ouralgo~Algorithm}
\label{ssec:our_algorithm}

We now propose and analyze the~\ouralgo algorithm to solve $P_B$ when $\mathcal{X} = \{x \in \RR^d, \|x\|_A \leq 1 \}$ and $\gls{A}$ is a symmetric positive definite matrix.
Let $\eps>0$ be the desired precision.
Define the eigendecomposition of $A^{1/2} W A^{1/2} = \gls{U}^\top \Lambda U$ such that $\gls{lambdamat}= \diag(\gls{lambda})$ and $\lambda_1 \ge ... \ge \lambda_d$ and $\gls{centerb} = U A^{-1/2} c$, with its clipped version $\smash{ \gls{clipcenterb} = ({\bf sign}(b_i) \max( |b_i|,\epsilon/(2 \sqrt{d})))_{i \in [d]}}$.

Let us give a rough sketch of the algorithm. 
The idea is to apply a series of reductions to obtain a convex maximization problem \ref{eq:pb2} and to show that the optimal Lagrange multiplier $\gls{mustar}$ for that problem is a solution of the equation.
\[
  \sum_{i=1}^d \frac{\lambda_i b_{+, i}^2}{(\mu \lambda_i - 1)^2} = 1
  \,.
\]
We find a range of values in which we know that the best $\mu^\star$ lies, and show that there is a single solution to the equation in that range. It then suffices to perform binary search inside that range, and to deduce the final value from $\mu^\star$ thanks to the KKT equations.

\begin{algorithm} 
  \SetKwInOut{Input}{Input}
  \Input{Matrices $A, W \in \mathbb{R}^{d \times d}$, vector $c \in \mathbb{R}^d$, target accuracy $\epsilon \ge 0$}
  Define the eigenvalue decomposition of $A^{1/2} W A^{1/2} = U^\top \Lambda U$ with
  $\Lambda = {\bf diag}(\lambda)$ and $\lambda_1 \ge ... \ge \lambda_d$ and $b = U
  A^{-1/2} c$ and $b_+ = ({\bf sign}(b_i) \max( |b_i|,\epsilon/(2 \sqrt{d})))_{i \in [d]}$\\
  Let $\ell = \lambda_{d}^{-1} (\lambda_{d}^{1/2} |b_{+, d}| + 1)$, 
  let $r = \lambda_{d}^{-1} (\sqrt{ \sum_{i=1}^{d} \lambda_{i} b_{+, i}^2} + 1)$, 
  and
  $
  J = \log_2 \bigg(
      \frac{2}{\eps\lambda_{d} b_{+, d}^2}\sqrt{\sum_{i=1}^d \lambda_i^2 b_{+, i}^2}
        \bigg( 
          \sqrt{\sum_{i=1}^d \lambda_i b_{+, i}^2}
          - \sqrt{ \lambda_{d} b_{+, d}^2}
          \bigg)
        \bigg)
  $
   \\
  \For{$j=1,...,J$}{
    {\bf If} $\displaystyle \sum_{i=1}^{d} \lambda_i b_{+, i}^2 
    \Big(\frac{\ell+r}{2} \lambda_i - 1\Big)^{-2}
     \le 1$ 
    {\bf then }$\displaystyle r \leftarrow \frac{\ell+r}{2}$ 
    {\bf else} $\displaystyle \ell \leftarrow \frac{\ell+r}{2}$ 
  }
  Let $\hat{\mu} = r$ and $\hat{\phi} = \big(\frac{\hat{\mu} \lambda_i b_{+, i}}{\hat{\mu}
  \lambda_i-1 }\big)_{i \in [d]}$ and $\hat{x} = A^{-1/2} U^\top
  \hat{\phi}/\|\hat{\phi}\|_{2}$ and $\hat{\theta}=A^{1/2} U^\top \hat{\phi}$ \\
  {\bf Output:} $(\hat{x}, \hat{\theta})$ an $\epsilon$-optimal solution to $P_{B}$
  \caption{The\gls{ouralgo}Algorithm}\label{alg:main-alg} 
\end{algorithm}

\begin{remark}[Computational complexity]
Running~\ouralgo requires time 
\[
O \bigg( 
  d^3 
  + d \log_2 \bigg( \frac{1}{\epsilon\lambda_{d} b_{+,d}^2} \bigg(\sum_{i=1}^d \lambda_i^2 b_{+, i}^2\bigg)^{1/2} |M-m| \bigg) \bigg)\,
  \]
  where $\gls{lowermu} = \lambda_{d}^{-1} (\lambda_{d}^{1/2} |b_{+, d}| + 1)$ and $\gls{uppermu} = \lambda_{d}^{-1} \sqrt{\sum_{i=1}^{d} \lambda_{i} b_{+, i}^2} + \lambda_{d}^{-1} $ and space $O(d^2)$ because of the change of basis plus the computation of $\hat{\phi}$.
Consequently, optimistic algorithms can be implemented in time $T$ times the time complexity of ~\ouralgo~ and space $O(d^2)$ for linear bandits over ellipsoids. This time complexity also beat any semi-definite programming approach as we do not need to check for positive semi-definiteness of a $2d \times 2d$ matrix.
\end{remark}

\subsection{Algorithm Rationale and Correctness}
\label{ssec:Algorithm rationale}

Theorem~\ref{thm:Correctness} states that~\ouralgo~is correct and outputs an $\gls{epsilon}$-optimal solution to $P_B$, and the proof illuminates its rationale.
The complete proof is involved and presented next.
\begin{theorem}[Correctness of~\ouralgo]\label{thm:Correctness}
	\ouralgo~outputs $(\hat{x}, \hat{\theta})$ an $\epsilon$-optimal solution to $P_B$, in the sense that $\hat{x}^\top \hat{\theta} \geq {x^\star}^\top \theta^\star - \epsilon$ 
\end{theorem}
\begin{proof}
The proof comprises five main steps. 

\underline{Step 1: Reduction to euclidean norm maximization} Recall that problem $P_B$ is 
\begin{align*}
\text{maximize } x^\top \theta \text{ subject to } \|x\|_{A} \le 1 \text{ and } \|\theta-c\|_{W} \le 1
\end{align*}
and consider the change of variables $ \gls{u}= U A^{1/2} x$ and $\gls{phi} = U A^{-1/2} \theta$, we have $x^\top \theta = u^\top \phi$, as well as $x^\top A x = u^\top u$ and we can verify that
\begin{align*}
	(\theta-c)^\top W (\theta-c) 
	&= (\phi-b)^\top \Lambda (\phi-b) 
  \,.
\end{align*}
So problem $P_B$ is equivalent to solving $\gls{pbprime}$ defined as
\begin{align*}
  \tag{$P_B'$}
	\text{max } u^\top \phi 
  \text{ s.t. } \|u\|_{2} \le 1 \text{ and } \|\phi-b\|_{\Lambda} \le 1
  \,.
\end{align*}
From Cauchy-Schwarz, for any fixed $\phi$, the maximum of $u^\top \phi$ subject to $\|u\|_{2} \le 1$ is $\| \phi \|_2$, and is attained at $\gls{ustar}= \phi/\|\phi\|_{2}$ so $P_B'$ can be reduced to $P_B''$ defined as
\begin{align}
  \label{eq:pb2}
  \tag{$P_B''$}
	\text{maximize } \|\phi\|_{2} \text{ s.t. } \|\phi-b\|_{\Lambda}^2 \le 1
  \,.
\end{align}
We emphasize that if $\gls{phistar}$ is an optimal solution to $P_B''$, then $(x^\star,\theta^\star)$ is an optimal solution to $P_B$ with $\smash{\theta^\star=A^{1/2}U^\top \phi^\star}$, and $\smash{x^\star = A^{-1/2} U^\top \phi^\star/\|\phi^\star\|_{2}}$.
It turns out that $P_{B}''$ can become unwieldy in some degenerate cases where some entries of $b$ are null. 
Define $b_+ \in \mathbb{R}^d$ as $b_{+,i} = \sign(b_i) \max( |b_i|,\epsilon/(2 \sqrt{d}))$ a thresholded version of $b$ and the optimization problem
\begin{align*}
  \tag{$P_B'''$}
	\text{maximize } \|\phi_+\|_{2} \text{ s.t. } \|\phi_+-b_+\|_{\Lambda}^2 \le 1
  \,.
\end{align*}
For any feasible solution $\phi$ of $P_{B}''$ define $\phi_+ = \phi + b_+ - b$ which is a feasible solution for $P_B'''$ since $\|\phi_+-b_+\|_{\Lambda}^2 = \|\phi-b\|_{\Lambda}^2 \le 1$, and
\begin{align*}
	\|\phi_+\|_{2}  \ge \|\phi\|_{2} - \|b-b_+\|_{2} \ge  \|\phi\|_{2} - \epsilon/2 
\end{align*}
since $|b_i - b_{+, i}| \le \epsilon/(2 \sqrt{d})$ for all $i$. 
This implies that any $\epsilon/2$-optimal solution for $P_{B}'''$ is an $\epsilon/2$-optimal solution for $P_B''$.

\underline{Step 2: Restriction to $\phi_+ \odot b_+ \ge 0$} We now focus on solving $P_B'''$. 
Consider $\phi_+ \in \mathbb{R}^d$, and define $\gls{phiprime}$ such that $\phi_i' = -\phi_i$ if $\phi_{+,i} b_{+,i} < 0$ and $\phi_{+,i}' = \phi_{+,i}$ otherwise.
One may readily check that $\|\phi_+\|_{2}^2 = \|\phi_+'\|_{2}^2$ and $\|\phi_+'-b_+\|_{\Lambda}^2 \le \|\phi_+-b_+\|_{\Lambda}^2$, so $\phi_+'$ is at least as good as $\phi_+$.
So we can restrict our attention to $\phi_+$ such that $\phi_+ \odot b_+ \ge 0$ to solve $P_B'''$.  

\underline{Step 3: KKT Conditions} 
Since $P_B'''$ is a strongly convex maximisation problem, any local solution saturates the constraint. 
Moreover, the KKT conditions \cite{luenberger1984linear}[Chapter 11.8] guarantee that at any optimal solution $\gls{phistarplus}$ (which necessarily activates the constraint), there exists $\gls{mu} \geq 0$ such that
\begin{equation}
  \label{eq:KKT-condition}
	\phi \odot (\mu \lambda-{\bf 1}) 
  = \mu \lambda \odot b_+
  \quad \text{and} \quad
  \|\phi_+ - b_+\|_\Lambda = 1
  \,.
\end{equation}
We now look for possible solutions to these optimality conditions $(\phi_+, \mu)$.

\underline{ Step 4: Restricting to $\mu \geq \lambda_{d}^{-1}$} 
  Multiplying by $b_+$ on both sides,
\begin{align*}
 \phi_+ \odot b_+ \odot (\mu \lambda - \mathbf 1) =  \mu \lambda \odot b_+ \odot b_+ \,.
\end{align*}
And since $\lambda > 0$ and $\mu \ge 0$, we have $\phi_+ \odot b_+ \odot (\mu \lambda-{\bf 1}) \ge 0$.
This implies that $\mu \lambda_i \ge 1$ for all $i$ since $b_{+,i} \ne 0$.
Therefore $\mu^\star \geq \lambda_{d}^{-1}$ (remember the $\lambda$ are ordered in non-increasing order).

\underline{Step 5: Solving the KKT conditions} If $(\phi_+, \mu)$ satisfy the KKT conditions and saturate the constraint
\begin{align*}
  \phi_+ 
  =  \phi_+(\mu) 
  =  \left(\frac{\mu \lambda_i b_{+, i}}{\mu \lambda_i-1 }\right)_{i \in [d]}  \quad \text{and} \\
  1 
  = \|\phi_+(\mu) - b^+ \|_{\Lambda}^2 
=  \sum_{i=1}^{d} \frac{\lambda_i b_{+, i}^2}{(\mu \lambda_i-1)^2} 
  \,.
\end{align*} 
The r.h.s of this second equation is strictly decreasing in $\mu$ for $\mu \geq \lambda_{d}^{-1}$ (since $\mu \lambda_i > 1$ for all $i$ in the sum), and therefore admits a unique solution $\smash{\mu^\star \in (\lambda_d^{-1}, +\infty)}$. 
Hence the KKT conditions admit a single solution $(\phi_+, \mu)$ which has to be $(\phi_+(\mu^\star), \mu^\star)$.
Let us further restrict the interval in which $\mu^\star$ lies. 
From the inequalities
\begin{align*}
& \frac{1}{(\mu^\star \lambda_{d}-1)^2}  \sum_{i=1}^{d} \lambda_{i} b_{+, i}^2
  \ge \sum_{i=1}^{d} \frac{\lambda_i b_{+, i}^2}{(\mu^\star \lambda_i-1)^2} 
  = 1 \\
  & 1 =\sum_{i=1}^{d} \frac{\lambda_i b_{+, i}^2}{(\mu^\star \lambda_i-1)^2} 
  \ge  \frac{\lambda_{d} b_{+, d}^2}{(\mu^\star \lambda_{d}-1)^2} \,,
\end{align*}
we can deduce that $\mu^{\star} \in [m,M]$ with $m = \lambda_{d}^{-1}
(\lambda_{d}^{1/2} |b_{+, d}| + 1)$ and $M = \lambda_{d}^{-1} \sqrt{
\sum_{i=1}^{d} \lambda_{i} b_{+, i}^2} + \lambda_{d}^{-1} $. 
We approximate $\mu^\star$ by binary search over this interval. 
Now, let us bound the error introduced by binary search.
Denote by $\gls{muhat}$ the approximate value of $\mu^\star$ computed by $\gls{dichonumber}$ iterations of binary search and $\hat{x},\hat{\theta},\hat{\phi}, \hat{\phi}_+$ the corresponding approximate values of $x^\star,\theta^\star,\phi^\star,\phi_+^\star$. 
Then
\begin{align*}
 {x^\star}^\top \theta^\star - \hat{x}^\top \hat{\theta} &= \|\phi^\star\|_{2} - \|\hat{\phi}_+\|_{2} \\
 & = \|\phi^\star\|_{2} - \|\hat{\phi}_+ + \phi_+^\star -  \phi_+^\star \|_{2} \\
& \leq \|\phi^\star - \hat{\phi}_+ - \phi_+^\star +  \phi_+^\star \|_{2} \\
& \leq \|\phi^\star - \phi_+^\star \|_{2} + \| \hat{\phi}_+ - \phi_+^\star \|_{2}  \\
& \leq \epsilon/2 + \| \phi_+(\hat{\mu}) - \phi_+(\mu^\star) \|_{2} \, .
\end{align*}
And then, 
\[\| \phi_+(\hat{\mu}) - \phi_+(\mu^\star) \|_{2} \leq  |\mu^\star - \hat{\mu}| \sup_{\mu \in [m,M]}  \|\phi'(\mu) \|_2 \, ,
\]
After $J$ iterations of binary search $|\mu^\star - \hat{\mu}| \le 2^{-J}(M-m)$.
Now, since $\mu \geq m$, and $\lambda_i \geq \lambda_d$,
\begin{align*}
  \|\phi'(\mu) \|_2^2
  &= 
  \sum_{i=1}^{d} \frac{(\lambda_i b_{+,i}(\mu \lambda_i -1) - \mu \lambda_i^2 b_{+,i})^2}
  {(\mu \lambda_i - 1)^4}  \\
  &= 
  \sum_{i=1}^{d} \frac{(\lambda_i b_{+,i})^2}
  {(\mu \lambda_i - 1)^4}  
  \le \frac{1}
  {\lambda_d^2 b_{+,d}^4}
  \sum_{i=1}^{d} \lambda_i^2 b_{+,i}^2
  \,.
\end{align*}
Replacing and setting 
\[
J 
\geq \log_2 \Big(2 |M-m| (\epsilon \lambda_{d} b_{+,d}^2)^{-1} 
  \sqrt{ \sum_{i=1}^{d} \lambda_i^2 b_{+,i}^2}
\Big) \, ,
\] we have proven that
\begin{equation*}
  {x^\star}^\top \theta^\star - \hat{x}^\top \hat{\theta} 
  \le \frac{ 2^{-J} |M-m| }
  {\lambda_{d} b_{+,d}^2}\sqrt{ \sum_{i=1}^{d} \lambda_i^2 b_{+,i}^2}  + \frac{\epsilon}{2}
  \leq \epsilon \,.
  \qedhere
\end{equation*}
\end{proof}

\section{The \oursecondalgo Algorithm for Ellipsoids}
\label{sec:convex_newton}

We describe a second approach to solve \ref{prob:bilinear}, based on a reduction to a convex minimization problem. This second approach yields another algorithm to solve $P_B$ when $\cX$ is a centered ellipsoid. We then apply this second method when $\cX$ is an $\ell_p$ ball and $\Theta$ is aligned with the axes.
\subsection{$\cX$ is an ellipsoid}
We reduce \problem to a convex optimization problem, \ref{prob:reduction_ball}, defined below, and analyse an interior-point method tailored to \gls{Pconv}. 
We present the full procedure in Algorithm~\ref{alg:reduct-opt}, in Appendix~\ref{app:reduct-alg}.
\begin{align}
  \label{prob:reduction_ball}
  \tag{$P_C$}
  \min_{y \in \Delta_{d-1}} \quad
  & F(y) \notag \\
  \text{where} \quad
  & \gls{F}(y) :=
    -\sum_{i=1}^{d} \sqrt{y_i} \, |b_i|
    - \sqrt{
     \sum_{i=1}^{d} \lambda_i^{-1} y_i}
  \,, \nonumber
\end{align}

where \gls{simplex} is the simplex of dimension $d-1$. 
To solve $P_C$, we suggest an algorithm based on Newton's method. 
The objective function $F$ is almost self-concordant \cite{nesterov_lectures_2018}: it satisfies the self-concordance differential inequality when restricted to a subdomain of the intersection of the $\ell_1$-ball and the positive orthant. 
Based on this observation, we perform an initial centering step using the Damped Newton algorithm, applied to a log-barrier regularized objective on this restricted domain. 
Then, we use an interior point method by successively applying Damped Newton steps to objective functions with exponentially decreasing levels of regularization. 
The complete analysis is in Theorem~\ref{thm:complexity_reduction_algo}, App.\ref{app:proof_time_complexity_reduc_algo}, with a short version in Theorem~\ref{thm:reduct-full-analysis-eta-d}.

The next theorem, proved in Appendix~\ref{app:proof-reduction-to-conv} is a change of variables.
Extra steps are required to ensure sign constraints on the variables and to transform the inequality constraint into an equality.
\begin{theorem}
  \label{thm:reduction-to-conv}
  Let $U$ be an orthogonal matrix that diagonalizes $A^{1/2}WA^{1/2}$ and $b = U^\top
  A^{-1/2}c$. Given $y \in \Delta_{d-1}$, for all $i\in[d]$ define 
  $
   \smash{ u_i = \sqrt y_i \sign(b_i) }
  $, together with the pair $(x, \theta^\star) \in \cX \times \Theta$:
  \[
    x = A^{-1/2}U^\top u 
    \quad \text{and} \quad 
    \theta^\star = c + W^{-1}x / \|x\|_{W^{-1}}
    \,.
  \]
  Then, for any $\eps > 0$, the vector $y$ is an $\eps$-solution to
  \ref{prob:reduction_ball} if and only if $(x , \theta^\star)$ is an $\eps$-solution to \problem.
\end{theorem}

\begin{theorem}
  \label{thm:reduct-full-analysis-eta-d}
  For $\eps$ small enough, there exists an algorithm (Algorithm~\ref{alg:reduct-opt} in Appendix~\ref{app:newton}) that outputs an $\eps$-solution in less than $N_{tot}(b, \lambda, \epsilon)$ Newton steps, with
  \begin{multline*}
  \gls{ntot}(b, \lambda, \epsilon) 
  < 38t_0
     \big(\|b\| + \lambda^{-1/2}_d\big) \\
    + 38d \log\Big(2d\lambda^{-1}_d(\|b\| + \lambda^{-1/2}_d)^2\Big) \\
    + 38\log\bigg( 
     \frac{\lambda_d(\|b\| + \lambda^{-1/2}_d)^2}{2\|b\|^2 + 1/2} 
     \bigg)  \\
     + \left(19 +  \log_2\log_2\Big(\frac{2}{\epsilon}\Big) \right) 
     \bigg\lceil  \sqrt{d+1}\log\Big(\frac{2(d+1)}{\epsilon t_{0}}\Big) + 1 \bigg\rceil
     \, ,
    \end{multline*}
    where $\gls{tzero} = 9 \max\Big({\displaystyle \max_{i \in \mathcal{I}}}(b_i^2 B_i)^{-1/2}
    , \, {\displaystyle \min_{i \in \mathcal{I}}}(\lambda_i^{-1} B_i)^{-1/2}
    \Big)$,
    $ \gls{Bi} = 
      b_i^2/ (\|b\| + \lambda^{-1/2}_d)^2
    $ and  $\gls{bset} = \{i \in [d]: b_i \neq 0\}$.
\end{theorem}

\subsection{$\cX$ is the $\ell_p$ ball and $W$ is diagonal}
\label{sec:axes-aligned}
Despite the hardness of the $P_B$ when $\cX$ is a convex $\ell_p$ ball and $\Theta$
an arbitrary ellipsoid, if the ellipsoid $\Theta$ is aligned with the axes, and $p \geq
2$, then $P_B$ becomes tractable again. Consider
  \begin{align}
    \label{prob:reduction_ball_ellp}
    \tag{$P_{C, p}$}
   & \min_{y \in \Delta_{d}} 
   H(y) \\
  & \text{where} \quad
  H(y) 
  = - \sum_{i=1}^{d}y_i^{1/p} |c_i| - \sqrt{
   \sum_{i=1}^{d} \lambda_i^{-1} y_i^{2/p}
  }
   \,. \nonumber
  \end{align}
A standard calculation, summarized in Proposition~\ref{prop:h-is-convex}, shows
that $H$ is indeed convex when $p \geq 2$. The next theorem shows that $P_{B}$ 
reduces to $P_{C, p}$; see Appendix~\ref{app:proofs-aligned} for a proof.
\begin{theorem}
  \label{thm:reduc-ellp}
  Let $p \geq 1$ and let $\cX$ be the $\ell_p$ ball. Consider \problem with $W = \Lambda$ diagonal. Given $\smash{y \in \Delta_{d-1}}$, define
  $
    \smash{x_i = y_i^{1 / p} \sign(c_i) }
    \quad \text{and} \quad 
    \smash{\theta^\star = c + W^{-1} x / \|x\|_{W^{-1}}}
    \,.
  $
  Then, for any $\eps > 0$, the vector $y$ is an $\eps$-solution to
  \ref{prob:reduction_ball_ellp} if and only if $(x , \theta^\star)$ is an
  $\eps$-solution to $P_B$.
\end{theorem}

\section{Numerical Experiments}\label{sec:Numerical experiments}

We now assess the numerical performance of \ouralgo~and \oursecondalgo on a variety of instances of $P_B$, with a special attention to instances that arise when running LinUCB and OFUL on linear bandit problems.
We work at machine precision $(10^{-8})$. Experiments are run on a Dell Inc. Precision 5570 laptop with 12th Gen Intel Core i9-12900H $\times$ 20 processor and 32.0 GiB RAM with \cite[Python]{10.5555/1593511} using standard library like \cite[Numpy, Scipy]{harris2020array,2020SciPy-NMeth}. No GPU was used.

\textbf{Instances generated by runs of OLSOFUL \quad}
We generate instances of $P_B$ by running OLSOFUL~\cite{gales2022norm-agn} (a
norm-agnostic version of OFUL) on linear bandit problems with set of actions
$\mathcal{X} = \{x \in \mathbb{R}: \smash{\|x\|_{2} \le 1}\}$, unknown vector $\zeta$
with a fixed norm $\smash{\|\zeta\|_{2} \in \{1,10,50\}}$, and Gaussian rewards with
unit variance. For each run of OLSOFUL, we obtain the least-squares estimates
$\smash{\hat{\zeta}_t}$ and matrices $\smash{W_t}$ which define the confidence
ellipsoids at time step $\smash{t \in [T]}$ with $\smash{T=10^4}$. We then define
an instance of $P_B$ for each $\smash{t \in [T]}$ with input parameters $ \smash{A =
I_d}$ and $\smash{W = W_t}$ and $\smash{c = \hat\zeta_t}$. Figures~\ref{fig:histograms_thetas} and~\ref{fig:histograms_eigenvalues}
present the histogram of the eigenvalues of $\smash{W_t}$, as well as the histogram of
the absolute value of the entries of $\smash{\hat\zeta_{t}}$ expressed in the basis of
the eigenvectors of $\smash{W_t}$. As $\smash{t}$ grows, one of the eigenvalues of
$\smash{W_t}$ dominates all others, all but one entries of $\smash{\hat\zeta_{t}}$
vanish and~$\smash{\hat\zeta_t}$ aligns perfectly with the eigenvector of
$\smash{W_t}$ associated with its largest eigenvalue. This is expected: OLSOFUL is
optimistic and always selects actions in a direction close to $\smash{\hat{\zeta}_t \to
\zeta}$ as $\smash{t \to \infty}$.

\begin{figure}
    \centering
    \begin{tabular}{@{}c@{}c@{}c@{}c@{}}
        \diagbox[width=1.cm, height=0.7cm]{$\boldsymbol{d}$}{$\boldsymbol{t}$}  &  100  & 5000 & 9999 \\
        \multirow{-3}{*}{\centering 5}   &  \includegraphics[width = 0.15\textwidth]{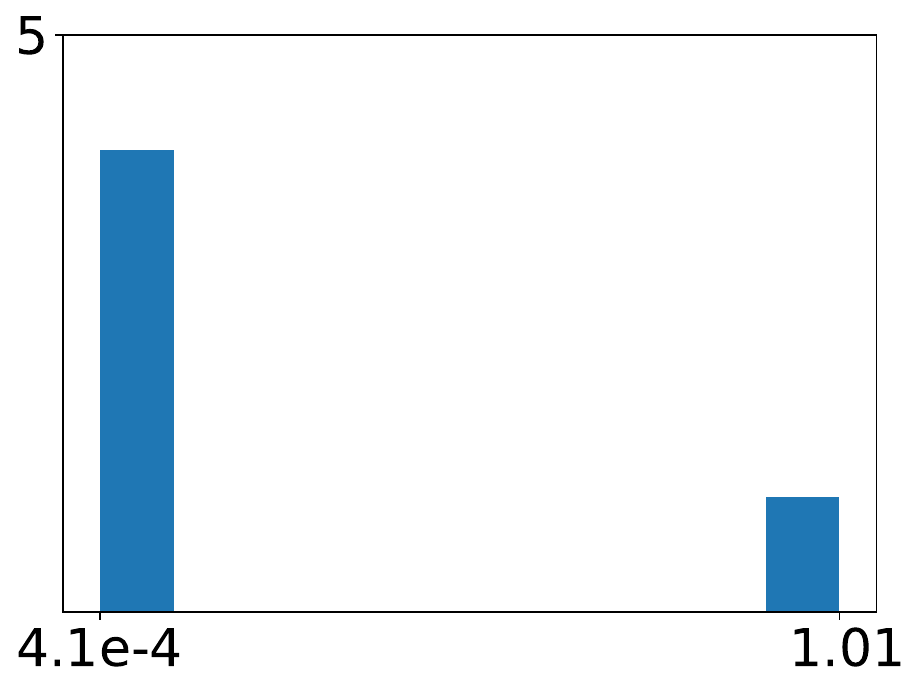} &  \includegraphics[width = 0.15\textwidth]{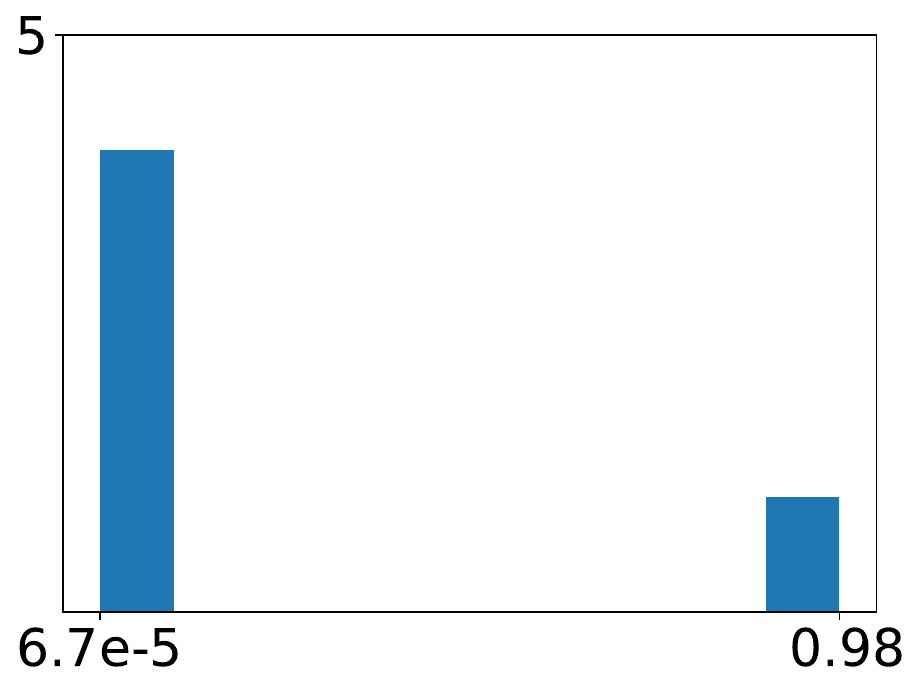} &   \includegraphics[width = 0.15\textwidth]{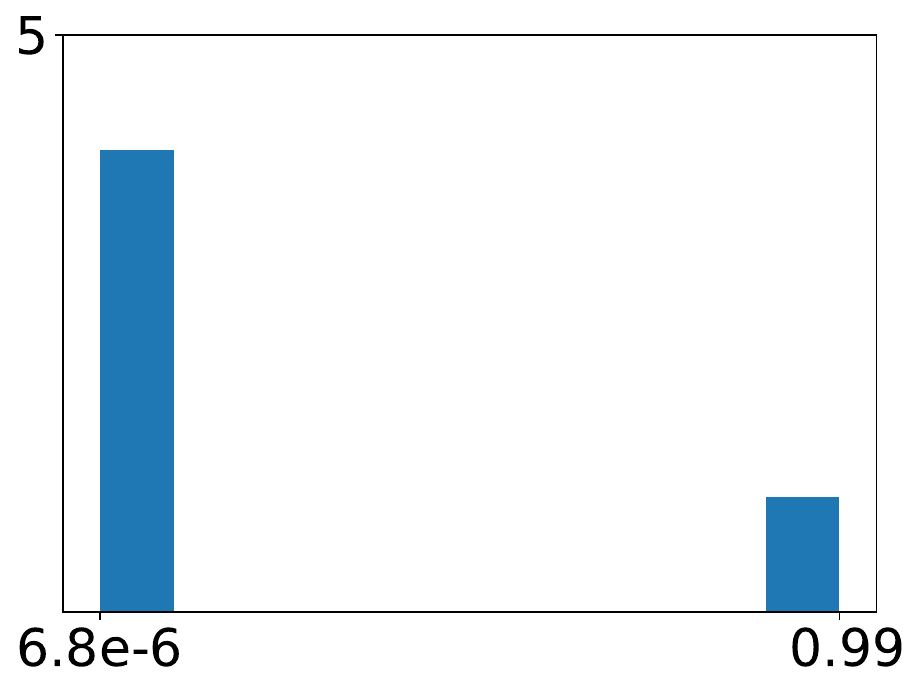} \\
        \multirow{-3}{*}{\centering 30}  &  \includegraphics[width = 0.15\textwidth]{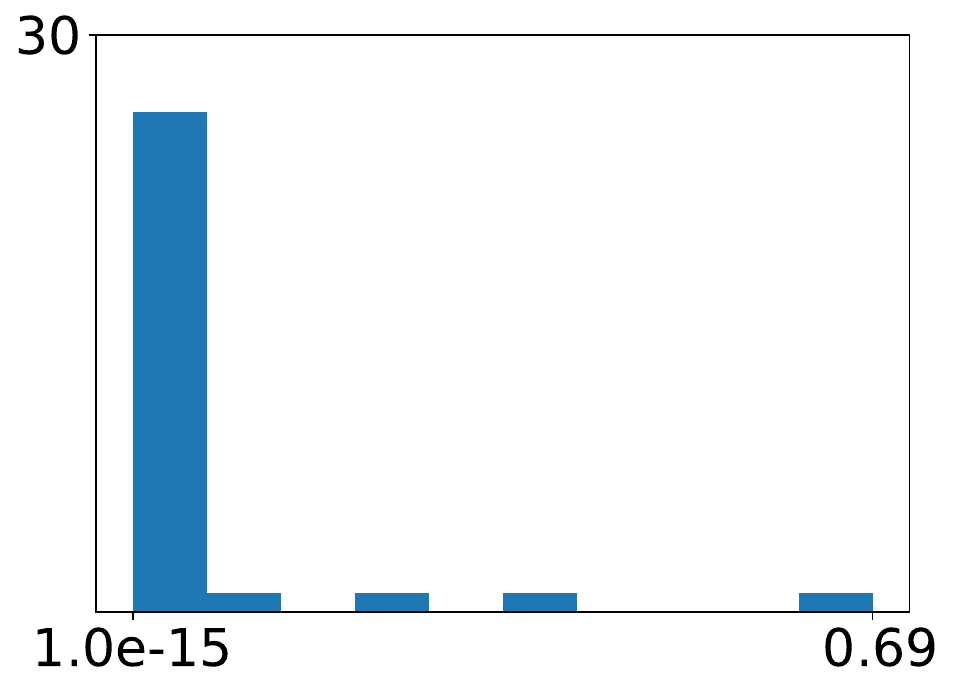} &  \includegraphics[width = 0.15\textwidth]{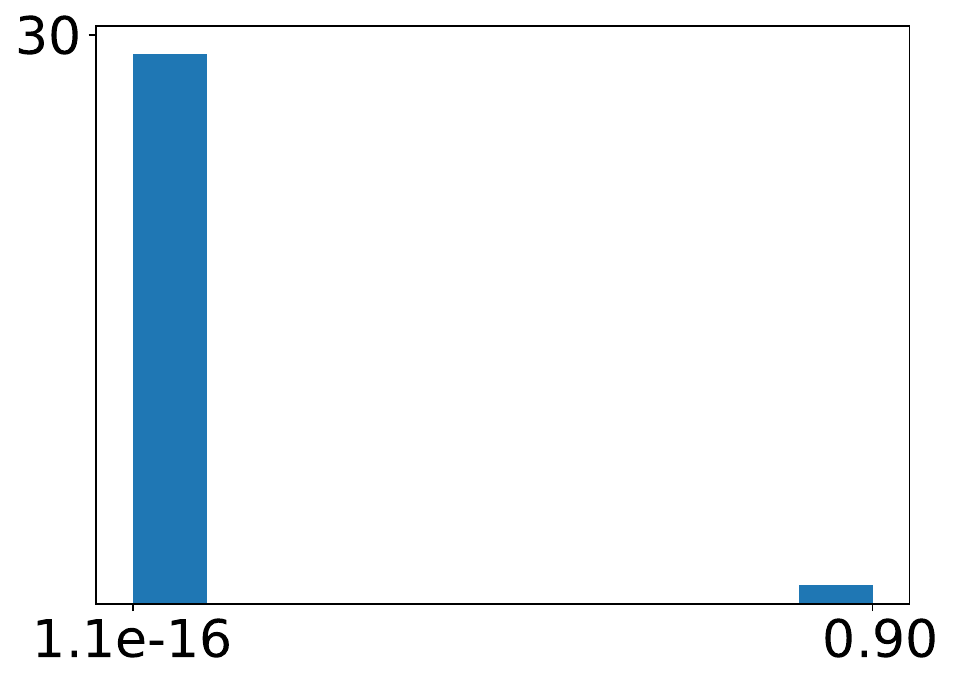} &   \includegraphics[width = 0.15\textwidth]{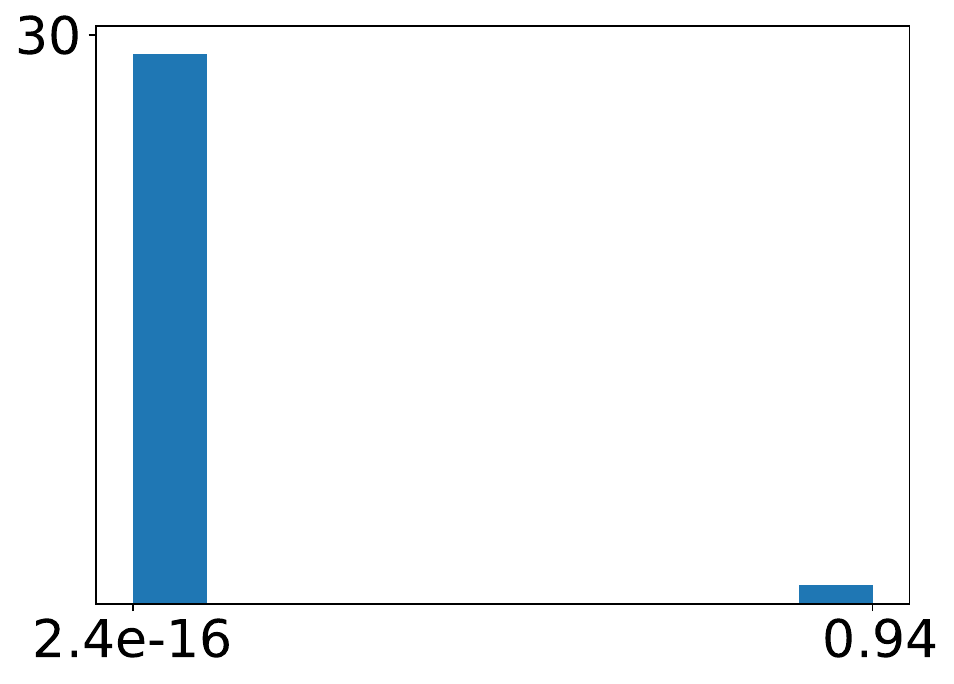} \\
        \multirow{-3}{*}{\centering 200} &  & \includegraphics[width = 0.15\textwidth]{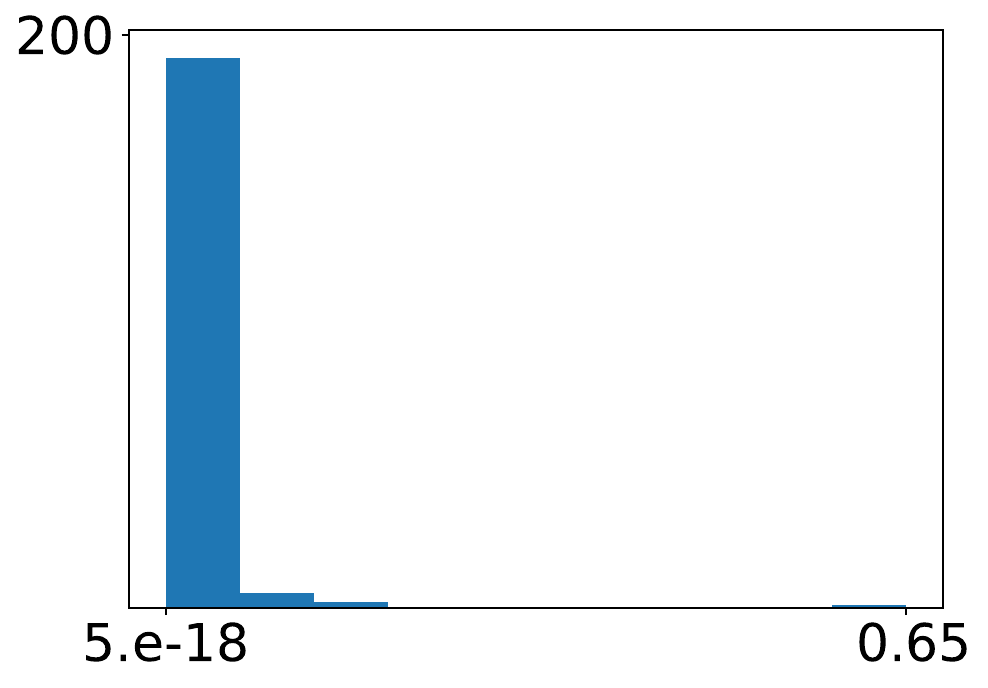} &   \includegraphics[width = 0.15\textwidth]{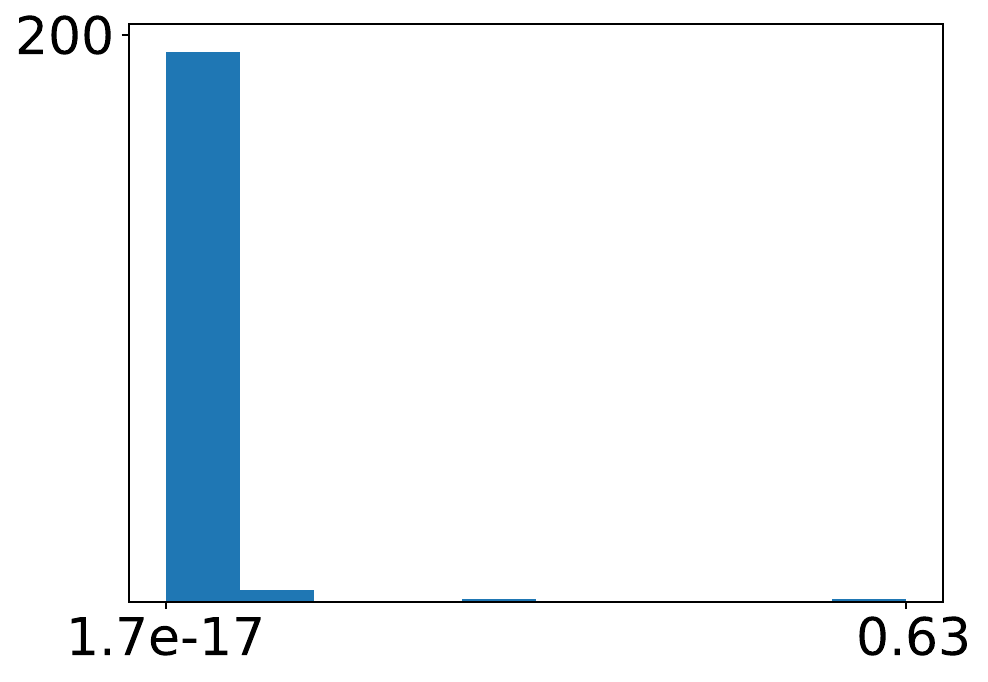} \\
    \end{tabular}
    \caption{Histogram of the entries of the centers $b$ for different dimensions and time.}
    \label{fig:histograms_thetas}
\end{figure}
\begin{figure}
    \centering
    \begin{tabular}{@{}c@{}c@{}c@{}c@{}}
        \diagbox[width=1.cm, height=0.7cm]{$\boldsymbol{d}$}{$\boldsymbol{t}$}  &  100  & 5000 & 9999 \\
        \multirow{-3}{*}{\centering 5}   &  \includegraphics[width = 0.15\textwidth]{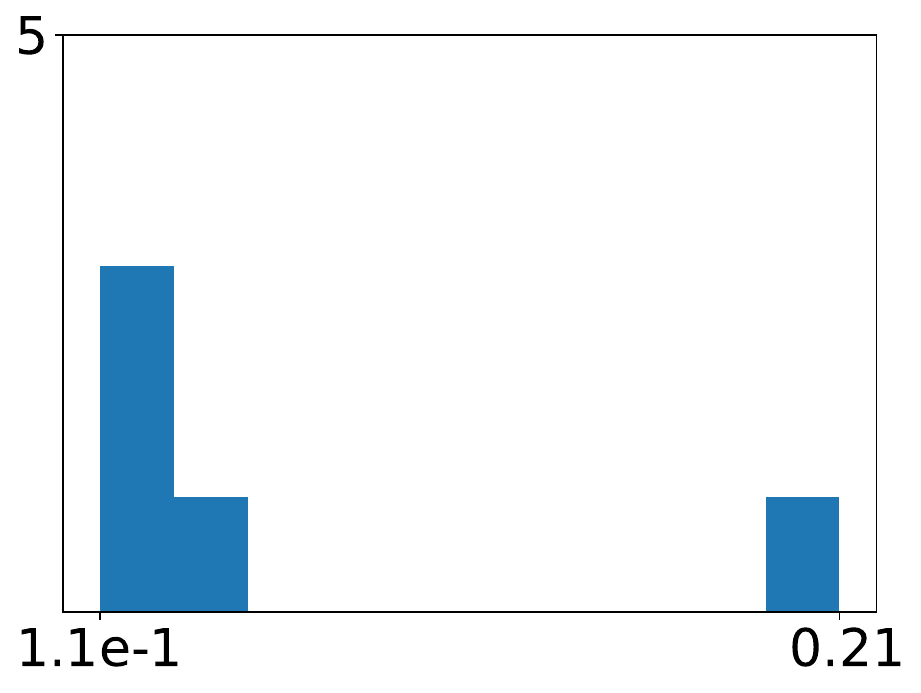} &  \includegraphics[width = 0.15\textwidth]{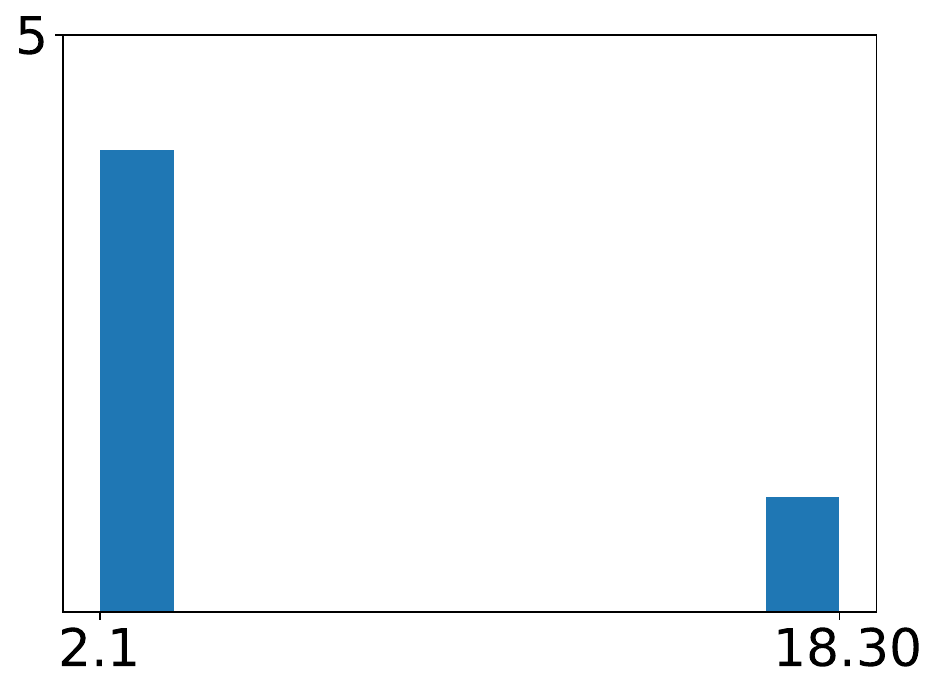} &   \includegraphics[width = 0.15\textwidth]{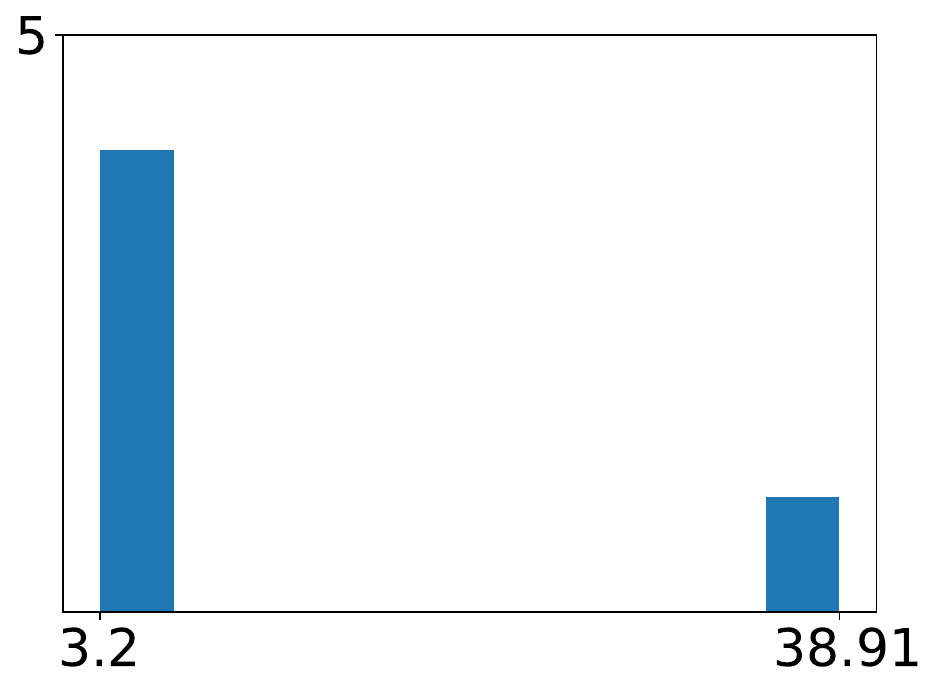} \\
        \multirow{-3}{*}{\centering 30}  &  \includegraphics[width = 0.15\textwidth]{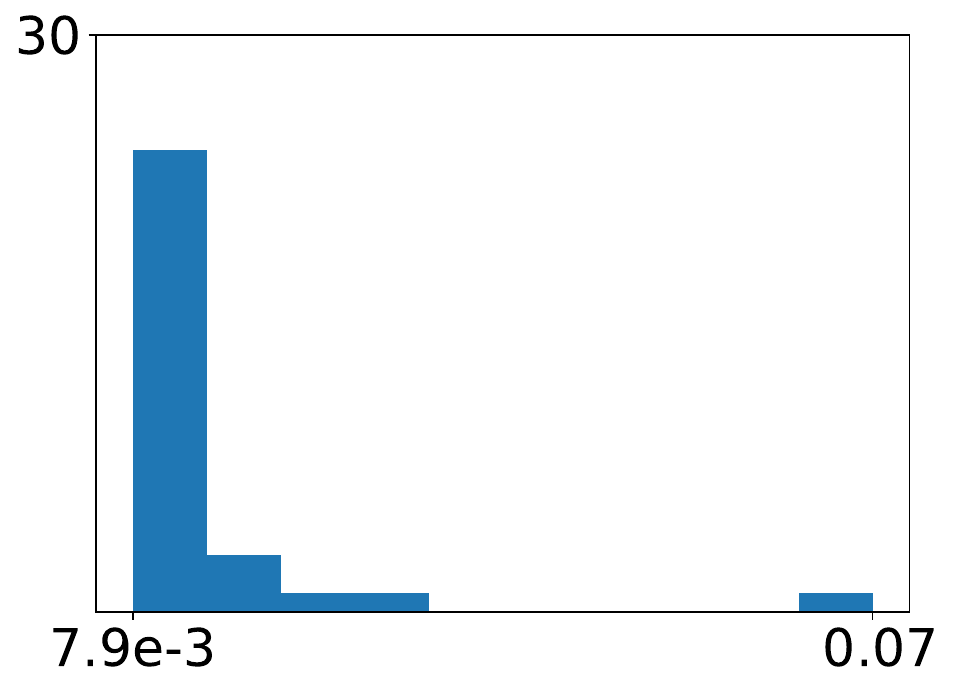} &  \includegraphics[width = 0.15\textwidth]{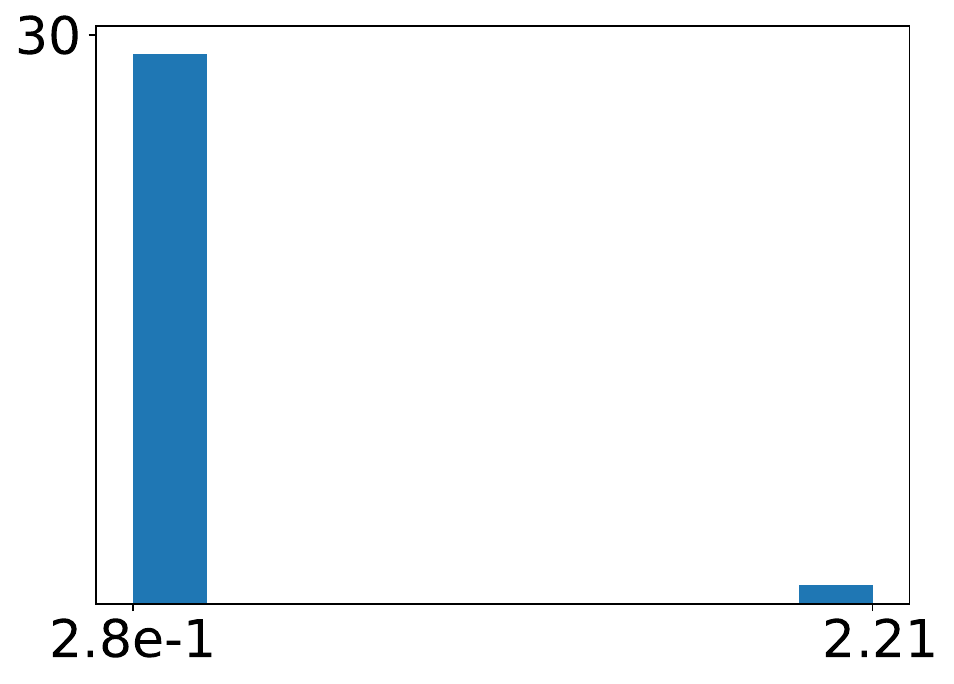} &   \includegraphics[width = 0.15\textwidth]{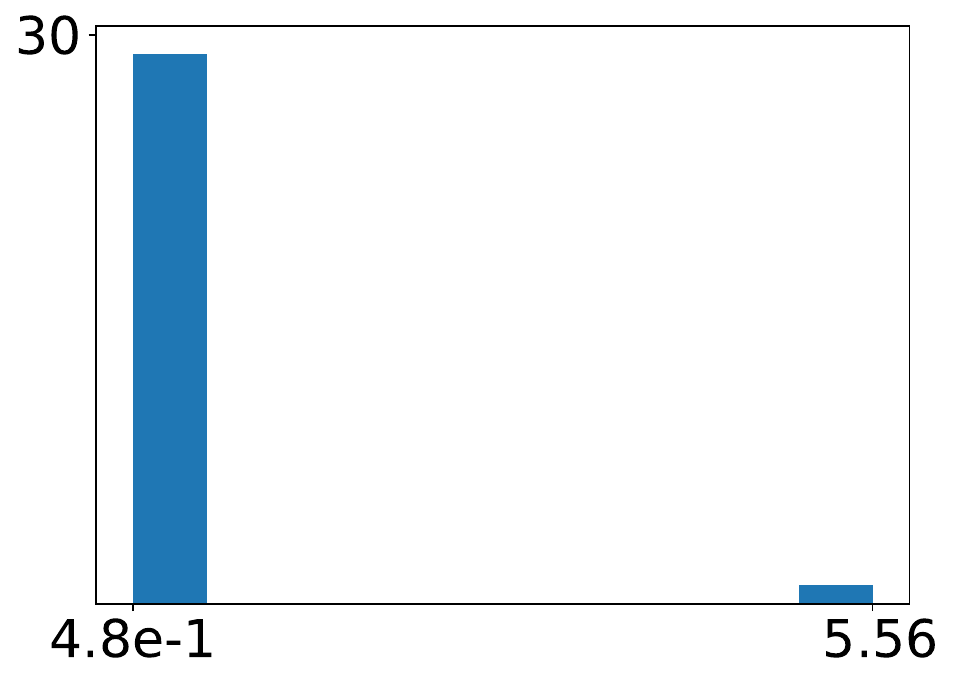} \\
        \multirow{-3}{*}{\centering 200} &  & \includegraphics[width = 0.15\textwidth]{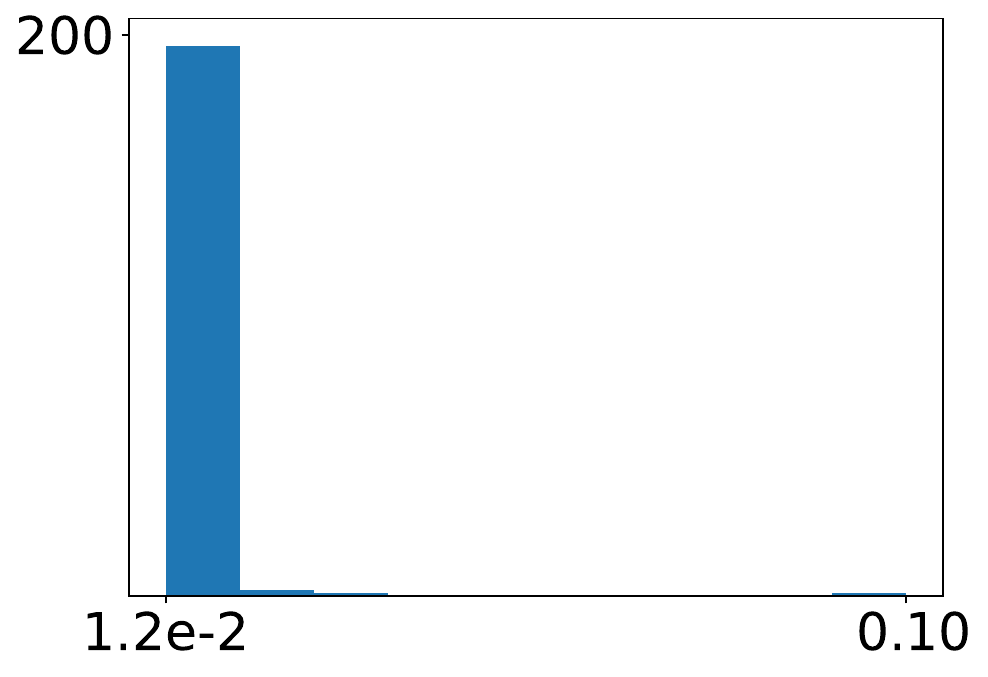} &   \includegraphics[width = 0.15\textwidth]{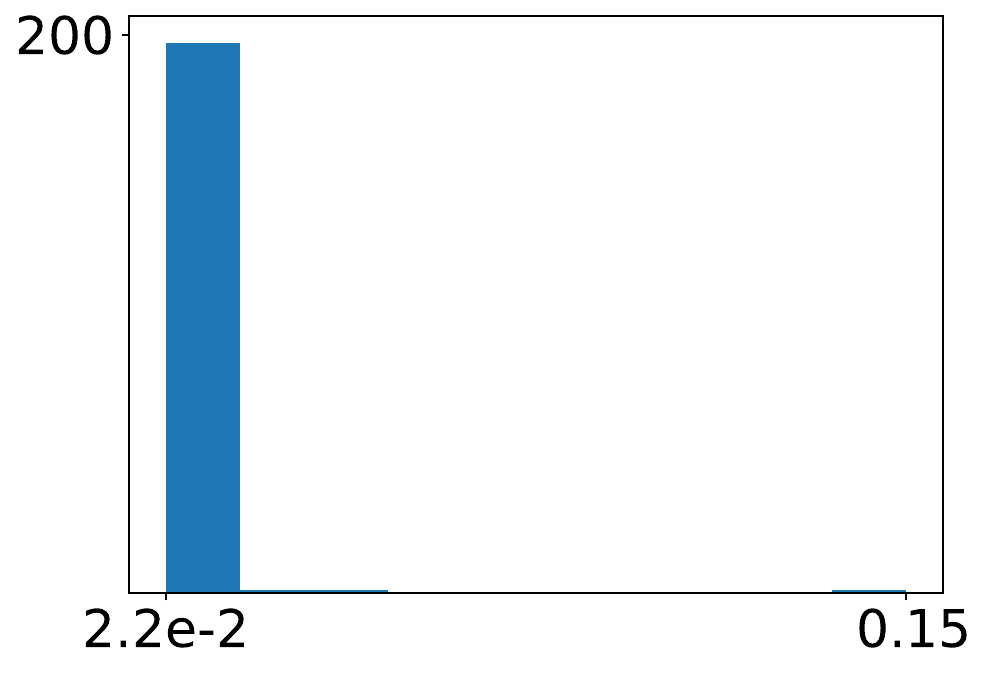} \\
    \end{tabular}
    \caption{Histogram of the eigenvalues for different dimensions and time.}
    \label{fig:histograms_eigenvalues}
\end{figure}

On Figure \ref{fig:objective_value_time}, we consider the instance of $\smash{P_{B}}$
corresponding to time step $\smash{t \in T}$, i.e., $\smash{P_B}$ with input parameters
$\smash{A = I_d}$, $\smash{W = W_T}$ and $\smash{c = \hat\zeta_T}$, and display the
performance of both \ouralgo and \oursecondalgo\!. To understand the relationship
between performance (objective value) and computation (CPU time), given time
$\smash{\tau \in \mathbb{R}^+}$ in x-axis, we represent on the y-axis the value of the
objective function of $P_B$ obtained by each algorithm after running for a duration of
time $\tau$. \ouralgo is an order of magnitude faster than \oursecondalgo\!, and both
algorithms manage to solve very high-dimensional instances with thousands of variables
$d \ge 10^3$. In contrast, off-the-shelf solvers for quadratically constrained bilinear
maximization like \cite[GUROBI]{gurobi}, do not find good solutions even in dimension
$d=3$. We believe this is due to the non-convexity of $P_B$ in its natural form.



\begin{figure}[h]
    \begin{tabular}{c@{}c@{}c@{}c@{}c@{}}
        \diagbox[width=1cm, height=0.7cm, innerleftsep=.2cm,innerrightsep=-0.3cm]{$\boldsymbol{d}$}{$\smash{\|\xi\|_2}$} & 1 &  10 &  50 \\ 
        \multirow{-3}{*}{\centering 5}    & \includegraphics[width = 0.15\textwidth]{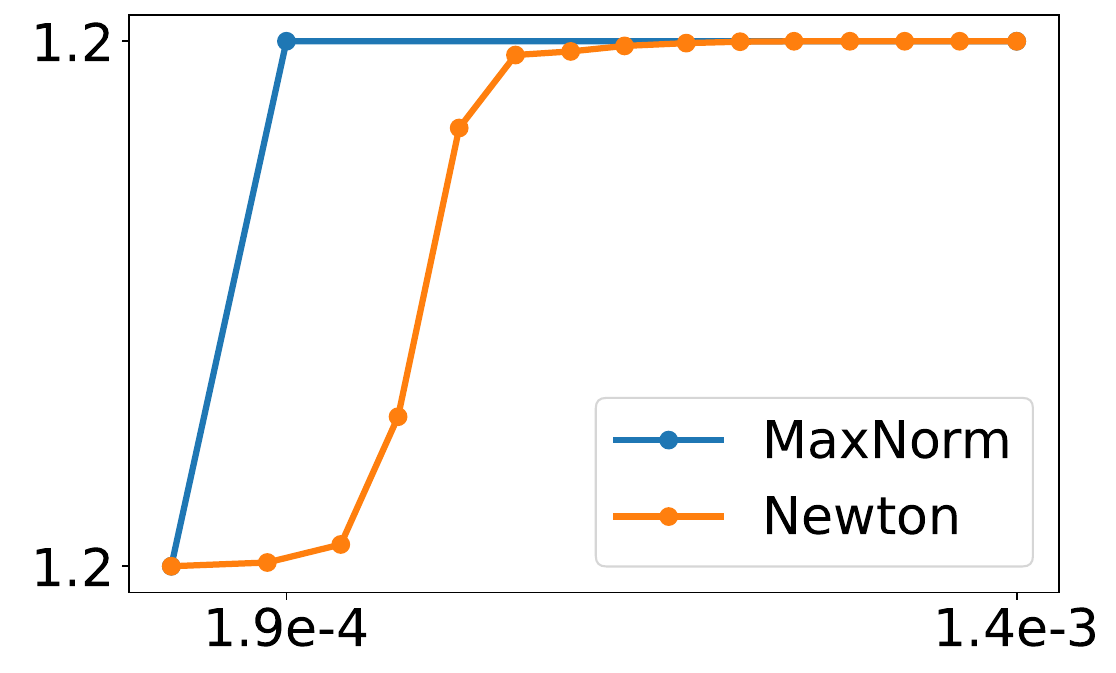} & \includegraphics[width = 0.15\textwidth]{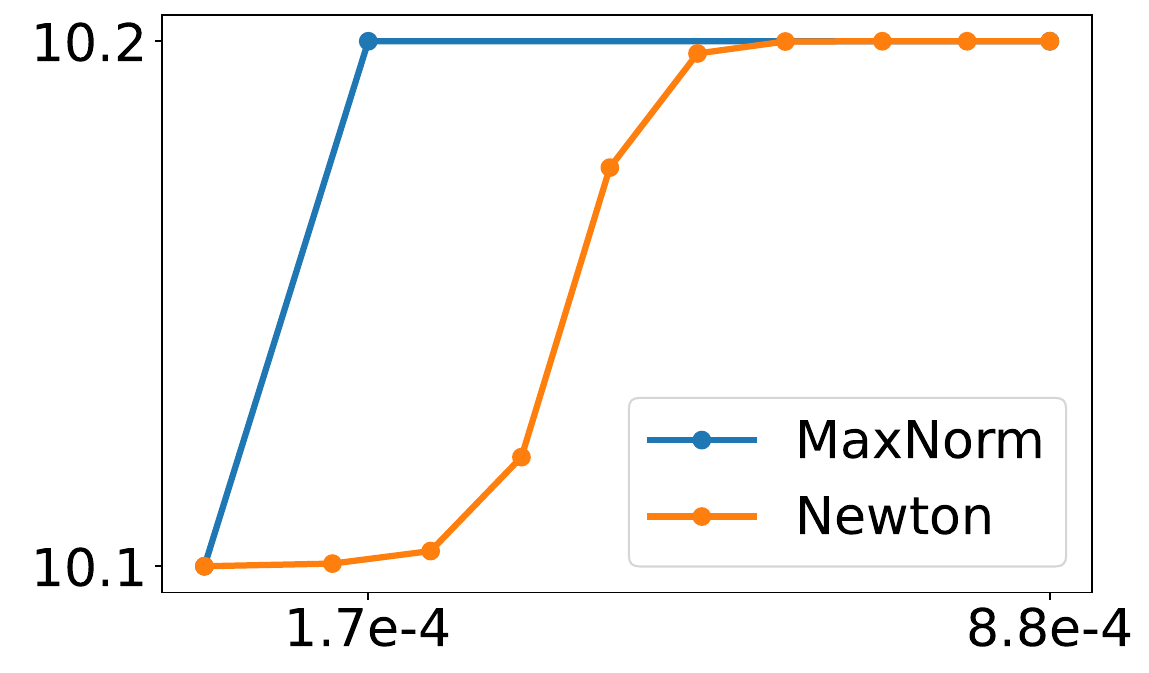} & \includegraphics[width = 0.15\textwidth]{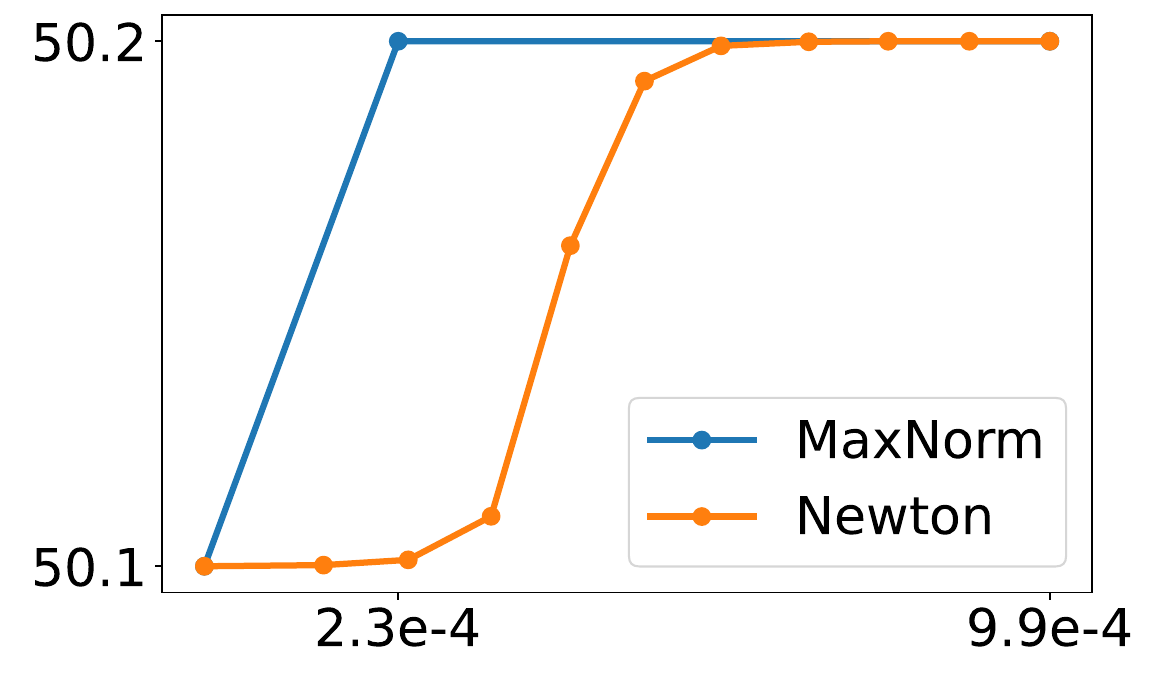} \\
        \multirow{-3}{*}{\centering 30}   & \includegraphics[width = 0.15\textwidth]{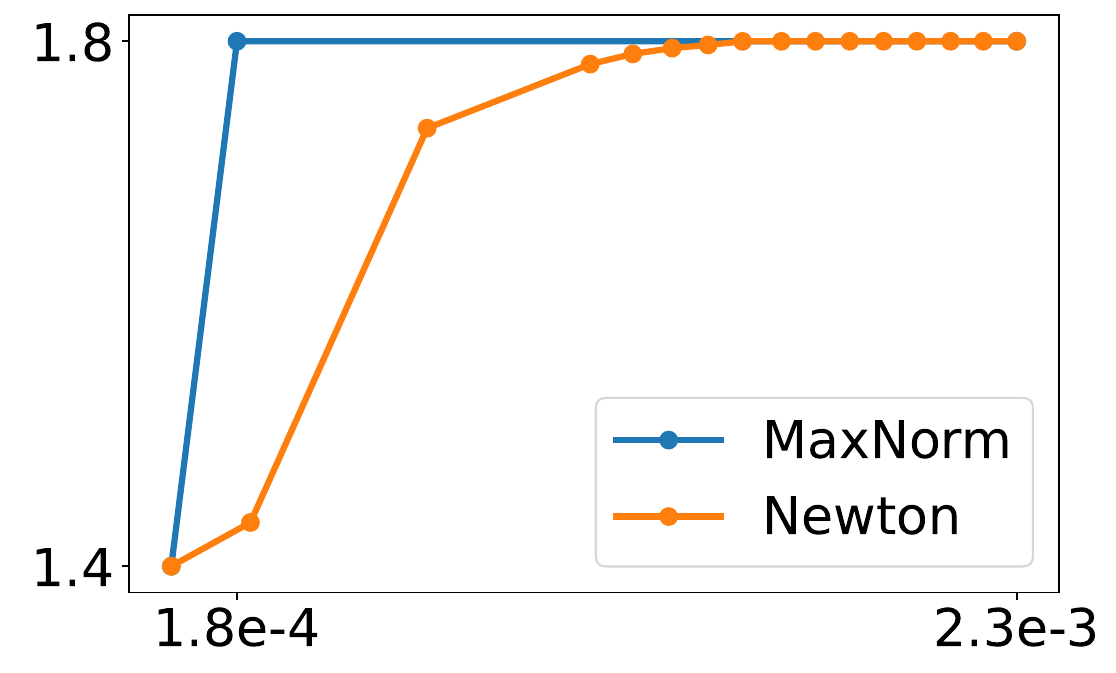} & \includegraphics[width = 0.15\textwidth]{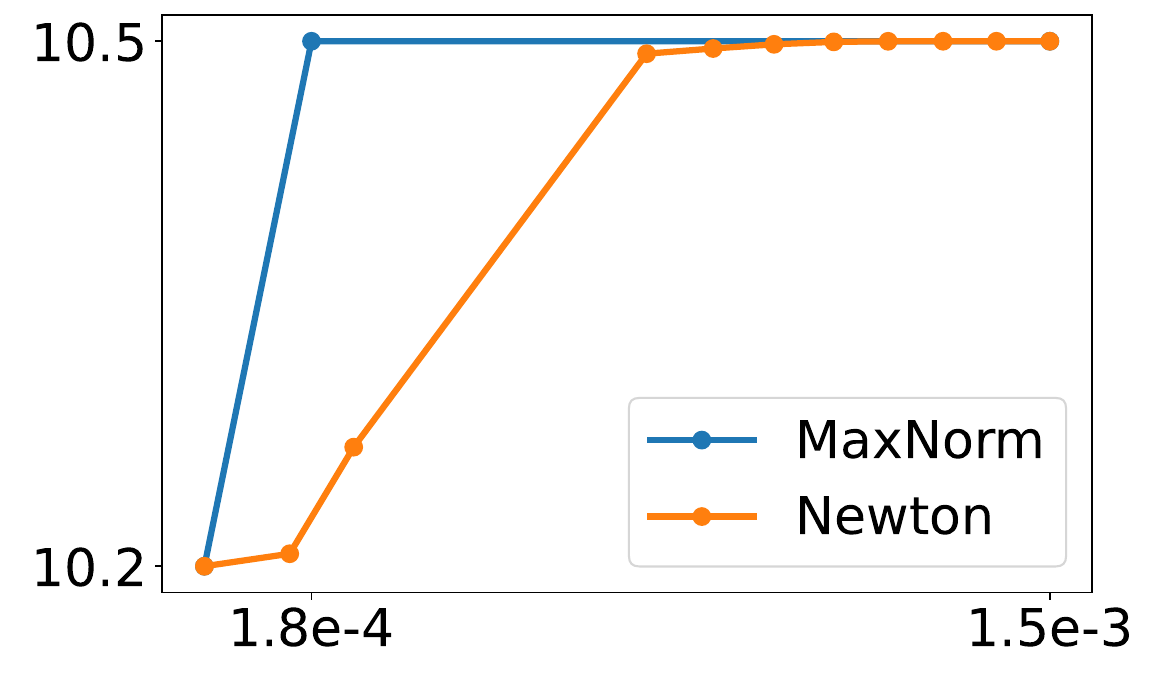} & \includegraphics[width = 0.15\textwidth]{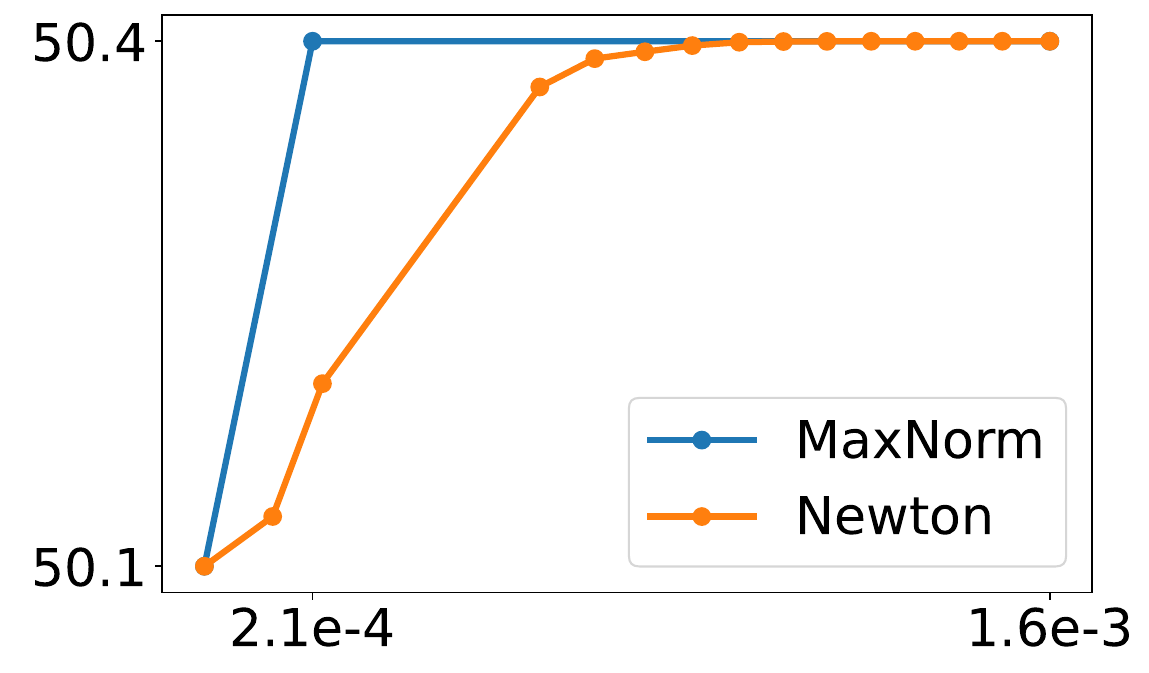} \\
        \multirow{-3}{*}{\centering 1600}  &  \includegraphics[width = 0.15\textwidth]{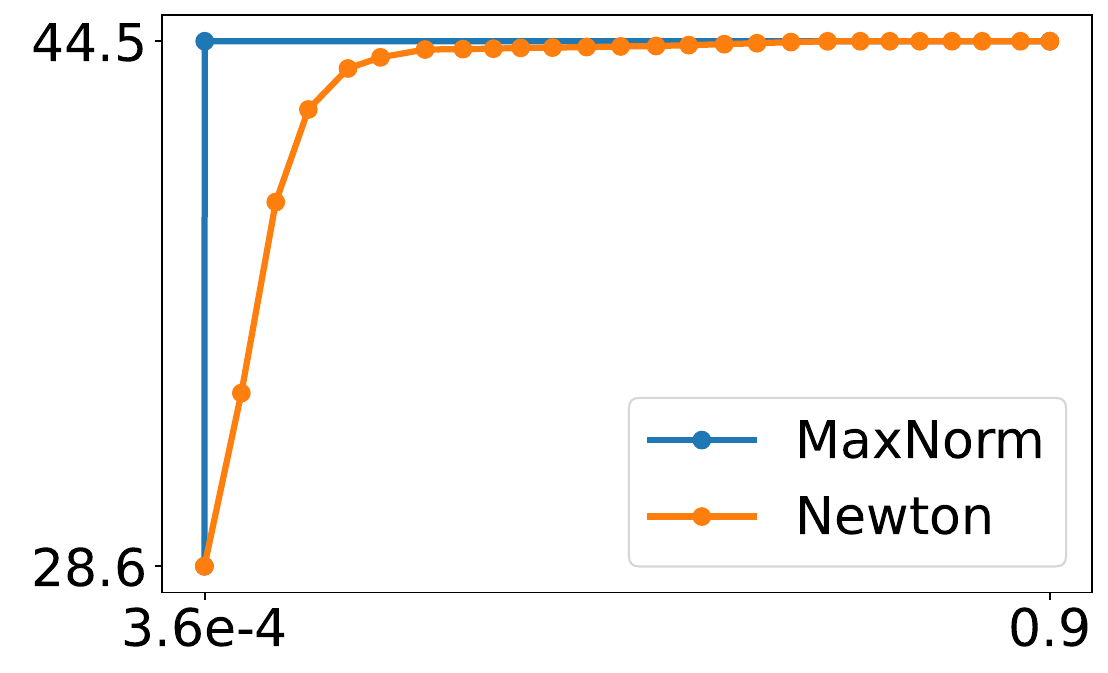} & \includegraphics[width = 0.15\textwidth]{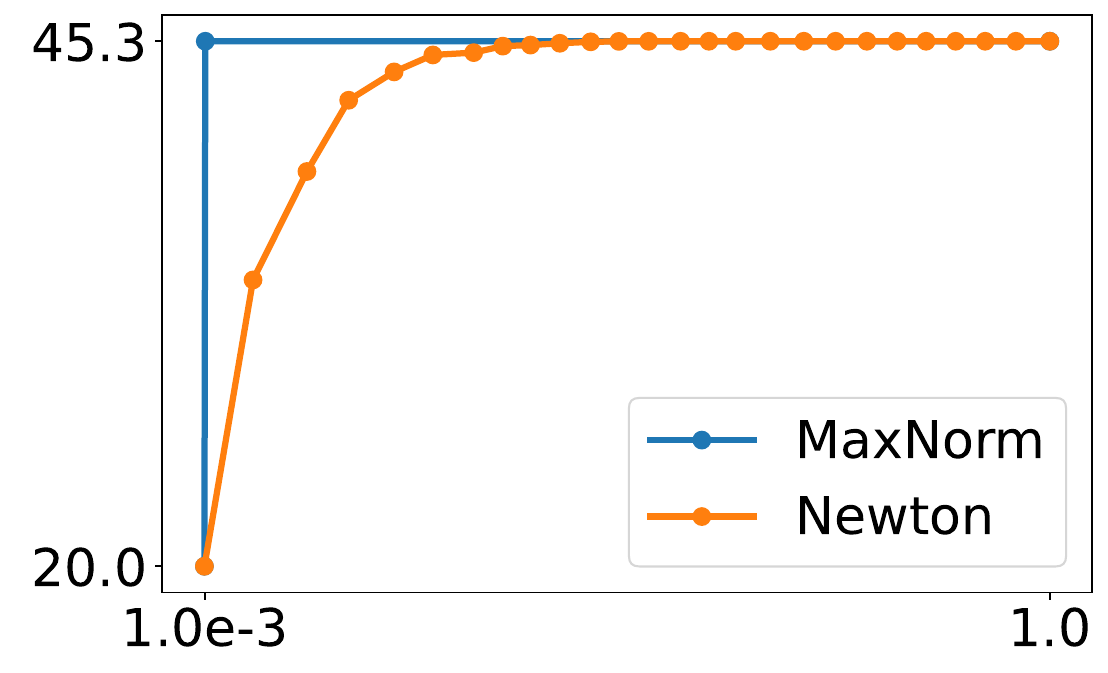} & \includegraphics[width = 0.15\textwidth]{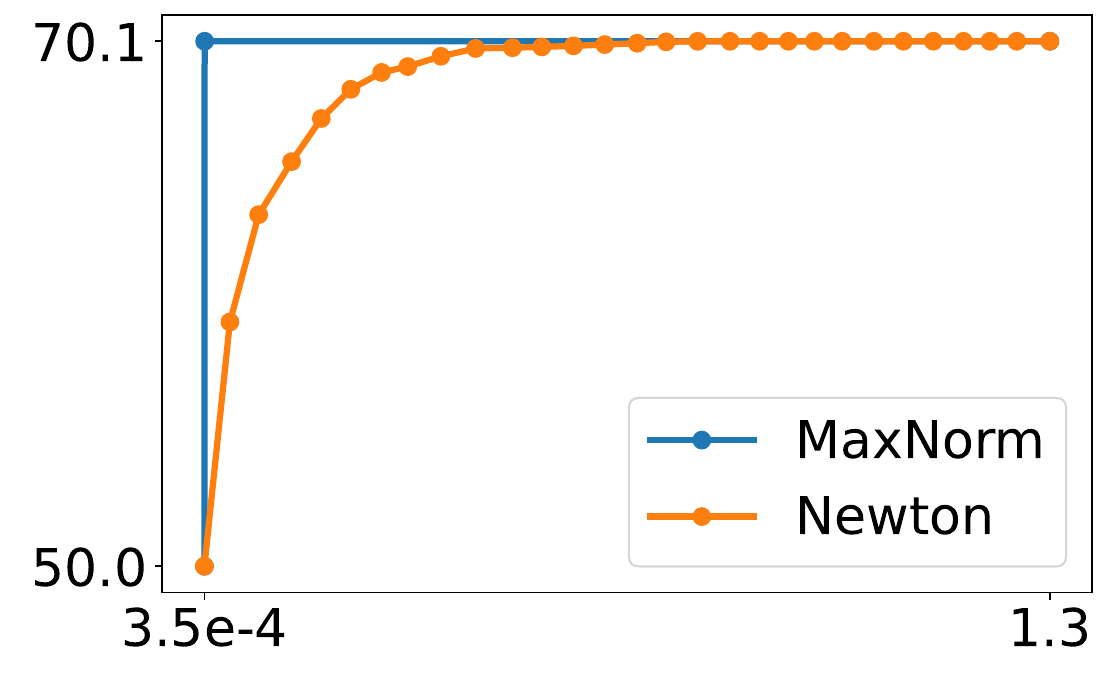} \\
    \end{tabular}
    \caption{
        Value vs. time for instances generated from runs of OLSUCB, as a
        function of $d$ and $\|\zeta\|_{2}$.}
    \label{fig:objective_value_time}
\end{figure}
\begin{remark}
We assume that \ouralgo and \oursecondalgo receive as an input the matrix $A^{1/2} W
A^{1/2}$ already in diagonalized form $U^\top \Lambda U$.
The timing of experiments does not include the diagonalization, which allows to flesh out
the differences between \ouralgo and \oursecondalgo more precisely. Diagonalization is
fast: even in very high dimensions ($d=1600$) it only takes about $1.5$ seconds.
\end{remark}
\textbf{Simulated Eigenvalues and Centers \quad}
To study how the input parameters of $P_B$ influence algorithm performance, we generate
synthetic data by sampling $(b, \lambda)$ from well-chosen distributions. We focus on
the impact of the condition number $\smash{\gls{kappa} := \lambda_{1}/\lambda_{d}}$.
Three cases are considered.
1) Stacked distribution: $\lambda_i = a \kappa$ if $i=1$ and $\lambda_i= a$ otherwise.
2) Random Stacked distribution: $\lambda_1 = a \kappa$ and $\smash{(\lambda_i)_{i \in
[2,d]} = a {\bf sort}(U)}$ (sorted in decreasing order), with $U$
uniformly distributed on $[0,1]^{d-1}$.
3) Ordered Exponential distribution: $\smash{\lambda = \kappa {\bf sort}(E)/2}$ with $E$
a vector of i.i.d. exponentially distributed variables with mean~$1$.
In all cases we draw $b$ from a random stacked distribution: $b_1 = 1 $ and
$\smash{(b_i)_{i \in [2,d]} = 0.1 \,U}$.
For each plotted value, we average over $100$ independent problem instances of $P_B$. 
The median and $\smash{90\%}$ quantiles of the computation time are presented. 
The code provided as supplementary material allows for even more possible distributions, allowing the reader to generate results beyond those presented here. 

In Figure~\ref{fig:performance_as_kappa} we present the performance of algorithms in dimension $d = 500$ as a function of the conditioning number $\kappa$. 
Surprisingly the computation time does not increase with the conditioning number and even decreases for the Exponential distribution (case 3).
The Random Stacked distribution (case 2) is $10$ times more difficult to solve than the Stacked distribution, indicating that small perturbations of the smallest eigenvalues can have a big impact on the speed of \oursecondalgo.
However, this does not seem to affect \ouralgo.
\begin{figure}
	\centering
    \begin{tabular}{@{}c@{}c@{}c@{}}
        $S(0.1, \kappa, 500)$ &  $RS(1, \kappa,500)$ & $Exp_\sigma(\frac{\kappa}{2}, 500)$ \\
	\includegraphics[width=\figuresize]{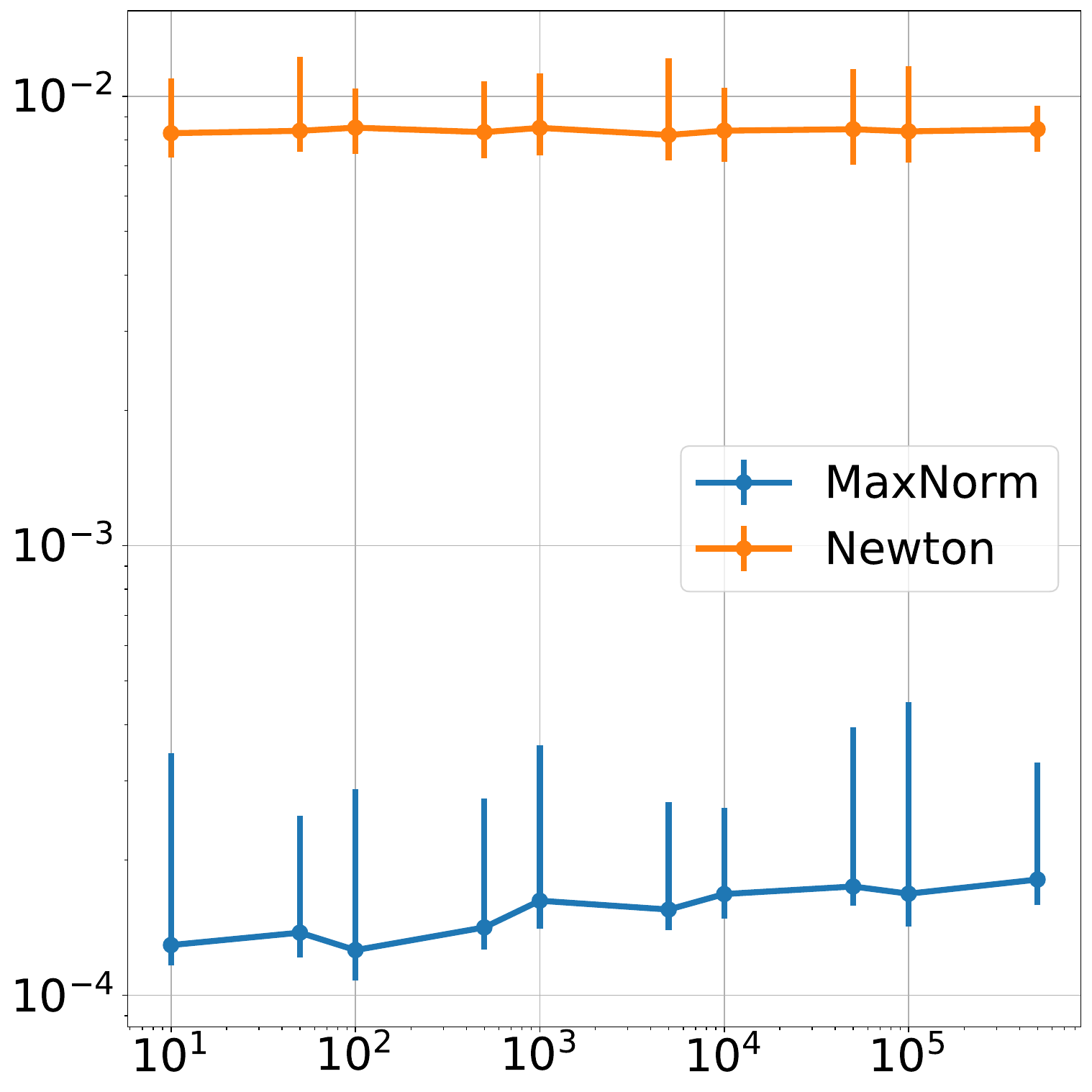} & \includegraphics[width=\figuresize]{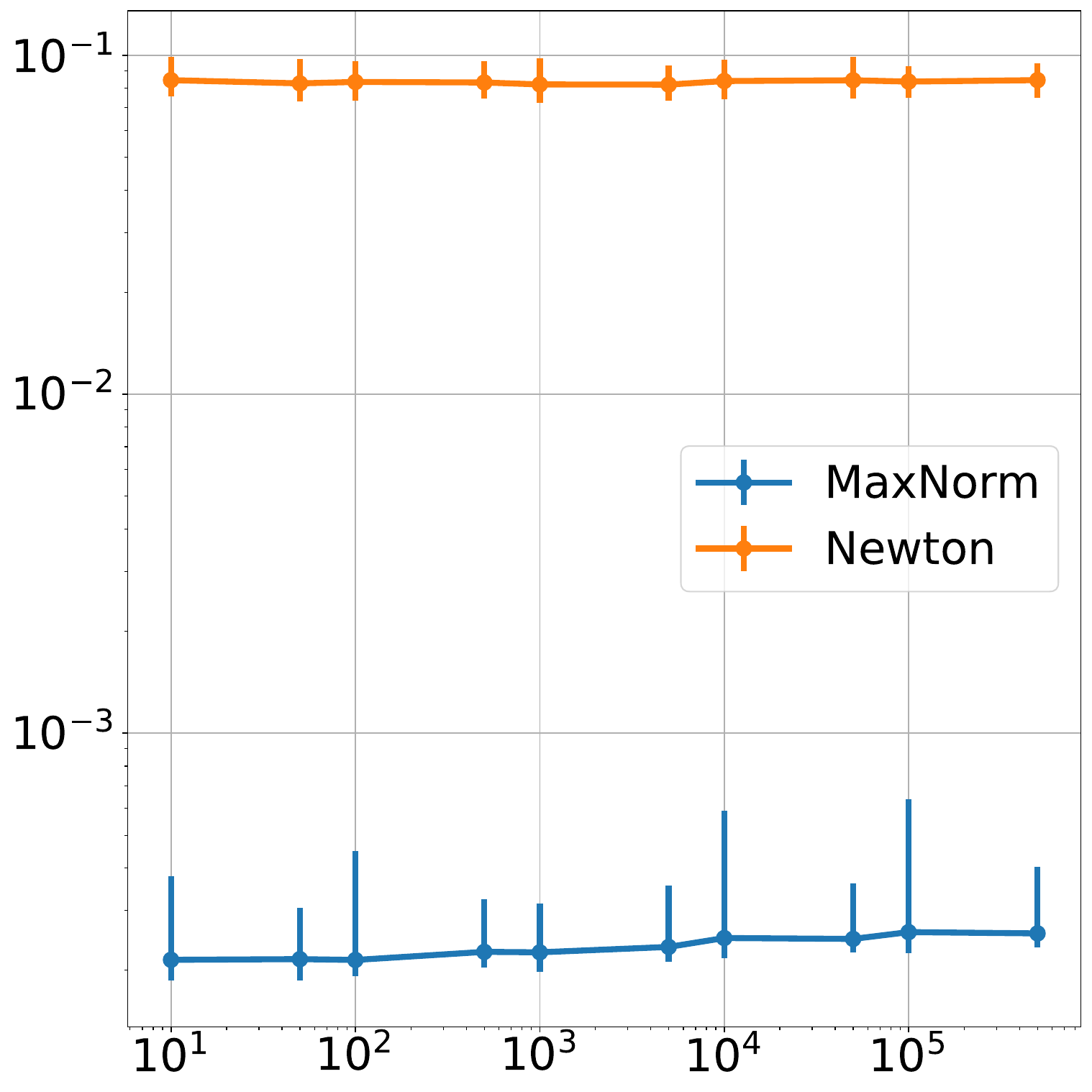} &\includegraphics[width=\figuresize]{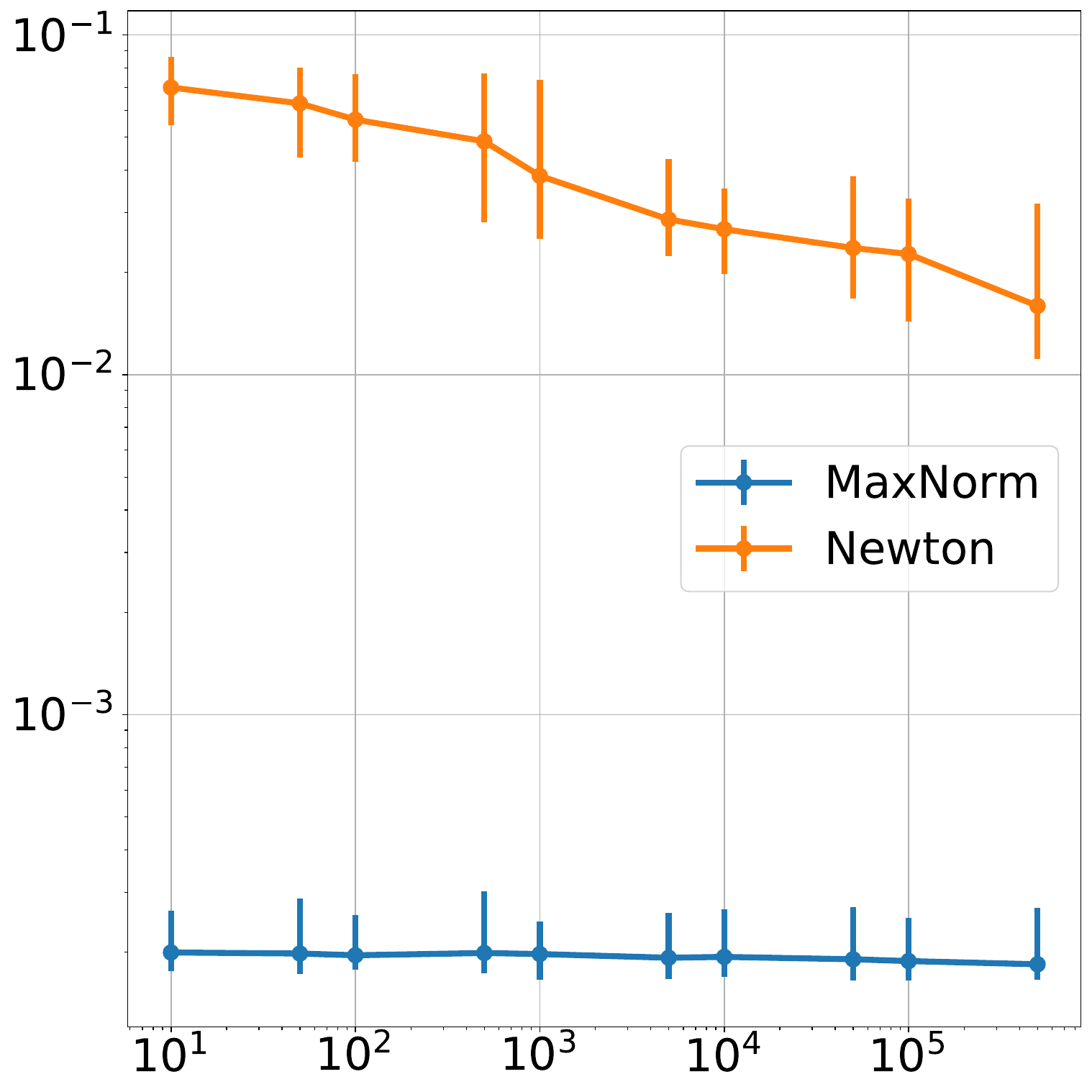} \\
    \end{tabular}
    \caption{Running time of \ouralgo and \oursecondalgo as a function of $\kappa$ for different distributions.}
    \label{fig:performance_as_kappa}
\end{figure}   

In Figure~\ref{fig:performance_as_dim}, we present the running time of \ouralgo and \oursecondalgo as a function of dimension $d$ for a fixed, large condition number $\kappa = 10^5$.
Running times seems to grow quadratically in $d$ for \oursecondalgo and linearly in $d$ for \ouralgo which performs best.

\begin{figure}
	\centering
    \begin{tabular}{@{}c@{}c@{}c@{}}
        Stacked & Random Stacked & Random Exp \\
	\includegraphics[width=\figuresize]{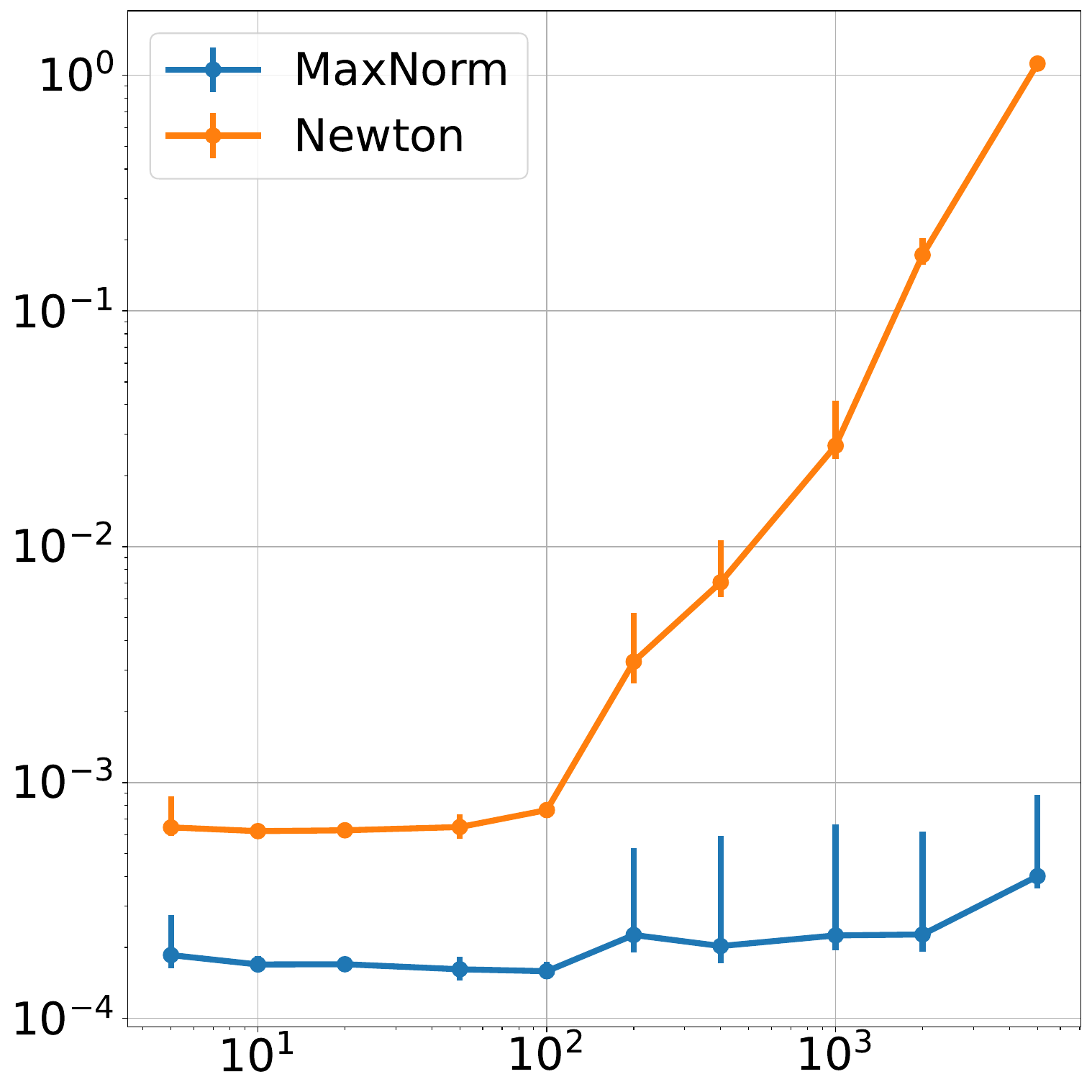} & \includegraphics[width=\figuresize]{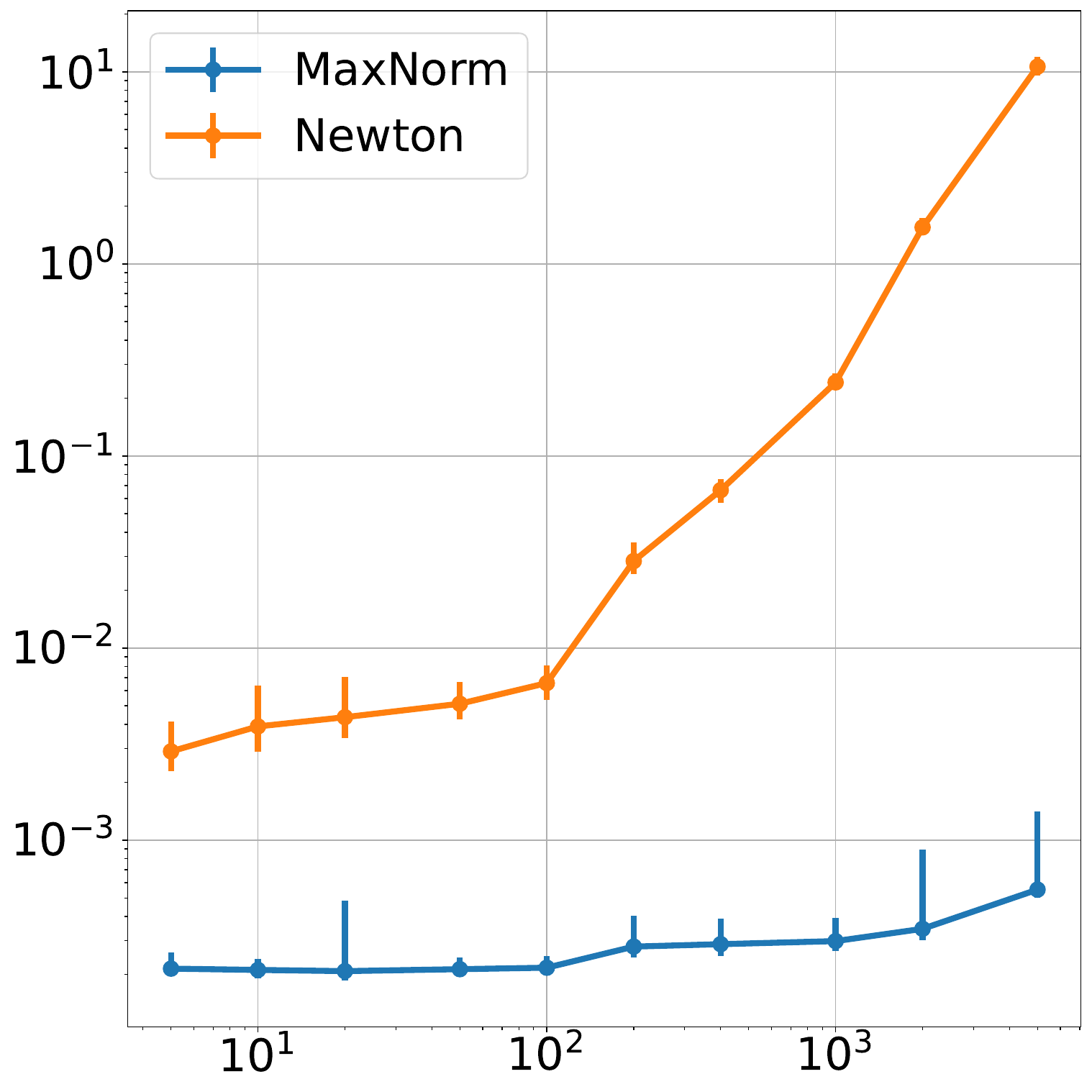} &\includegraphics[width=\figuresize]{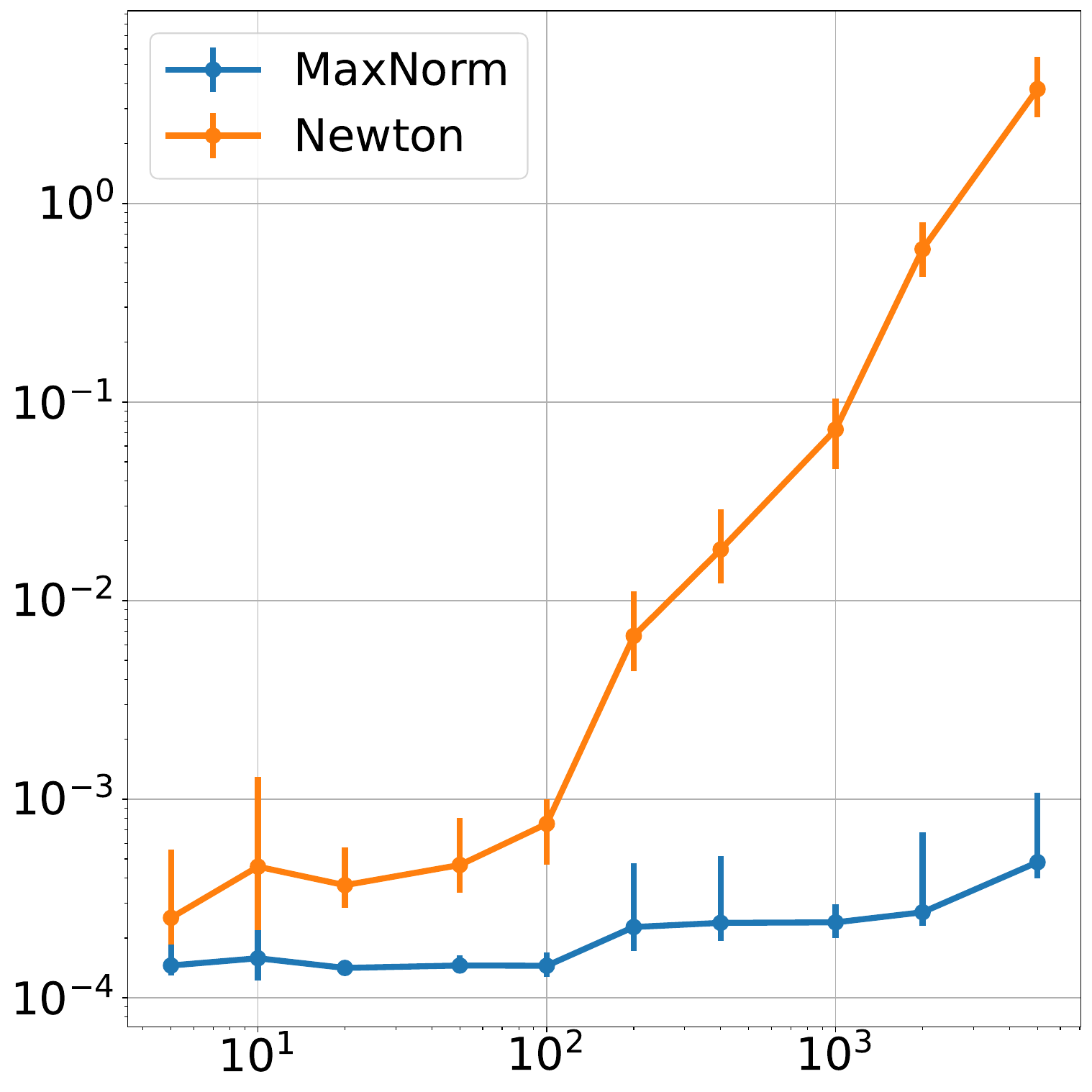} \\
    \end{tabular}
    \caption{Running time of \ouralgo and \oursecondalgo as a function of $d$ for different distributions.}
    \label{fig:performance_as_dim}
\end{figure}

\textbf{Performance of OLSOFUL as a function of the dimension \quad}
We now compare the regret of OFUL~\cite{abbasi2011improved} and
OLSOFUL~\cite{gales2022norm-agn} using \ouralgo and \oursecondalgo as subroutines in
order to solve $P_B$ with E2TC~\cite{zhang_linear_2025} and the non inflated version of
TS~\cite{abeille_when_2025} for linear bandits with smooth actions set.
We consider the unknown vector~$\zeta$ with norm $\|\zeta\|_{2} = 10$ and average the
performance of each algorithm over $10$ independent runs of length $T = 10^4$. In Figure~\ref{fig:olsucb_vs_ts_vs_e2tc} we plot the average final regret $R_T(\zeta)$
as a function of the dimension $d$ as well as $95\%$ confidence intervals.
Using \ouralgo over \oursecondalgo seems to slightly improve the regret of the algorithm while being much faster this may due to numerical imprecision as \oursecondalgo requires more intensive computation. See Appendix~\ref{app:numerical_experiments} for the time taken for each algorithm to perform a run. 
\begin{figure}
    \centering
    \includegraphics[width=0.45\textwidth]{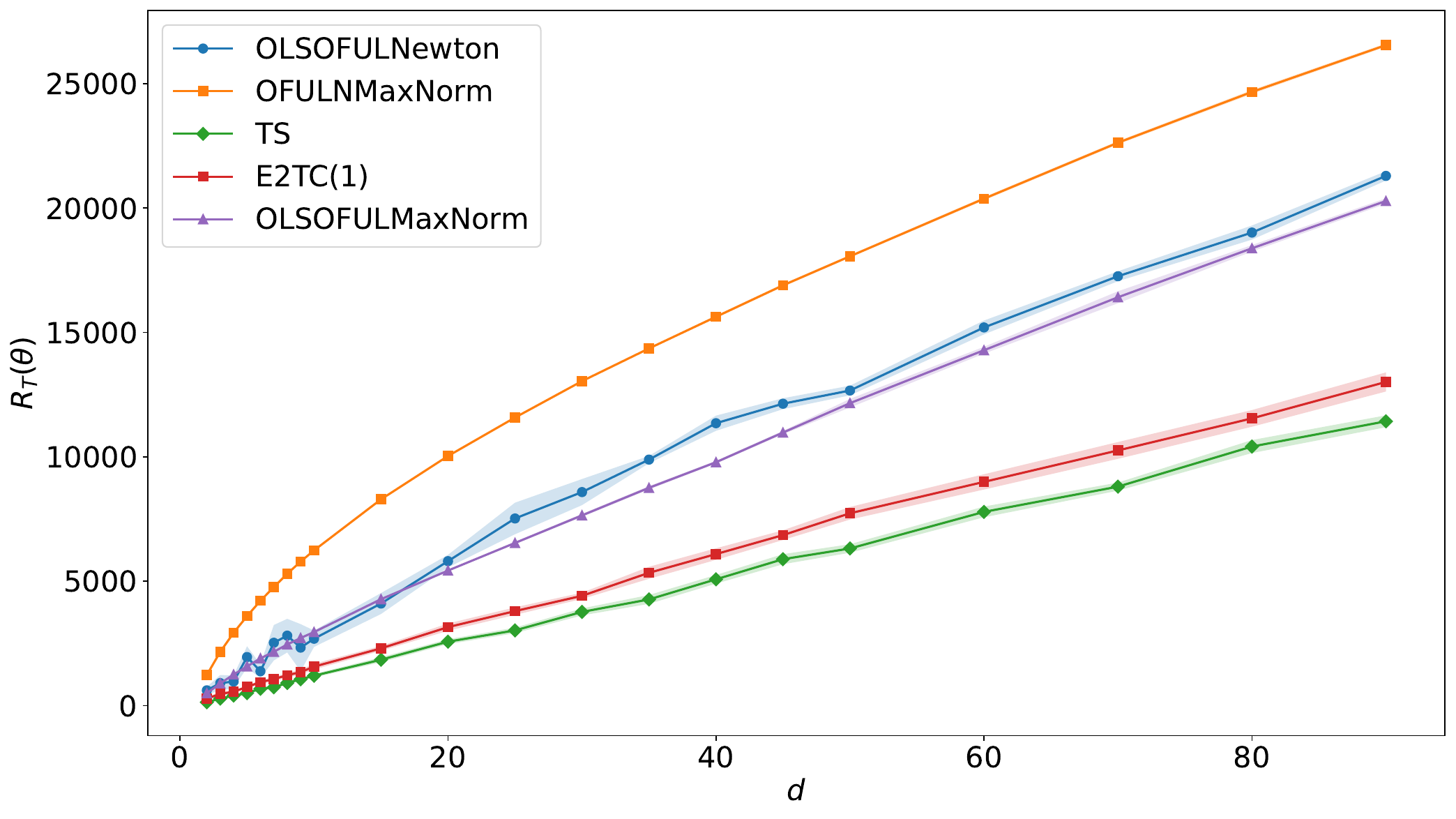}
    \caption{Regret of linear bandit algorithms vs. dimension $d$.}
    \label{fig:olsucb_vs_ts_vs_e2tc}
\end{figure}

\section{Future Work}\label{sec:Conclusion}

Our work opens three questions: i) whether there exists a systematic
way to solve bilinear maximization for more families of sets $\mathcal{X}$, ii) 
whether bilinear maximization is tractable if $\Theta$ is not an ellipsoid, which would
enable solving linear bandits with more general (e.g., heavy-tailed) noise. iii) How to use our method in the framework of robust optimization
with ellipsoidal uncertainty \cite{ben-tal_robust_1998, henrion2001lmi}.

\newpage

\printbibliography
\section*{Checklist}

\begin{enumerate}

  \item For all models and algorithms presented, check if you include:
  \begin{enumerate}
    \item A clear description of the mathematical setting, assumptions, algorithm, and/or model. [Yes]
    \item An analysis of the properties and complexity (time, space, sample size) of any algorithm. [Yes]
    \item (Optional) Anonymized source code, with specification of all dependencies, including external libraries. [Yes]
  \end{enumerate}

  \item For any theoretical claim, check if you include:
  \begin{enumerate}
    \item Statements of the full set of assumptions of all theoretical results. [Yes]
    \item Complete proofs of all theoretical results. [Yes]
    \item Clear explanations of any assumptions. [Yes]     
  \end{enumerate}

  \item For all figures and tables that present empirical results, check if you include:
  \begin{enumerate}
    \item The code, data, and instructions needed to reproduce the main experimental results (either in the supplemental material or as a URL). [Yes]
    \item All the training details (e.g., data splits, hyperparameters, how they were chosen). [Not Applicable]
    \item A clear definition of the specific measure or statistics and error bars (e.g., with respect to the random seed after running experiments multiple times). [Yes]
    \item A description of the computing infrastructure used. (e.g., type of GPUs, internal cluster, or cloud provider). [Yes]
  \end{enumerate}

  \item If you are using existing assets (e.g., code, data, models) or curating/releasing new assets, check if you include:
  \begin{enumerate}
    \item Citations of the creator If your work uses existing assets. [Yes]
    \item The license information of the assets, if applicable. [Yes]
    \item New assets either in the supplemental material or as a URL, if applicable. [Yes]
    \item Information about consent from data providers/curators. [Not Applicable]
    \item Discussion of sensible content if applicable, e.g., personally identifiable information or offensive content. [Not Applicable]
  \end{enumerate}

  \item If you used crowdsourcing or conducted research with human subjects, check if you include:
  \begin{enumerate}
    \item The full text of instructions given to participants and screenshots. [Not Applicable]
    \item Descriptions of potential participant risks, with links to Institutional Review Board (IRB) approvals if applicable. [Not Applicable]
    \item The estimated hourly wage paid to participants and the total amount spent on participant compensation. [Not Applicable]
  \end{enumerate}

\end{enumerate}


\clearpage
\appendix
\thispagestyle{empty}
\onecolumn

\section{Proofs for Section~\ref{sec:Model}}

\noindent \emph{Proof of Proposition~\ref{proposition:LinUCB_approx}.}

For both cases, the reward and observation at time $t$ is given by
\[ y_t = x_t + z_t \, ,\]
with $(z_t)_t$ is i.i.d subgaussian with variance proxy $\sigma^2= 1$. 
We consider the estimator and the confidence ellipsoid defined by OFUL~\cite{abbasi2011improved}, so that 
\[
\begin{aligned}
	V_t := 1 + \sum_{s=1}^{t-1} x_s^2\, , \quad f(t) := 1 + \sqrt{\log(T + tT)}\, , \quad \hat{\zeta}_t : = \frac{1}{V_t}\sum_{s=1}^{t-1} y_sx_s \, .
\end{aligned}
\]
This construction ensures that with high probability and uniformly in time, the true parameter belongs to the confidence ellipsoid
\begin{equation*}
	\PP( \forall t \in [T], \zeta \in C_t ) \geq 1-\frac{1}{T},
\end{equation*}
where
\begin{equation*}
	C_t  := \left\{ \theta \in \RR, | \theta-\hat{\zeta}_t |  \leq \frac{f(t)}{\sqrt{V_t}} \right\} \, .
\end{equation*}
\textbf{The case $\cX = \{1-\epsilon, 1\}$.}

Let us take $\cX = \{x_{\epsilon}, x_1\}$ with $x_\epsilon = 1-\epsilon$ and $x_1 = 1$ and the unknown parameter $\zeta = 1$. 
Assume that the high probability event $G_T := \{ \forall t \in [T], \zeta \in C_t \}$ happens. 

\underline{Step 1: Show that after a certain amount of time $t_{\epsilon}$ the confidence ellipsoid is a subset of $\RR+$.}
Under event $G_T$, we have that $\forall t \in [T], |\hat{\zeta}_t - 1| \leq \frac{f(t)}{\sqrt{V_t}}$. This implies that $\forall t \in [T], 1 - 2 \frac{f(t)}{\sqrt{V_t}} < \min_{\theta \in C_t} \theta$. 
However, because we have two actions $V_t \geq 1 + (1-\epsilon)^2 t$ thus:
\[
\begin{aligned}
1 - 2 \frac{f(t)}{\sqrt{V_t}} \geq 1 - 2 \frac{1+ \sqrt{\log(T+tT)}}{1 + (1-\epsilon) \sqrt{t}} \geq 1 - 2 \frac{1+ \sqrt{\log(T+T^2)}}{(1-\epsilon) \sqrt{t}}
\end{aligned}
\]

So when $t > t_{\epsilon} := 4 \frac{1+ \sqrt{\log(T+T^2)}}{(1-\epsilon)^2}$ we have that $\min_{\theta \in C_t} > 0$,  therefore $\forall t \geq t_{\epsilon},  C_t \subset \RR_+$

 \underline{Step 2: Show that when $C_t \subset \RR_+$ the approximate $\epsilon-$\algoname{LinUCB} algorithm can always choose the suboptimal action $1-\epsilon$.}

 When $C_t \subset \RR_+$ we have that:
\[ (1-\epsilon) \times \max_{\theta \in C_t} \theta = (1-\epsilon) \times 1 \times \max_{\theta \in C_t} \theta  \geq (1-\epsilon) \max_{x \in \cX, \theta \in C_t} x \theta. 
\]
 
Therefore $x_\epsilon, \max_{\theta \in C_t}$ is a solution with approximation ratio $\epsilon$ and the algorithm that plays $x_\epsilon$ after $t_\epsilon$ is indeed an $\epsilon-$\algoname{LinUCB} algorithm.

\underline{Step 3: Lower bound on the regret.}

Playing $x_\epsilon$ instead of $x_1$ induces a regret of $\epsilon$ at each time step. Therefore, the regret of the $\epsilon-$\algoname{LinUCB} algorithm that always plays $x_\epsilon$ after $t_\epsilon$ is lower bounded by:
\[
 R_T(\epsilon) \geq \PP( \forall t \in [T], \zeta \in C_t )(T - t_\epsilon)\epsilon \geq (1-\frac{1}{T})(T - 4\frac{1+ \sqrt{\log(T+T^2)}}{(1-\epsilon)^2})\epsilon.
\]
And finally :
\[
\lim_{T \rightarrow +\infty} \frac{R_T(\zeta)}{T} \geq \epsilon
\]

\newpage
\textbf{The case $\cX = [a,b]$ with $a < 0 < b$.}

The confidence set at time $t$ is 
\begin{align*}
	C_t  &= \left\{ \theta \in \RR, | \theta-\hat{\zeta}_t |  \leq \frac{f(t)}{\sqrt{V_t}} \right\} \, .\\
	     &= \left[ \hat\zeta_t - { f(t) \over \sqrt{V_t}}, \hat\zeta_t + { f(t) \over \sqrt{V_t}}  \right]
\end{align*}
The optimization problem is
\begin{align*}
	\max_{(x,\theta) \in \mathcal{X} \times C_t} x^\top \theta  
	&= \max\left( \max_{x \ge 0,x \in \mathcal{X}} x^\top \left(\hat\zeta_t + { f(t) \over \sqrt{V_t}} \right) ,  \max_{x \le 0,x \in \mathcal{X}} x^\top \left(\hat\zeta_t - { f(t) \over \sqrt{V_t}} \right) \right) \\
	&= \max\left( a \left(\hat\zeta_t + { f(t) \over \sqrt{V_t}} \right) , b \left(\hat\zeta_t - { f(t) \over \sqrt{V_t}} \right) \right)
\end{align*}
The exact OFUL algorithm selects the optimal solution 
\begin{align*}
	x_t^{e} &= \begin{cases}
		a  \text{ if } a \left(\hat\zeta_t + \frac{f(t)}{\sqrt{V_t}}   \right) \ge b \left(\hat\zeta_t - \frac{f(t)}{\sqrt{V_t}}   \right), \\
			b  \text{ if } a \left(\hat\zeta_t + \frac{f(t)}{\sqrt{V_t}}   \right) <  b \left(\hat\zeta_t - \frac{f(t)}{\sqrt{V_t}}  \right),  
		\end{cases}
\end{align*}
Now consider an algorithm that selects 
\begin{align*}
	x_t^{a} &= (1-\epsilon) x_t^e  
\end{align*}
one may readily check that this algorithm is indeed an approximate OFUL algorithm with approximation ratio $\epsilon$ in the sense that 
\begin{align*}
	\max_{\theta \in C_t} (x_t^{a})^\top \theta \ge (1-\epsilon) \max_{(x,\theta) \in \mathcal{X} \times C_t} x^\top \theta
\end{align*}

Since $x_t^a \in [(1-\epsilon) a, (1-\epsilon) b]$ we have that 
\begin{align*}
	\max_{x \in \mathcal{X}} x^\top \zeta - (x_t^a)^\top \zeta \ge 
	\max_{x \in [a,b]} x^\top \zeta - \max_{x \in [(1-\epsilon) a, (1-\epsilon) b]} x^\top \zeta
	=
	\begin{cases}
		\epsilon b \zeta &\text{ if } \zeta \ge 0	\\
		\epsilon a \zeta &\text{ if } \zeta \le 0	
	\end{cases}
\end{align*}
In both cases one may check that 
\begin{align*}
	\max_{x \in \mathcal{X}} x^\top \zeta - (x_t^a)^\top \zeta \ge \epsilon \max_{x \in \mathcal{X}} x^\top \zeta 
\end{align*}
Hence the regret of this algorithm is lower bounded by 
\begin{align*}
	R_T(\zeta) \ge T \epsilon \max_{x \in \mathcal{X}} x^\top \zeta 
\end{align*}
\qed

\label{app:proofs_model}

\noindent {\emph{Proof of Proposition~\ref{proposition:maximization_over_theta}.}

Consider $x$ fixed so that we are maximizing a linear function $x^\top \theta$ over a
convex set $\Theta$. From the KKT conditions, the maximizer of $\theta \mapsto x^\top
\theta$ subject to $\|\theta-c\|_{W} \le 1$ is $\theta = c + W^{-1} x/\|x\|_{W^{-1}}$,
with value $x^\top \theta = x^\top c + \|x\|_{W^{-1}}$, and replacing in the definiton
of $P_B$ proves the result. 
\qed

\noindent \emph{Proof of Proposition~\ref{proposition:polynomially_many_vertices}.}

The objective function $x \mapsto x^\top c + \|x\|_{W^{-1}}$ is convex by
convexity of norms, so its maximum over $x \in \mathcal{X}$ which is a polytope must be
attained at one of the vertices $\mathcal{V}$. The result then follows from applying
Proposition~\ref{proposition:maximization_over_theta}, so that solving $P_X$ is
equivalent to solving $P_B$. 
\qed

\noindent \emph{Proof of Proposition~\ref{proposition:centered_ellipsoid}.} 

Consider the change of variables $u = A^{1/2} x$, $v = W^{1/2} \theta$, problem $P_B$ is
\[
\max_{\|x\|_{A} \le 1 } \max_{\|\theta\|_{W} \leq 1 } 
x^\top \theta 
= 
\max_{\|u\|_{2} \le 1 } \max_{\|v\|_{2} \leq 1 } 
u^\top (W A)^{-1/2} v 
\]
which is maximized at $u = v = \psi$, as claimed.
\qed

\noindent \emph{Proof of Proposition~\ref{proposition:hardness}.}

 If $\cX = B_{p,d}$ and $\Theta$ is centered at $0$, consider the change of variables $v = W^{1/2} \theta$, then $P_B$ is 
\[
  \max_{\| x \|_p \leq 1} \max_{\theta^\top W \theta \leq 1} x^\top \theta 
  =
  \max_{\| x \|_p \leq 1} \max_{ \|v\|_{2} \leq 1} x^\top W^{-1/2} v 
  = 
  \max_{\| x \|_p \leq 1} \sqrt{x^\top W^{-1} x}
\]
which is equivalent to computing the operator norm $\| W^{-1/2} \|_{p \to 2}$ of the matrix
$W^{-1/2}$.
From Theorem 1.3 of~\cite{bhattiprolu2023inapprox} there exists $\epsilon_p$ such that
computing the $p \to 2$ operator norm of a symmetric matrix with approximation ratio
$\epsilon_p$ is $\mathcal{NP}$-hard, for any $p > 2$.
Note that when $p = +\infty$, the maximum cut problem, which is $\mathcal{NP}$-Hard and
conjectured to be not approximable efficiently (see~\cite{khot2007optimal}) reduces to
$P_B$. \qed

\section{Proofs of Section~\ref{sec:convex_newton}}
\label{app:proof-section-convex-newton}

\subsection{Full Newton algorithm for Bilinear maximization on ellipsoids}
\label{app:reduct-alg}

The main idea of the proof is to carefully transform the optimization problem into another problem which involves optimizing a self-concordant convex function.
We then follow the analysis of Newton's method for self-concordant convex functions of~\cite{nesterov_lectures_2018}, where some arguments must be adapted to fit our setting.
\begin{algorithm} 
  \SetKwInOut{Input}{Input}
  \SetKwInOut{Parameters}{Parameters}
  \Input{Matrices $A, W \in \mathbb{R}^{d \times d}$, 
  vector $c \in \mathbb{R}^d$,
  precision $\epsilon>0$}
  \Parameters{
  $\gls{eta} > 1$ barrier increase parameter,
  }
  \% Preliminary computations \\
  Compute the eigenvalue decomposition of $A^{1/2} W A^{1/2} = U^\top \Lambda U$ 
  with $\Lambda = {\bf diag}(\lambda)$ and $\lambda_d^{-1} \geq ... \geq \lambda_d$ and $b = U^\top A^{1/2} c$. \\

  Compute 
  \begin{equation}
  \label{eq:domdefinition}
    B_i = \bigg( 
    \frac{b_i}{\|b\| + \lambda^{-1/2}_d}
    \bigg)^{2} 
    \quad  \text{and} \quad
    \gls{DF} = \Big\{
      y \in \R^{d} \mid
        y_i >  B_i \text{ for all } i\in [d] 
        \quad \text{and} \quad
        \sum_{i=1}^{d} y_i \leq 1
    \Big\} 
    \,.
  \end{equation} \\
  \tcp{Compute the initial point}
  \begin{equation}
  \label{eq:initial_point}
  \gls{y0}_i 
        = B_i + \frac{\lambda^{-1}_d / 2}{\big(\|b\| + \lambda^{-1/2}_d\big)^2} 
        \quad \text{for all } i \in [d]
  \end{equation} \\
  \tcp{Collect the non-null index of $b$}
  \[ \mathcal{I} = \{i \in [d]: b_i \neq 0\} \]\\
  \tcp{Compute the first penalization}
  \[  t_{0} =9 \max\bigg(\max_{i \in \mathcal{I}}(b_i^2 B_i)^{-1/2}
  , \, \min_{i \in \mathcal{I}} (\lambda^{-1}_i B_i)^{-1/2}
  \bigg) 
  \] \\
  \tcp{Set the barrier penalization}
  $t \leftarrow t_{0}$ \\
  \tcp{Define the barrier penalized objective}
  \begin{equation}
    \label{eq:barrier_penalized_objective}
    F^{(t)}(y) =
   t F(y) 
   - \sum_{i = 1}^{d} \log( y_i - B_i)
   -  \log\bigg( 1 - \sum_{i=1}^{d} y_i\bigg)
   \,.
  \end{equation}\\
  With 
  \begin{equation}
    F(y) := - \sum_{i=1}^d | b_i | \sqrt{y_i} - ( \sum_{i=1}^d \lambda_i^{-1} y_i )^{1/2}
  \end{equation}
\tcp{Centering  by applying the Newton method Alg.\ \ref{alg:Newton Method}}
  $y \leftarrow \mathrm{Newton}(F^{(t)}, y^0, \epsilon)$ \\ 
  \While{$(d+1) / t > \epsilon / 2$}{
    \tcp{Increase the barrier parameter}
    Set $t \leftarrow \eta t$\\
    \tcp{Define the new barrier penalized objective}
    \[
    \gls{Ft}(y) =
     t F(y) 
     - \sum_{i = 1}^{d} \log( y_i - B_i)
     -  \log\bigg( 1 - \sum_{i=1}^{d} y_i\bigg)
     \,.
    \]\\
  \tcp{Apply the Newton method Alg.\ \ref{alg:Newton Method} on the new regularized function starting from the output of the previous phase}
  $y \leftarrow \mathrm{Newton}(F^{(t)}, y, \epsilon/2)$
  }
 
  Get $\tilde{x}$ with coordinates $\tilde{x}_i = \sqrt y_i \sign(b_i)$ \\
  \tcp{Change to the original base}
  \[
  \hat x = A^{-1/2} U \tilde{x} 
  \quad \text{and} \quad
  \hat{\theta} = c + W^{-1}\hat x / \| \hat x \|_{W^{-1}}
  \]
  {\bf Output:} $(\hat{x}, \hat{\theta})$ approximate solution to $P_{B}$
  \caption{\gls{oursecondalgo} : Reduction to Convex Optimization}\label{alg:reduct-opt} 
\end{algorithm}

\subsection{Newton's Algorithm for Self-Concordant Convex Functions}
\label{app:newton}
As a preliminary, we recall some results on self-concordant convex functions and Newton's method from~\cite{nesterov_lectures_2018} and \cite{boyd_convex_2023}.
The main result we will use is Theorem~\ref{thm:total_steps_newton_optim} which provides an upper bound on the number of iterations required for Algorithm~\ref{alg:Newton Method} (a particular version of Newton's method adapted to self-concordant convex functions) in order to output an $\epsilon$-minimizer of a self-concordant convex function. 
Before stating the theorem we highlight each of the propositions which constitute its proof. 
\begin{definition}
  Function $F$ satisfies the self-concordance differential inequality (\gls{scdi}) on $D \subset \mathbb{R}^d$ if for any $z \in D$, for any $h \in \R^d$, 
   \[
      |\gls{dif3}(z)[h, h, h]| \leq 2 (\gls{dif2}(z)[h, h])^{3/2} \,.
   \]  
   Function $F$ defined on $D$ is self-concordant if it satisfies the SCDI on $D$ and $F(z) \to +\infty$ when $z \to \partial D$, where $\partial D$ denotes the boundary of $D$.
 \end{definition}

\begin{definition}
  Consider $F$ self-concordant on $D_F$ with strictly positive Hessian so that $\nabla^2 F(y) \succ 0$ for all $y \in D_F$. 
  For any $y \in D_F$, we define the corresponding local norm of the gradient $\gls{dif}$
  \begin{equation}
  \label{eq:def_localnorm}
  \gls{localnorm}(y) = [(\nabla F)(y)]^\top  [(\nabla^2 F)(y)]^{-1} [(\nabla F)(y)].
  \end{equation}
\end{definition}

\begin{algorithm} 
  \SetKwInOut{Input}{Input}
  \Input{$F$ a self-concordant function convex function, $y_0$ initial point, $\epsilon$ Precision}
  \tcp{Set the initial point}
  $y= y^0$ \\
  Compute $\lambda_F(y)$ using \ref{eq:def_localnorm} \\
  \tcp{First Stage}
  \While{ $\lambda_F(y) \geq 1/4$}{
    \tcp{Damped Newton Step}
    \begin{equation}
      \label{eq:def_damped-newton-step}
      y = y -   \frac{1}{1 + \lambda_F(y)} [(\nabla^2 F)(y)]^{-1}  (\nabla F)(y)
    \end{equation}
    Compute $\lambda_F(y)$ using \ref{eq:def_localnorm}
    }
  \tcp{Quadratic Convergent stage}
  \While{ $\lambda^2_F(y) > \eps $}{
    \tcp{Intermediate Newton step}\
    $\xi  = \frac{\lambda^2_F(y)}{1+ \lambda_F(y)}$ 
  \begin{equation}
      \label{eq:def_intermediate-newton-step}
      y = y -   \frac{1}{1 + \xi} [(\nabla^2 F)(y)]^{-1}  (\nabla F)(y) 
  \end{equation}
  Compute $\lambda_F(y)$ using \ref{eq:def_localnorm}
  }
  {\bf Output:} $\hat{y}$ an $\eps$-minimizer of $F$
  \caption{Newton Method for self-concordant function}\label{alg:Newton Method} 
\end{algorithm}

Theorem~\ref{thm:decrease_objective} shows that during the first stage of Algorithm~\ref{alg:Newton Method} (when $\lambda_F(y) \ge 1/4$) the function value decreases at least by a constant amount each time a damped Newton step is performed.
\begin{theorem}[Theorem 5.1.15 in \cite{nesterov_lectures_2018}]
  \label{thm:decrease_objective}
  Let $F$ be a self-concordant function on $D_F$.
  Let $y \in D_F$ with $\lambda_F(y) \ge 1/4$.
  Let $y^{+}$ the outcome of a Damped Newton step starting at $y$:
 \begin{align*}
  \gls{yplus} = y -   \frac{1}{1 + \lambda_F(y)} [(\nabla^2 F)(y)]^{-1}  (\nabla F)(y)
 \end{align*} 
  Then
  \[
    F(y^{+}) \leq F(y) - \omega(\lambda_F(y)) \le F(y) - \omega(1/4) \le F(y) - 1/38
  \]
  where $\gls{omega}(l) := l-\ln(1+l)$, since $l \mapsto \omega(l)$ is non-decreasing on $\mathbb{R}^+$.
\end{theorem}
Corollary~\ref{cor:duration_first_stage} provides an upper bound on the number of damped Newton steps necessary for the first stage of Algorithm~\ref{alg:Newton Method} to end, as this stage lasts until $\lambda_F(y) < 1/4$ for the first time.
\begin{corollary}
  \label{cor:duration_first_stage}
  Let $F$ self-concordant on $D_F$.
  Let $y^\star$ the minimizer of $F$ over $D_F$. 
  Let $y^{+,N_1}$ the outcome of $N_1=38 (F(y_0) - F(y^\star))$ Damped Newton steps starting at $y_0 \in D_F$.
  Then $\lambda_F(y^{+,N_1}) < 1/4$.  
\end{corollary}
Proposition~\ref{prop:quadratic_stage} shows that during the second stage of Algorithm~\ref{alg:Newton Method}, i.e. after event $\lambda_F(y) < 1/4$ occurs for the first time, the function value decreases rapidly. 
\begin{proposition}
  \label{prop:quadratic_stage}\cite[Page 358]{nesterov_lectures_2018} :
  Let $F$ be a self-concordant function on $D_F$.
  Let $y \in D_F$ such that $\lambda_F(y) < 1/4$. 
Let $y^{+}$ the output of an intermediate Newton step
\begin{align*}
      y^+ = y -   \frac{1}{1 + \xi} [(\nabla^2 F)(y)]^{-1}  (\nabla F)(y) 
      \text{ with }  \xi(y)  = \frac{\lambda^2_F(y)}{1+ \lambda_F(y)}
  \end{align*}
  Then 
  \[
    \lambda_F(y^{+}) \leq \lambda_F^{2}(y) (1+2 \lambda_F(y)) 
    \,.
  \]
\end{proposition}
Corollary~\ref{cor:duration_quadratic_stage} allows to upper bound the number of intermediate Newton steps necessary for the second stage of Algorithm~\ref{alg:Newton Method} to end, i.e. after event $\lambda_F(y) < 1/4$ occurs for the first time. We recall that, by definition, Algorithm~\ref{alg:Newton Method} has two only two stages, so that this allows to upper bound the number of iterations for the algorithm to terminate.
\begin{corollary}
  \label{cor:duration_quadratic_stage}
  Let $F$ self-concordant on $D_F$.
  Let $y^\star$ the minimizer of $F$ over $D_F$. 
  Let $y^{+,N_2}$ the outcome of $N_2= \log_2(\log_2(1/\epsilon))$ intermediate Newton steps starting at $y_1 \in D_F$ with $\lambda_F(y_1) < 1/4$.
  Then $\lambda_F(y^{+,N_2}) \le \epsilon$.  
\end{corollary}
\begin{proof}
	Denote by $y^{+,k}$ for $k=0,...,N_2$ the outcome of $k$ intermediate Newton steps starting at $y_1 \in D_F$ with $\lambda_F(y_1) < 1/4$. 
	Let us show that $\lambda_F(y^{+,k}) \le (1/2) 2^{-2^{k}}$ for all $k$ by induction. 
	The inequality holds for $k = 0$, and if it holds for $k$, proposition~\ref{prop:quadratic_stage} gives
  \begin{align*}
	  \lambda_F(y^{+,k+1}) &\leq \lambda_F^{2}(y^{+,k}) (1+2 \lambda_F(y^{+,k}))\\
    & \leq   \frac{1}{4}  (1/2)^{2^{k-1}\times 2} \big(1+ (1/2)^{2^{k-1}}\big)\\
    & \leq   \frac{1}{2} (1/2)^{2^{k}} \,.
  \end{align*}
  So $\lambda_F(y^{+,k}) \le (1/2) 2^{-2^{k}}$ for all $k$, and $ \lambda_F(y^{+,N_2}) \le (1/2) 2^{-2^{N_2}} \le \epsilon$ as announced.
\end{proof}
 Proposition~\ref{prop:bound_localnorm} allows to control the difference in terms of function value $F(y)-F(y^\star)$ as a function of the local norm of the gradient $\lambda_F(y)$.
\begin{proposition}
\cite[Theorem 5.1.13, page 347, Lemma 5.1.5 page 345]{nesterov_lectures_2018}:
  \label{prop:bound_localnorm}
  Let $F$ self-concordant on $D_F$.
  Let $y \in D_F$ and $y^\star$ the minimizer of $F$ over $D_F$. 
  If $\lambda_F(y) < \frac{1}{4}$ then
  \[
    F(y)-F(y^\star) \leq \omega_\star(\lambda_F(y)) < \lambda_F^2(y) 
  \]
  Where $\gls{omegastar}(a) := -a - \ln(1-a) < a^2 / (2(1-a))$.
\end{proposition}
So, combining Corollaries \ref{cor:duration_first_stage}, \ref{cor:duration_quadratic_stage} and Proposition \ref{prop:bound_localnorm} we get Theorem~\ref{thm:total_steps_newton_optim} which, as announced, gives an upper bound on the number of Newton steps necessary for Algorithm~\ref{alg:Newton Method} to terminate and output an $\epsilon$-optimal minimizer.
\begin{theorem}
  \label{thm:total_steps_newton_optim}
  Let $F$ self-concordant on $D_F$. Let $y_0 \in D_F$. Let $y^\star$ the
  minimizer of $F$ over $D_F$. Let $\epsilon > 0$ be the required precision. 
 Then Algorithm~\ref{alg:Newton Method} applied to $F$ starting at $y_0$ with accuracy $\epsilon$ terminates after at most 
  \[  
    N = N_1 + N_2 = 38(F(y_0) - F(y^\star)) + \log_2(\log_2( 1/\epsilon ))
  \]
  Newton steps and its output $\hat{y}$ is an $\epsilon$-minimizer of $F$ that is  $F(\hat{y}) - F(y^\star) < \epsilon$.
\end{theorem}

\subsection{Notation index for \oursecondalgo algorithm}
The pseudo-code of \oursecondalgo is stated as Algorithm~\ref{alg:reduct-opt}.  Before stating the analysis, we define several useful notations 
\begin{align*}
	F(y) &:= - \sum_{i=1}^d | b_i | \sqrt{y_i} - ( \sum_{i=1}^d \lambda_i^{-1} y_i )^{1/2} \\
	\gls{ystar} &:= \arg\min_{y \in \Delta_{d-1}} F(y) \\
	\gls{B} &:= (b_i^2 ( \|b\|_{2} + \lambda_d^{-1/2} )^{-2})_{i \in [d]} \\
	D_F &:= \{ y \in \Delta_{d-1}: y \ge B \} \\
	y_0 &:= (B_i +(1/2) \lambda^{-1}_d (\|b\|_2 + \lambda^{-1/2}_d)^{-2})_{i \in [d]} \\
	F^{(t)}(y) &:= t F(y) - \sum_{i = 1}^{d} \log( y_i - B_i) -\log( 1 - \sum_{i=1}^{d} y_i) \\
  \gls{ytstar} &:= \arg\min_{y \in D_F} F^{(t)}(y) \\
   	\mathcal{I} &:= \{i \in [d]: b_i \neq 0\} \\
	t_0 &:= 9 \max\Big(\max_{i \in \mathcal{I}} (b_i^2B_i)^{-1/2}  , \, \min_{i \in \mathcal{I}} (\lambda^{-1}_i B_i)^{-1/2} \Big)\\
\end{align*}
The above is meant to be used as a reference for all the notation used in this section. 

\subsection{Sketch of Analysis of the \oursecondalgo algorithm}
Because the analysis of \oursecondalgo is relatively long, we first highlight the main steps, along with pointers to the lemmas which constitute the analysis. Those lemmas are presented in full details in the next subsection.

Step 1: Problem $P_C$ involves minimizing convex function $F$ over $\Delta_{d-1}$, however the behaviour of $F$ is undesirable whenever $y_i$ is close to $0$ for some $i$, so that we first show in lemma~\ref{lem:minimum_in_domain} that the optimal solution $y^\star$ must verify $y^\star \ge B$, this allows to restrict our attention to $D_F$in order to solve $P_C$.

Step 2: The goal of \oursecondalgo is to apply Newton's method for self-concordant convex functions, which is Algorithm~\ref{alg:Newton Method}, however the original form of $P_C$ does not allow this, because its objective function $F$ is not self-concordant. So we prove Lemma~\ref{lem: selfconcordance}, which shows that in order to solve $P_C$ by penalization, we can minimize $F^{(t)}(y)$ when $t$ is large enough and that for all $t \ge t_0$ large enough, $F^{(t)}$ is self-concordant on $D_F$.

Step 3: \oursecondalgo applies Algorithm~\ref{alg:Newton Method} to minimize the convex self-concordant function $F^{(t)}$ over $D_F$ for increasing values of $t$, starting at $t = t_0$. Proposition~\ref{prop:time_centering} provides an upper bound on the necessary number of Newton steps for the inital "centering step", i.e. when $t = t_0$.

Step 4: \oursecondalgo applies Algorithm~\ref{alg:Newton Method} for several values of $t$ which increase geometrically, and Lemma~\ref{lem:geometric_barrier} provides an upper bound on the total number of Newton steps necessary such that, upon termination, one indeed obtains an optimal solution of the original problem $P_C$ with accuracy $\epsilon$. 

Putting everything together, we obtain Theorem~\ref{thm:complexity_reduction_algo}, which provides a complete upper bound on the time and memory complexity of \oursecondalgo as annnounced.

\subsection{Analysis of the \oursecondalgo algorithm}
\label{app:proof_time_complexity_reduc_algo}
Based on the analysis of Algorithm~\ref{alg:Newton Method} we now analyze the \oursecondalgo algorithm, which we propose in this work in order to solve~\problem\!\!, as an alternative to \ouralgo. Our first intermediate result is Lemma~\ref{lem:minimum_in_domain}, stating that $y^\star$ the optimal solution to $P_C$ must verify $y^\star \ge B$. This allows to narrow down the space in which we search for the solution.
\begin{lemma}
  \label{lem:minimum_in_domain}
  Let $y^\star$ the optimal solution to $P_C$. Then $y^\star \ge B$.
%
%
\end{lemma}
\begin{proof}
	Define $h(z) = -\sqrt{z}$ which is convex over $\mathbb{R}^+$. Hence
	\begin{align*}
		F(y) = \sum_{i=1}^d |b_i| h(y_i) + h\left(\sum_{i=1}^d \lambda_i^{-1} y_i\right)
	\end{align*}
	is convex over $\Delta_{d-1}$, since sums of convex functions are convex, and a linear transformation of a convex function is convex. So $P_C$ is a convex optimization problem, with $d+1$ linear constraints and Lagrangian, for $\gls{mlagrangian} \in (\mathbb{R}^+)^{d+1}$:
	\begin{align*}
		\mathcal{L}(y,\nu) = F(y) - \nu_0 \left(1-\sum_{i=1}^d y_i\right) - \sum_{i=1}^d \nu_i y_i 	
	\end{align*}
KKT conditions state that there must exist $\gls{mlagrangianstar}$ such that $\nabla_y \mathcal{L}(y^\star,\nu^\star)=0$. That is for $i \in [d]$
	\begin{align*}
		0 = \frac{d}{d y_i}	F(y^\star) + \nu_0^\star - \nu_i^\star = -\frac{b_i}{2 \sqrt{y_i^\star}} - \frac{\lambda_i^{-1}}{ 2 \sqrt{\sum_{i=1}^d \lambda_i^{-1} y_i^\star}  } +  \nu_0^\star - \nu_i^\star ,  
	\end{align*}
	with complementary slackness $\nu_i^\star y_i^\star = 0$ for $i \in [d]$ and $\nu_0^\star (1 - \sum_{i=1}^d y_i^\star) = 0$.
	Also, since $y_i \mapsto F(y)$ is decreasing, the constraint $\sum_{i=1}^d y_i \le 1$ must be saturated so that $\sum_{i=1}^d y_i^\star = 1$.

	Multiplying the above by $y^\star_i$ and summing over $i$ we get
	\begin{align*}
		0 &= \sum_{i=1}^d -\frac{b_i\sqrt{y_i^\star} }{2}  - \frac{\sum_{i=1}^d y_i^\star \lambda_i^{-1}}{ 2 \sqrt{\sum_{i=1}^d \lambda_i^{-1} y_i^\star}  } +  \nu_0^\star \sum_{i=1}^d y_i^\star - \sum_{i=1}^d y_i^\star \nu_i^\star  \\
		 &= \sum_{i=1}^d -\frac{b_i\sqrt{y_i^\star} }{2}  - \frac{\sqrt{\sum_{i=1}^d y_i^\star \lambda_i^{-1}}}{2} +  \nu_0^\star
	\end{align*}
	where we used both that $\nu_i^\star y_i^\star = 0$ for $i \in [d]$ and that $\sum_{i=1}^d y_i^\star = 1$. Hence
\begin{align*}
	\nu_0^\star  = \frac{1}{2} \left(\sum_{i=1}^d b_i\sqrt{y_i^\star}  + \sqrt{\sum_{i=1}^d y_i^\star \lambda_i^{-1}} \right)
\end{align*}
First note that if $b_i = 0$ then $y_i^\star \ge B_i = 0$.
Now consider $i$ such that $b_i > 0$. If $y_i^\star = 0$, then $\frac{d}{d y_i} F(y^\star) = -\infty$ which is a contradiction, so that we must have $y_i^\star > 0$, and in turn $\nu^\star_i = 0$ so that the KKT condition becomes
\begin{align*}
	\frac{b_i}{ \sqrt{y_i^\star}} + \frac{\lambda_i^{-1}}{  \sqrt{\sum_{i=1}^d \lambda_i^{-1} y_i^\star}  } = \left(\sum_{i=1}^d b_i\sqrt{y_i^\star}  + \sqrt{\sum_{i=1}^d y_i^\star \lambda_i^{-1}} \right)
  \,.
\end{align*}
Lower bounding the second term of the l.h.s. by $0$
\begin{align*}
	\frac{b_i}{ \sqrt{y_i^\star}} \le \left( \sqrt{\sum_{i=1}^d b_i^2} \sqrt{\sum_{i=1}^d y_i^\star}  + \max_{i \in [d]} \lambda_i^{-1}  \sqrt{ \sum_{i=1}^d y_i^\star } \right) = \|b\|_{2} + \lambda_d^{-1/2}
\end{align*}
where we used the Cauchy-Schwarz inequality, the fact that $\sum_{i=1}^d y_i^\star \le 1$ and the fact that $\lambda_i$ are sorted in decreasing order.
So $y^\star \ge B$ which concludes the proof.
\end{proof}
Our next result is Lemma~\ref{lem: selfconcordance}, which shows that one can transform $P_C$, which is a convex minimization problem with linear constraints, into the minimization of $F^{(t)}$, which is a convex self-concordant function for all $t \ge t_0$ by changing the constraints as well as adding logarithmic barriers. 
This is both non-trivial and required is because the objective function of $P_C$ in its original form is not self-concordant, so that Algorithm~\ref{alg:Newton Method} does not apply to $P_C$ in its original form. 
\begin{lemma}\label{lem: selfconcordance} 
	For all $t \ge t_0$, $F^{(t)}$ is self-concordant on $D_F$.
\end{lemma}
\begin{proof}
  If $b> 0$, the function $x \mapsto -9b^{-1/2} \sqrt{x}$ satisfies the SCDI over 
  the domain $x \in [b, +\infty)$. So for any $i \in \mathcal{I}$, the function $y \mapsto -9 B_i^{-1 / 2} \sqrt{y_i} 
  =  -9 B_i^{-1 / 2}(b_i)^{-1} b_i\sqrt{y_i}$ satisfies the SCDI on $D_F$.
  If $b=0$, the function $x \mapsto b\sqrt{x} = 0$ is self-concordant on $[0, +\infty)$. So for any $i \in [d] \setminus \mathcal{I}$ the function $y \mapsto 0\sqrt{y_i}$  satisfied the SCDI on $D_F$.
  
  Similarly, for any $ y \in D_F$, $\sum_{i=1}^{d} \lambda^{-1}_i y_i
    \geq \max_{i \in [d]} \lambda^{-1}_i B_i$ so the SCDI is also satisfied on $D_F$ by
  \[
    y \mapsto 
    -9 \Big(\max_{i \in [d]} \lambda^{-1}_i B_i\Big)^{-1/2}  
    \sqrt{\sum_{i =1}^{d} \lambda^{-1}_i y_i} 
    \,.
  \]  
  Since the SDCI is stable by addition, and by scaling by numbers greater than $1$,  $tF(y)$ satisfies the SCDI.
Also $y \mapsto - \sum_{i = 1}^{d} \log( y_i - B_i) -  \log( 1 - \sum_{i=1}^{d} y_i)$ satisfies the SCDI as a logarithmic barrier. 
Hence $F^{(t)}$ is self-concordant on $D_F$ as announced.
\end{proof}

We now show Proposition~\ref{prop:time_centering}, which provides an upper bound on the number of Newton steps necessary to complete the so-called "centering step", i.e. the phase of \oursecondalgo where we attempt to minimize the self-concordant convex function $F^{(t_0)}$ up to accuracy $\epsilon$. The proof relies on the analysis of Algorithm~\ref{alg:Newton Method} we have presented in Theorem~\ref{thm:total_steps_newton_optim}.

\begin{proposition}
  \label{prop:time_centering}
Consider Algorithm~\ref{alg:Newton Method} applied to function $F^{(t_0)}$ starting at $y_0$ and accuracy $\epsilon$. This algorithm terminates after at most $N_c(\epsilon, b, \lambda)$ Newton steps where
  \begin{align*}
	  N_c(\epsilon, b, \lambda) &\le 38 \Delta^{(t_0)} + \log_2(\log_2(1/\epsilon)) \text{ with }  \Delta^{(t_0)} = F^{(t_0)}(y_0) - \min_{y \in D_F} F^{(t_0)}(y) \text{ and } \\
	  \Delta^{(t_0)} &\le t_0 (\|b\| + \lambda_d^{-1/2})  + d \log(2d\lambda_d(\|b\| + \lambda^{-1/2}_d)^2) + \log(\lambda_d(\|b\| + \lambda^{-1/2}_d)^2/(2\|b\|^2 + 1/2) ) 
  \end{align*}
  and outputs $\hat{y}$ an $\epsilon$-optimal solution $F^{(t_0)}(\hat{y}) - \min_{y \in D_F} F^{(t_0)}(y) \le \epsilon$  .

\end{proposition}

\begin{proof}
From Lemma~\ref{lem: selfconcordance}, $F^{(t_0)}$ is self concordant on $D_F$. 
So from Theorem~\ref{thm:total_steps_newton_optim}, Algorithm~\ref{alg:Newton Method} terminates after at most $N_c(\epsilon, b, \lambda)$ Newton steps where
  \begin{align*}
	  N_c(\epsilon, b, \lambda) &\le 38 \Delta^{(t_0)} + \log_2(\log_2(1/\epsilon)) \text{ with }  \Delta^{(t_0)} = F^{(t_0)}(y_0) - \min_{y \in D_F} F^{(t_0)}(y)
  \end{align*}
Let us upper bound $\Delta^{(t_0)}$. For any $y \in D_F$, by definition we have $-\|b\| - \lambda^{-1/2}_d \le F(y) \leq 0$, which yields the lower bound
\begin{align*}
	\min_{y \in D_F} F^{(t_0)}(y) \geq -t_0 \big(\|b\| + \lambda^{-1/2}_d\big)
\end{align*}
and the upper bound
\begin{align*}
  F^{(t_0)}(y^0)
    &= t_0 F(y^0) 
   - \sum_{i = 1}^{d} \log( y^0_i - B_i)
   -  \log\bigg( 1 - \sum_{i=1}^{d} y^0_i \bigg) \\
   & \leq
   -\sum_{i = 1}^{d} \log\bigg(\frac{\lambda^{-1}_d/(2d)}{(\|b\| + \lambda^{-1/2}_d)^2}\bigg)
    -  \log\bigg( 
     1 -  \frac{\|b\|^2 + \lambda^{-1}_d/2}{(\|b\| + \lambda^{-1/2}_d)^2}  
    \bigg) \\
    & \leq  d \log\bigg(\frac{(\|b\| + \lambda^{-1/2}_d)^2}{\lambda^{-1}_d/(2d)}\bigg) + \log\bigg( 
      \frac{(\|b\| + \lambda^{-1/2}_d)^2}{\lambda^{-1}_d (2\|b\|^2 + 1/2)} 
      \bigg) 
\end{align*}
Combining the upper and lower bounds give the upper bound on $\Delta^{(t_0)}$
\begin{align*}
	  \Delta^{(t_0)} &\le t_0 (\|b\| + \lambda_d^{-1/2})  + d \log(2d\lambda_d(\|b\| + \lambda^{-1/2}_d)^2) + \log(\lambda_d(\|b\| + \lambda^{-1/2}_d)^2/(2\|b\|^2 + 1/2) ) 
\end{align*}
which concludes the proof.
\end{proof}

The following standard Lemma let us control the otpimality gap induced by the  logarithmic barrier penalization :

\begin{lemma}
  \label{lem:barrier_optimality_gap} For $t > 0$, let $y^{t,\star} \in \arg\min_{y \in D_F} F^{(t)}(y)$ and $y^\star \in \arg\min_{y \in D_F} F(y)$. Then
  \[
    F^{(t)}(y^{t,\star}) - F(y^\star) \leq \frac{d+1}{t}
  \] 
\end{lemma}

\begin{proof}
  Define the quantities $\gls{nutstar} \in \RR^{d+1}$, so for all $i \in [d]$,
  \[ 
    \nu^{t,\star}_{i} := \frac{1}{t (y^{t, \star}_i- B_i)}
  \quad \text{and} \quad
  v^{t,\star}_{0} = \frac{1}{t(1 - \sum_{i=1}^d y^{t, \star}_i)} \,.
  \]
  By optimality of $y^{t,\star}$, for all $i \in [d]$,
   \[
   t \nabla{F(y^{t,\star})}_i - \frac{1}{y^{t,\star}_i- B_i} + \frac{1}{1 - \sum_{i = 1}^d y^{t,\star}_i}
   =  0
   \,,
   \]
  which can we rewrite as 
    \begin{equation}
      \label{eq:optimality_condition_ytstar} 
    \nabla{F(y^{t,\star})}_i - \nu^{t,\star}_i  + \nu^{t,\star}_{0}=  0
    \,.
  \end{equation}
  The Lagrangian of the original problem of minimizing $F$ over $D_F$ is :
  \begin{equation}
    \label{eq:def_lagrangian_original}
    \mathcal{L}(y,\nu) 
    := F(y) - \sum_{i=1}^d \nu_i (y_i - B_i) 
    -   \nu_{0} \left( 1 - \sum_{i=1} ^d y_i \right)
    \,.
  \end{equation}
  Differenciating the Lagrangian with respect to $y$ gives
  \[
    (\nabla_y \mathcal{L}(y,\nu))_i = \nabla F(y)_i - \nu_i + \nu_0
    \,.
  \]
  And we have that the dual function 
  \begin{equation}
    \label{eq:def_dualfunction}
    g(\nu^{t,\star}) := \min_y 	\mathcal{L}(y,\nu^{t,\star}) = \mathcal{L}(y^{t,\star},\nu^{t,\star}),
  \end{equation}
  checking that $y^{t,\star}$ is the critical point (see \ref{eq:optimality_condition_ytstar}). Finally we have :
  \begin{equation}
    \label{eq:dualfunctionvtstar}
    g(\nu^{t,\star}) 
    = F(y^{t,\star}) 
    - \sum_{i=1}^d \nu^{t,\star}_i \left( y_i^{t,\star}-B_i \right) 
    - \nu^{t,\star}_{0}\left( 1 - \sum_{i=1}^d y_i^{t,\star}\right) 
    = F(y^{t,\star}) - \frac{d+1}{t} 
    \,.
  \end{equation}

  Therefore, because $g(\nu^{t,\star}) \leq F(y^\star)$, (see \cite[Page
  241]{boyd_convex_2023}), we have
  $F(y^{t, \star}) - F(y^\star) < (d+1) / t$.
\end{proof}

The last intermediate result to complete the analysis is lemma~\ref{lem:geometric_barrier}, which gives an upper bound on the total number of Newton steps performed by \oursecondalgo when the barrier parameter $t$ increases geometrically until the desired accuracy is reached. 
\begin{lemma}\label{lem:geometric_barrier}
	For all $k=1,...,N_o$ define where $t_k = t_0 (1+\eta)^k$ and 
	\begin{align*}
		N_o(\epsilon,\eta,b,\lambda) = \lceil \log\big(2 (d+1)/(\epsilon t_0)\big) / \log(\eta) \rceil
	\end{align*}
	Consider the procedure where $y^{(t_k)}$ is the output of Algorithm~\ref{alg:Newton Method} applied to function $F^{(t_k)}$ starting at $y^{(t_{k-1})}$ with accuracy $\epsilon$ and with the convention that $y^{(t_0)} = y_0$. This procedure yields an $\epsilon$-optimal solution to $P_C$ that is  $y^{(t_{N_o})} \in \Delta_{d-1}$ with
\begin{align*}
		F(y^{(t_{N_o}),\star}) - \min_{y \in \Delta_{d-1}} F(y) \le \frac{\epsilon}{2},
\end{align*}
and :
\[
  F^{(t_{N_o})}(y^{(t_{N_o})})- F^{(t_{N_o})}(y^{(t_{N_o},\star)}) \leq \frac{\epsilon}{2}.
\]
It performs at most $N_b(\epsilon,\delta) N_o(\epsilon,\eta,b,\lambda)$ Newton steps with 
  \[
    N_b(\epsilon,\eta) = 38 (d+1) \left(\eta -1 -\log(\eta)\right)  + \log_2(\log_2(2/\epsilon)).
\]
\end{lemma}
\begin{proof}
We have $t_{N_o} = t_0( 1 + \eta)^{N_o} \le (d+1)/(2 \epsilon)$ so that \oursecondalgo indeed terminates after at most $N_o$ calls to Algorithm~\ref{alg:Newton Method}. From Lemma \ref{lem:barrier_optimality_gap} we have that $F(y^{t_{N_o},\star}) - F(y^\star) \leq \frac{d+1}{t_{N_o}} \leq \frac{\eps}{2}$ where $y^{t_{N_o},\star}$ is the minimizer of the penalized objective function $F^{(t_{N_o})}$. And from Theorem \ref{thm:total_steps_newton_optim} we have that : $F^{(t_{N_o})}(y^{(t_{N_o})})- F^{(t_{N_o})}(y^{(t_{N_o},\star)}) \leq \frac{\epsilon}{2}$.

To complete the proof, let us show that for all $k=1,...,N_o$, Algorithm~\ref{alg:Newton Method} applied to function $F^{(t_k)}$ starting at $y^{(t_{k-1})}$ with accuracy $\frac{\epsilon}{2}$ will terminate after at most $N_b(\epsilon,\eta)$ Newton steps.
  
We call $y^{(t)}$ the output of the Newton method for the barrier penalized objective function $F^{(t)}$. We call $y^{t, \star}$ (resp $y^{\eta t, \star}$) the minimizer of the barrier penalized objective function  $F^{(t)}$ (resp $F^{(\eta t)}$) which stay self concordant because $\eta * t > t \geq t_0 $ see Lemma \ref{lem: selfconcordance}. Using theorem \ref{thm:total_steps_newton_optim} we have that :
\[  
  N_b(\epsilon,\eta) < 38(F^{(\eta t )}(y^{(t)}) - F^{(\eta t)}(y^{\eta t, \star})) + \log_2 \log_2\bigg(\frac{2}{\epsilon}\bigg). 
\]
We can bound the gap of optimality $F^{(\eta t )}(y^{(t)}) - F^{(\eta t)}(y^{\eta t, \star}) = F^{(\eta t )}(y^{(t)}) - F^{(\eta t  )}(y^{t,\star}) + F^{(\eta t  )}(y^{t,\star}) - F^{(\eta t)}(y^{\eta t, \star})$. We reproduce the result that can be found in \cite[page 588]{boyd_convex_2023}.

Let us control the second term :
\begin{align*}
  F^{(\eta t )}(y^{t,\star}) - F^{(\eta t)}(y^{\eta t, \star}) & = \eta t F(y^{t,\star}) - \eta t F(y^{\eta t, \star}) + \sum_{i=1}^{d} \log(  \eta t \nu^{t,\star}_i ( y^{\eta t, \star}_{i} -B_i )) \\
  &  \hspace{5ex}   + \log\bigg( 
    \eta t \nu^{t,\star}_{0} (1 - \sum_{i=1}^{d}y^{\eta t, \star}_{i}) 
    \bigg)  
    - (d+1)\log(\eta)  \\
  & \leq \eta t F(y^{t,\star}) - \eta t \left[ F(y^{\eta t, \star}) - \sum_{i}^{d} \nu^{t,\star}_i (y^{\eta t, \star}_{i} - B_i) 
  - \nu^{t,\star}_{0}\bigg(1 - \sum_{i=1}^{d}y^{\eta t, \star}_{i}\bigg) \right]\\
  & \hspace{5ex}- (d-1) - (d+1) \log(\eta) \\
  & \leq \eta t F(y^{t,\star}) - \eta t g(\nu^{t,\star}) - (d+1) - (d+1)\log(\eta) \\
  & = (d+1)(\eta - 1 - \log(\eta)) 
  \,.
\end{align*}
We defined for the first equality $\nu^{t,\star}_i := \frac{1}{t
\left(y^{t,\star}_{i} -B_i\right)}$ and $\nu^{t,\star}_{0} := \frac{1}{t \left( 1 -
\sum_{i=1}^{d}y^{t,\star}_{i}\right)}.$ For the first inequality we used the fact that
$\log(a)\leq a - 1$ for any $a >0$. For the second inequality we used the definition of the dual function $g(\nu^{t,\star})$ for the original problem with the Lagrangian
\ref{eq:def_lagrangian_original}. Finally, the last equality is due to the fact
$g(\nu^{t,\star}) = F(y^{t,\star}) - (d+1)/t $ for this particular choice of
$\nu^{t,\star}$ see \ref{eq:dualfunctionvtstar} in the proof of Lemma \ref{lem:barrier_optimality_gap}.

For the first term,  we use a standard argument from \cite[Acuracy of centering, Chapter 11, Page 570]{boyd_convex_2023}, to assume that $F^{(\eta t )}(y^{(t)}) - F^{(\eta t  )}(y^{t,\star}) \simeq 0$ and is indiscernible for a computer due to the quadratic convergence of Newton method. One can ask for a precision far greater that $\frac{\epsilon}{2}$ for the required Newton precision, it would only impact the number of iterations through the $\log_2\log_2$ terms.  

We have by definition we have that $t_k = \eta t_{k-1}$ which conclude the proof.
\end{proof}

Putting everything together, we may now state Theorem~\ref{thm:complexity_reduction_algo} which provides a complete analysis of \oursecondalgo including the total number of Newton steps performed across all phases, and the total time and space complexity. This concludes our analysis of the \oursecondalgo Algorithm and justifies the claims made in the paper.
\begin{theorem}
  \label{thm:complexity_reduction_algo}
	Algorithm \oursecondalgo run with parameter $\eta$, intialization $y_0$ and accuracy $\epsilon$ perfoms at most $N_{tot}(b,\lambda, \epsilon, \eta )$ Newton steps where,
  \begin{multline*}
    N_{tot}(b,\lambda, \epsilon, \eta ) < 
    t_0 \big(\|b\| + \lambda_d^{-1/2}\big) 
      + d \log\Big(2d\lambda_d(\|b\| + \lambda^{-1/2}_d)^2\Big)
      + \log\bigg( 
       \frac{\lambda_d(\|b\| + \lambda^{-1/2}_d)^2}{(2\|b\|^2 + 1/2)} 
       \bigg)  \\
       + \log_2\log_2\bigg(\frac{1}{\epsilon}\bigg)
       + \left(38 (d+1) \left(\eta -1 -\log(\eta)\right) 
       + \log_2\log_2\bigg(\frac{2}{\epsilon}\bigg) \right) 
       \bigg\lceil  \frac{\log(2(d+1)/(\epsilon t_0))}{\log(\eta)} \bigg\rceil
       ,
    \end{multline*}
  If $\eta = 1 + 1/\sqrt{d+1}$, we have
  \begin{multline*}
    N_{tot}(b, \lambda, \epsilon, 1+1/\sqrt{d+1}) 
    < 38t_0
       \big(\|b\| + \lambda^{-1/2}_d\big) 
      + 38d \log\Big(2d\lambda^{-1}_d(\|b\| + \lambda^{-1/2}_d)^2\Big) \\
      + 38\log\bigg( 
       \frac{\lambda_d(\|b\| + \lambda^{-1/2}_d)^2}{2\|b\|^2 + 1/2} 
       \bigg) 
       + \left(19 +  \log_2\log_2\Big(\frac{2}{\epsilon}\Big) \right) 
       \bigg\lceil  \sqrt{d+1}\log\Big(\frac{2(d+1)}{\epsilon t_{0}}\Big) + 1 \bigg\rceil.
      \end{multline*}
 It runs in time at most $O(d^2 N_{tot}(b,\lambda,\epsilon, \eta) ) + d^3$ and outputs $(\hat{x},\hat{\theta})$, an $\epsilon$-optimal solution to $P_B$ in the sense that $\|\hat{\theta} - \theta^\star \|_{2} \le \epsilon$ and $\hat{x} \in \arg\max_{x \in \mathcal{X}} x^\top \theta$. 
\end{theorem}

\begin{proof}
  The first step of the algorithm is to reduce \problem to \ref{prob:reduction_ball}. From Theorem \ref{thm:reduction-to-conv} we have that if the final $y$ is an $\epsilon$-minimizer of the problem \ref{prob:reduction_ball} the output $\hat{x}, \hat{\theta}$ is an $\epsilon$-minimizer of the problem \ref{prob:bilinear}. 

  The total number of Newton steps is the sum of the number of Newton steps of the initial
  centering step $N_c$ (see Proposition \ref{prop:time_centering}) and the number of Newton steps during the barrier phase $N_b(\epsilon,\eta) N_o(b,\lambda,\epsilon, \eta)$ (see Lemma \ref{lem:geometric_barrier}). The total number of Newton steps is thus
  \[
    N_{tot}(b,\lambda,\epsilon, \eta) = N_c(b,\lambda,\epsilon, \eta) + N_b(\epsilon,\eta) N_o(b,\lambda,\epsilon, \eta ).
  \]
  Replacing those quantities per their value gives the announced upper bound on the number of Newton steps. 

  Setting $\eta = 1+1 / \sqrt{d+1}$ and following the computation of \cite[Page
  591]{boyd_convex_2023} yields the second formula and the result of Theorem~\ref{thm:reduct-full-analysis-eta-d}.

  The algorithm requires a change of basis which is of time complexity $O(d^3)$. Each
  Newton step requires the computation of the gradient which can be done in $O(d)$, the
  computation of the Hessian of $F^{(t)}$, which is the sum of a diagonal matrix and a
  rank $1$ matrix. Therefore, the time complexity of the computation of the Hessian and
  its inversion is $O(d^2)$ using the Sherman-Morrison formula. Then some matrix-vector
  products are required to compute the descent direction and the local norm of the
  gradient. which can be done in $O(d^2)$. The total time complexity of \oursecondalgo is
  thus $O(d^2 N_{tot} + d^3) $.
\end{proof}

\subsection{Proof of Theorem~\ref{thm:reduction-to-conv}}
\label{app:proof-reduction-to-conv}

\begin{proof}[Proof of Theorem~\ref{thm:reduction-to-conv}]
 Let us start by following exactly Step I in the proof of
  Theorem~\ref{thm:Correctness}.
With the change of variables $u= U A^{1/2} x$ and $\phi = U A^{-1/2} \theta$, note that
$P_B$ is equivalent to solving $P_b'$
\begin{align*}
  \tag{$P_b'$}
	\text{maximize } u^\top \phi 
  \text{ subject to } \|u\|_{2} \le 1 \text{ and } \|\phi-b\|_{\Lambda} \le 1
  \,.
\end{align*}
Note that $P_b'$ falls under the assumptions of Theorem~\ref{thm:reduc-ellp}:
$\Lambda$ is diagonal and the first domain is the unit euclidean ball. 
Therefore, we know that $y\in \Delta_{d-1}$ is an $\eps$-solution to $P_C$ if and only if 
\[
  u = \sqrt{y_i}\sign(b_i) 
  \quad \text{and} \quad
  \phi = b +  \frac{\Lambda^{-1}u}{\|u\|_{\Lambda^{-1}} }  \,.
\]
Reversing the change of variables,
\[
  \theta 
  = A^{1/2} U^\top \phi 
  = A^{1/2} U^\top (b + \frac{\Lambda^{-1} UA^{1/2}x}{\|u\|_{\Lambda^{-1}}}) 
  = c + \frac{W^{-1} x}{\|x \|_{W^{-1}}}  \,
\]
by noticing that $\|u\|_{\Lambda^{-1}} = \|x\|_{W^{-1}}$.
\end{proof}

\section{Proofs of Section~\ref{sec:axes-aligned}}
\label{app:proofs-aligned}

\subsection{Convexity of $H$}
\begin{proposition}
  \label{prop:h-is-convex}
   If $p \geq 2$, then the function $y \mapsto H(y)$ is convex on $\R_{+}^d$.
\end{proposition}
\begin{proof}[Proof of Proposition~\ref{prop:h-is-convex}]
Recall that
\begin{align*}
	H(y) = -\sum_{i=1}^{d} y_i^{1/p} |c_i|-\sqrt{ \sum_{i=1}^{d} \lambda_i^{-1} y_i^{2/p} }
\end{align*}
$y \mapsto -\sum_{i=1}^{d} y_i^{1/p} |c_i|$ is convex as a sum of convex functions since $y \mapsto -y_i^{1/p} |c_i|$ is convex for $p \geq 2$. 
$y \mapsto -\sqrt{ \sum_{i=1}^{d} \lambda_i^{-1} y_i^{2/p} }$ is convex (see for instance~\cite[p. 84]{boyd_convex_2023}) as the composition of a convex non-increasing function $z \mapsto - \sqrt{z}$ and a concave function $y \mapsto \sum_{i=1}^{d} \lambda_i^{-1} y_i^{2/p}$, since $y \mapsto y_i^{2/p} \lambda_i^{-1}$ is concave for all $p \ge 2$. So $H$ is convex as announced.
\end{proof}

\subsection{Reduction from $P_B$ to $P_{C, p}$}
\begin{proof}[Proof of Theorem~\ref{thm:reduc-ellp}]
   Define for any $z \in \mathcal{X}$, the objective function
    \[
    G(z) = \max_{\theta \in \Theta } \big\{z^\top \theta \big\}
    = 
    z^\top c
     + \sqrt{z^\top \Lambda^{-1} z} \,.
    \]  
    Then $(z, \theta)$ is an $\eps$-approximate solution to \problem iff $z \in \mathcal{X}$ and 
    $z^\top \theta \geq G(z') - \eps$ for any $z' \in \cX$.
    
    Note that for any $y \in \Delta_{d-1}$, we have, denoting by $x$ the vector with
    coordinates $x_i = y_i^{1/p} \sign(c_i)$,
    \[
    H(y) =
     -\sum_{i=1}^{d} y_i^{1/p}\sign(c_i) c_i
     - \sqrt{\sum_{i=1}^d \lambda_i^{-1} y_i^{2/p}}
     =
     -x^\top c
     - \sqrt{ x^\top \Lambda^{-1}  x}
     = -G(x) \,.
    \]
    \paragraph{First direction}
    Let $y$ be an $\eps$-solution to \ref{prob:reduction_ball_ellp}, and $x$ be the
    corresponding vector in $\cX$.
    Let $(z, \theta) \in \cX \times \Theta$. By definition of $G(z)$, 
    \begin{multline*}
      z^\top \theta \leq G(z) 
     = \sum_{i=1}^d  z_i c_i + \sqrt{\sum_{i = 1}^{d} \lambda_i^{-1}  z_i^2}
     \leq 
     \sum_{i=1}^d | z_i| |c_i| +\sqrt{ \sum_{i = 1}^{d} \lambda_i^{-1} z_i^2}
     \leq \frac{1}{\|z\|_p}
     \bigg(\sum_{i=1}^d | z_i| |c_i| 
     + \sqrt{\sum_{i = 1}^{d} \lambda_i^{-1}  z_i^2}\bigg)
     \,.
    \end{multline*}
    We used the fact that $\|z\|_p \leq 1$ for the final inequality.

  Let $y' \in \Delta_{d-1}$ be the vector with coordinates
  $
   y_i' = (|z_i| / \|z\|_p)^p
  \,.
  $
  Then $y'$ is indeed in the simplex $\Delta_{d-1}$ since all its coordinates are
  non-negative and sum to $1$. We have shown above that
  \[
    z^\top \theta \leq 
   \bigg(\sum_{i=1}^d (y_i')^{1/p} |c_i| 
   + \sqrt{\sum_{i = 1}^{d} \lambda_i^{-1} (y_i')^{2/p}}\bigg)
   = -F(y') \,.
  \]
  But by our assumption on the approximate optimality of $y$, we know that $H(y') \geq
  H(y) - \eps $, so
  \[
    z^\top \theta \leq -H(y) + \eps  = G(x) + \eps
    = x^\top \theta^\star + \eps \, .
  \]
  Taking the supremum over $(z, \theta)$ completes the proof.

  \paragraph{Converse}
  Let $y \in \Delta_{d-1}$ be such that $(x, \theta^\star)$ is a solution to
  \problem. 
  Then for any $y' \in \Delta_{d-1}$, with corresponding $x'$, 
  \[
    F(y') 
    = -G(x') 
    \geq -G(x) - \eps = F(y) - \eps \,.
  \]
  Therefore $y$ is an $\eps$-approximate solution to \ref{prob:reduction_ball_ellp} \,.
\end{proof}

\section{Numerical Experiments}
    \label{app:numerical_experiments}

\subsection{Time for a run of a Linear Bandits Algorithms}
Here in Figure~\ref{fig:time} we  average time taken for each algorithm to perform a run with horizon $T = 10000$. The results are averaged over 10 runs and are presented with $95\%$ confidence intervals. We see that the algorithm that requires a bilinear maximization step and using \oursecondalgo is significantly slower than the ones that use \ouralgo. Then Thompson Sampling is even faster because there is no bilinear maximization step and only requires a matrix inversion/ diagonalization at each step. Finally, E2TC is the faster algorithm because it only requires diagonal matrix inversion and multiplication by a vector at most $\log_2(T)$ times and vector additions during the exploration phase which is can be much smaller than $T$.

\begin{figure}[ht]
    \centering
    \includegraphics[width=0.95\textwidth]{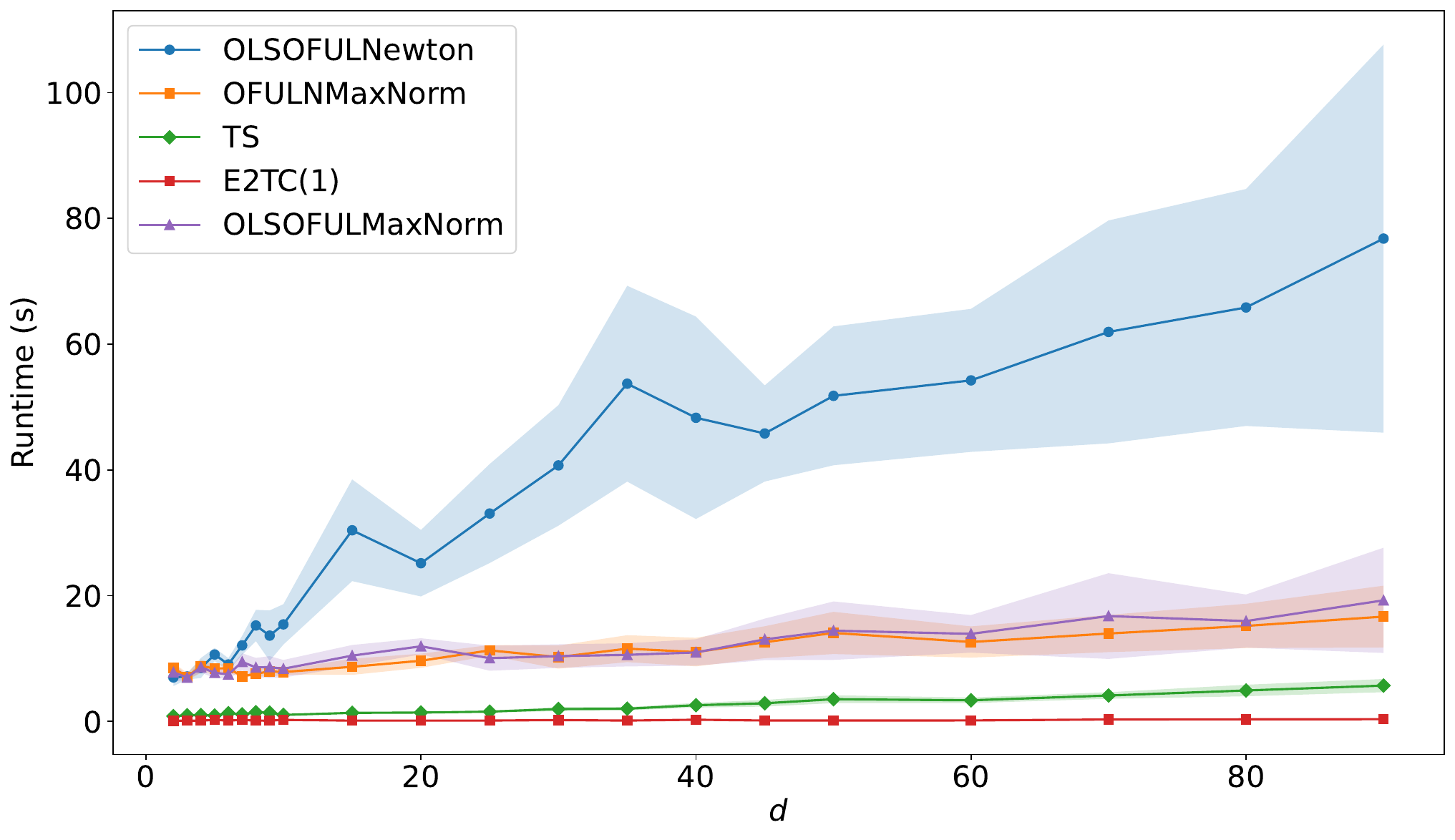}
    \caption{Average time taken for each algorithm to perform a run with horizon $T = 10000$}
    \label{fig:time}
\end{figure}

\subsection{A Little Example in Dimension 2}
In Figure~\ref{fig:little_example} we provide a little example in dimension 2 for \problem when $\mathcal{X}$ is a unit ball and $\Theta$ is an ellipsoid to give some intuition on \problem. The GeoGebra file is given in the supplementary material so that one can change the parameters and see how the solution changes.    

\begin{figure}[ht]
    \centering
    \includegraphics[width=0.8\textwidth]{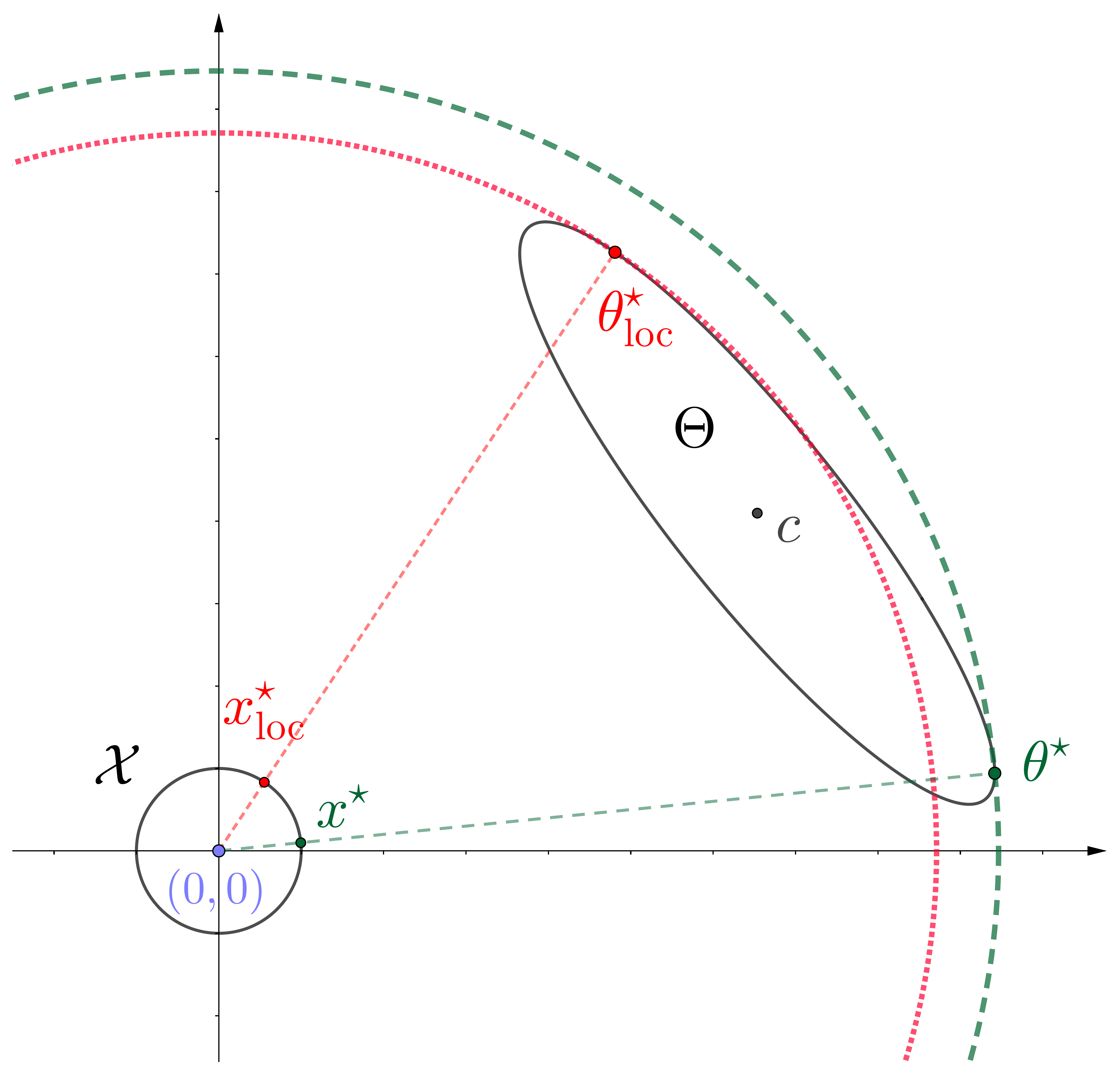}
    \caption{Example in dimension 2 when $\mathcal{X} = \{x, \| x\|_2 \leq 1 \}$ is a unit ball and $\Theta$ an ellipsoid.}
    \label{fig:little_example}
\end{figure}
In the example provided, one can see that there can be multiple local maxima. So without a careful initialization, local search or iterative algorithms such as Successive Alternate Maximization \ref{alg:succ-alg} can converge to a local maximum such as $\theta^{\star}_{loc}$. (if the algorithm is initialized at $\theta^{\star}_{loc}$ it stays at this local maximum). 
\begin{algorithm} 
    \SetKwInOut{Input}{Input}
    \Input{Matrices $A, W \in \mathbb{R}^{d \times d}$, vector $c \in \mathbb{R}^d$, initial guess $\theta_0$, accuracy $\epsilon \ge 0$}
    \tcp{Initialization}
    $\hat{\theta} =  \theta_0$ \\
    $\hat{x} = A^{-1/2}\frac{\hat{\theta}}{\|\hat{\theta}\|_2}$  \\
    $\hat{\theta}^+ = c+ W^{-1} \hat{x}/\|\hat{x}\|_{W^{-1}}$ \\
    \tcp{Iterative steps of alternate maximization}
    \While{$\| \hat{\theta}^+ - \hat{\theta} \|_2 > \epsilon$}{
        $\hat{\theta} = \hat{\theta}^+$ \\
        $\hat{x} = A^{-1/2} \frac{\hat{\theta}}{\|\hat{\theta}\|_2}$  \\
        $\hat{\theta}^+ = c + W^{-1} \hat{x}/\|\hat{x}\|_{W^{-1}}$ 
    }
    $\hat{\theta} = \hat{\theta}^+$ \\
    {\bf Output:} $(\hat{x}, \hat{\theta})$ an approximate solution to $P_{B}$
    \caption{Successive Alternate Maximization Algorithm}\label{alg:succ-alg} 
  \end{algorithm}


\clearpage
\section{Glossary}
\printnoidxglossaries

\end{document}